\documentclass[twoside,11pt]{article}
\pdfoutput=1

\usepackage{ntheorem}
\usepackage{jmlr2e}

\usepackage{lastpage}

\ShortHeadings{Towards Unbiased Exploration in Partial Label Learning}{Zombori and Rissaki}
\firstpageno{1}

\usepackage{hyperref}
\usepackage{xcolor}
\hypersetup{
    colorlinks=True,
    citecolor={blue!50!black},
    linkcolor={blue},
    urlcolor={blue},
}

\usepackage{url}
\usepackage{subcaption}

\usepackage{tabularx} %

\usepackage{xcolor}

\newcommand{\datalogarrow}{\leftarrow}
\newcommand{\prploss}{\mathsf{Libra}\text{-}\mathrm{loss}}
\newcommand{\prplosssubscript}{\mathrm{Lib}}
\newcommand{\biprploss}{\mathsf{Sag}\text{-}\mathrm{loss}}
\newcommand{\biprplosssubscript}{\mathrm{Sag}}
\newcommand{\biprp}{\mathrm{bi}\text{-}\mathrm{PRP}}
\newcommand{\biprps}{\mathrm{bi}\text{-}\mathrm{PRP_s}}
\newcommand{\nllloss}{\mathsf{NLL}\text{-}\mathrm{loss}}

\newcommand{\uloss}{\mathrm{uniform}\text{-}\mathrm{loss}}
\newcommand{\lwsloss}{\mathrm{LWS}\text{-}\mathrm{loss}}
\newcommand{\rcloss}{\mathsf{RC}\text{-}\mathrm{loss}}
\newcommand{\bmeritloss}{\mathrm{\beta}\text{-}\mathrm{merit}\text{-}\mathrm{loss}}

\newcommand{\bmeritlossB}{\mathrm{0.5}\text{-}\mathrm{merit}\text{-}\mathrm{loss}}

\newcommand{\bmerit}{\mathrm{\beta}\mathrm{M}}

\newcommand{\sigmoid}{\sigma}

\usepackage{amsmath}
\usepackage{paralist}
\usepackage{thmtools, thm-restate}
\usepackage{booktabs}       %

\newtheorem{postulate}[theorem]{Postulate}
\newtheorem{fact}{Fact}[section]
\newtheorem{claim}[theorem]{Claim}

\newcommand{\pd}[2]{\frac{\partial{#1}}{\partial{#2}}}
\newcommand{\softmax}{\mathop{\mathrm{softmax}}}
\newcommand{\RR}[0]{\mathbb{R}}
\newcommand{\cfunct}{\kappa}
\newcommand{\mt}{\mathbb{MT}}
\newcommand{\m}{\mathbb{M}}
\newcommand{\source}{\mathbb{S}}

\newcommand{\possup}{\mathbb{P}}
\newcommand{\negsup}{\mathbb{N}}

\renewcommand{\vec}{\mathbf}	%

\newcommand{\introparagraph}[1]{\textbf{#1.}} %
\newcommand{\R}{{\mathbb{R}}} %
\newcommand{\smallsection}[1]{\vspace{2mm}\noindent\textbf{#1.}} %

\newcommand{\prp}{\mathrm{PRP}}
\newcommand{\prps}{\mathrm{PRP_s}}
\newcommand{\nll}{\mathrm{NLL}}

\newcommand{\nsamples}{n}
\newcommand{\indim}{d}
\newcommand{\outdim}{m}
\newcommand{\numallowed}{k}

\newcommand{\myinput}{\boldsymbol x}
\newcommand{\mylabel}{\boldsymbol y}
\newcommand{\mylabelscalar}{y}
\newcommand{\truetarget}{\boldsymbol y_{\textrm{true}}}
\newcommand{\trueprob}{p_{\textrm{true}}}
\newcommand{\truelogit}{z_{\textrm{true}}}
\newcommand{\dataset}{D}
\newcommand{\prefix}{{\boldsymbol s}_{\textrm{prefix}}}
\newcommand{\seq}{\vec s}
\newcommand{\seqlen}{\ell}
\newcommand{\model}{\mathcal{M}}

\newcommand{\myinputvar}{\myinput}
\newcommand{\mylabelvar}{\mylabel}
\newcommand{\truetargetvar}{\truetarget}

\newcommand{\outputspace}{\mathcal{Y}}
\newcommand{\distribution}{\mathcal{P}}
\newcommand{\Loss}{\mathcal{L}}
\newcommand{\prate}{r_{Dpool}}
\newcommand{\srate}{r_{Docc}}

\newcommand{\logits}{{\boldsymbol z}}
\newcommand{\logitsscalar}{z}
\newcommand{\probs}{{\boldsymbol p}}
\newcommand{\probsscalar}{{p}}					%
\newcommand{\classifier}{\boldsymbol f}			%
\newcommand{\g}{\boldsymbol g}
\newcommand{\params}{\boldsymbol{\theta}}
\newcommand{\paramsscalar}{\theta}

\usepackage[capitalise, nameinlink]{cleveref} %

\Crefname{postulate}{Postulate}{Postulates}
\Crefname{definition}{Definition}{Definitions}
\Crefname{proposition}{Proposition}{Propositions}
\Crefname{claim}{Claim}{Claims}
\Crefname{equation}{Equation}{Equations}
\Crefname{definition}{Definition}{Definitions}

\title{Towards Unbiased Exploration in Partial Label Learning}

\author{\name Zsolt Zombori \email zombori@renyi.hu \\
        \addr Alfr\'{e}d R\'{e}nyi Institute of Mathematics \\
        Budapest, Hungary
        \AND
        \name Agapi Rissaki \email rissaki.agapi@gmail.com \\
        \addr Northeastern University \\
        Boston, USA    
        \AND
        \name Krist\'{o}f Szab\'{o} \email krist.sz13@gmail.com \\
        \addr Alfr\'{e}d R\'{e}nyi Institute of Mathematics \\
        Budapest, Hungary
        \AND
        \name Wolfgang Gatterbauer \email wgatterbauer@northeastern.edu \\
        \addr Northeastern University \\
        Boston, USA    
        \AND
        \name Michael Benedikt \email michael.benedikt@cs.ox.ac.uk \\
        \addr Department of Computer Science \\
        University of Oxford\\
        Oxford, UK
}

\editor{}

\begin{document}

\maketitle

\begin{abstract}%
We consider learning a probabilistic classifier from
partially-labelled supervision (inputs denoted with multiple
possibilities) using standard neural architectures with a softmax as
the final layer.  We identify a bias phenomenon that can arise from
the softmax layer in even simple architectures that prevents proper
exploration of alternative options, making the dynamics of gradient
descent overly sensitive to initialization.  We introduce a novel loss
function that allows for unbiased exploration within the space of
alternative outputs.  We give a theoretical justification for our loss
function, and provide an extensive evaluation of its impact on
synthetic data, on standard partially labelled benchmarks and on a
contributed novel benchmark related to an existing rule learning
challenge.

\end{abstract}

\begin{keywords}
    partial label learning, disjunctive supervision, rule learning
\end{keywords}

\section{Introduction}

\emph{Partial Label Learning} (PLL) \citep{pllavgbased,pllearly,pllzoubin,pllidentificationbased,provablyconsistentpll,pllleveraging,plltemporal,pllsurvey} deals with learning in the presence of imperfect supervision, 
where training data has a set of labels, 
one of which is the true label.
The framework of PLL  is very general, and a number of well-studied problems, including 
learning in the presence of partially-observable variables,
can be seen as particular instances 
with certain specialized assumptions (e.g.\ that one has a probabilistic model that
constrains the generation of disjunctive outputs). 
Over the last decade a multitude of proposals for PLL have emerged: 
for example, 
methods that treat the set of labels as an ensemble and average over them \citep{pllavgbased}, or 
methods that try
to learn patterns that distinguish noisy labels from true labels \citep{pllzoubin,pllearly,pllanotheridentificationbased}.
We are motivated by the setting where no assumptions are made about how the partial supervision is generated,
but only on the class of functions being learned.

A scenario related to but different from PLL that we refer to as
\emph{Disjunctive Supervision} (DS) is when supervision gives multiple
possible outputs and \emph{any one} of these outputs is acceptable.
One motivating application for this scenario of disjunctive
supervision is in applying modern machine learning techniques to
\emph{rule learning}, where the goal is to learn rules that can be
used to derive target facts from some source facts.  When formulated
as a supervised machine learning problem, an important feature is that
there may be multiple rules that can be used to derive any given
target fact.

\begin{example}[Rule Learning with Disjunctive Supervision]
\label{ex:running}
Assume that we have two database tables, \texttt{Person} and \texttt{Author}. 
A tuple \texttt{Person(x,y,z)} implies there is a person called $x$, who is $y$ years old and belongs to group $z$, and a tuple \texttt{Author(x)} implies that $x$ is an author.
As a simple example of DS, suppose the source facts include
\texttt{Person(alice,45,1)} and
\texttt{Person(bob,34,1)}
and we would like to find mapping rules that derive  target facts
\texttt{Author(alice)}
and \texttt{Author(bob)}.
Two  candidate rules may be
\begin{align*}
  & \texttt{Author}(x) \leftarrow \exists a, t. \texttt{Person}(x, a, t) \\
  & \texttt{Author}(x) \leftarrow \exists a. \texttt{Person}(x, a, 1)
\end{align*}
Above we use Prolog-style syntax, where $x$ is implicitly universally quantified. 
We are interested in neural models that generate
rules of the form above
from such source and target facts.
Either rule above is equally acceptable as an output for deriving the target facts. Thus the target facts can be associated
with the disjunctive label consisting of the set of output rules that can derive them.
\end{example}

\begin{example}[Semantic parsing with Disjunctive Supervision] 
\label{ex:semparse} 
As another example of DS, consider a variant of the \emph{semantic parsing} task, inspired
by \cite{guu-etal,clarkcurran}:  A user issues a sequence of commands
in natural language, where each command describes a transformation of
a fixed state (e.g.\ repositioning objects within a scene).  
The goal is to
translate the natural language utterances into 
commands in some fixed programming language. 
A human annotator
provides supervision on training examples, but only at the level of
the observed state sequence.  %
Since several commands can have
the same impact along the entire training and test dataset, there may
be no unique correct answer.
As a simple example, suppose that the parser is trying to learn the
state transition associated with utterance \texttt{Alice moves to the
  left of Bob}. The available supervision only reveals that Alice ends
up at position 1 (which is left of Bob), making it impossible to
distinguish the intended state transition from the one associated with
utterance \texttt{Alice moves to position 1}.

In solving this problem it is natural to learn a sequential model, where a network outputs the
probabilities of a command for a given utterance, conditioned on the
prior sequence of utterances.  Notice that in this task we can
efficiently check whether a command sequence matches the supervision,
by executing it.  
But usually we cannot hope to compute an explicit
list of the acceptable outputs that match the supervision: the number
of possible sequences can be enormous.  
Since we cannot enumerate all acceptable sequences when we want to
compute the aggregate loss over examples, the best we can do is  to
sequentially sample according to our current learned distribution.
\end{example}

Note that there is no difference in the training data between PLL and
DS.  The formal difference concerns the assumptions about the
underlying process, the corresponding task loss, and the evaluation
methodology.  For PLL we assume an unknown joint distribution on the
true function and on the noise model that generates the additional
outputs. Its goal is, as in classical multi-class classification
tasks, to learn the true value. In evaluating a solution for PLL, one
needs a gold standard of correct values for the test data.  In DS,
however, we assume only an unknown process generating sets of
labels. Our optimization problem is to maximize the expected value
that the chosen output is one of the correct ones (see the definition
of $\nllloss$ later in \cref{eq:NLL1} for a more precise task loss).
In evaluating performance for DS, we do not need to have any gold
standard true values, evaluation is based on the partial labelling
(see \Cref{Fig_OverallSetup} for key differences).

\begin{figure}[t]
  \centering
  \begin{subfigure}[b]{0.46\textwidth}  
	  \centering
	  \includegraphics[scale=0.55]{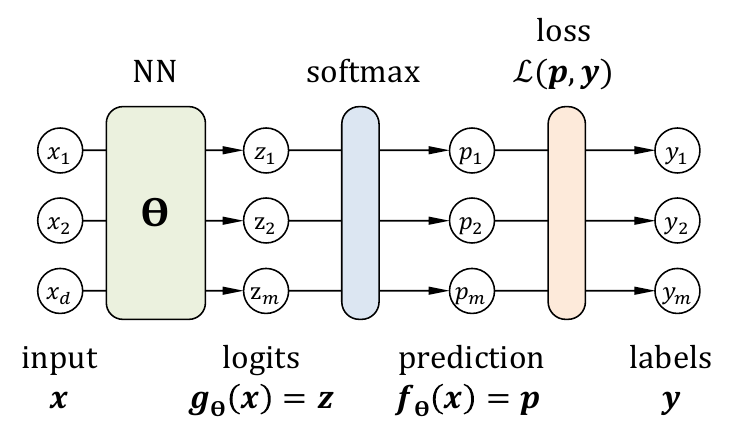}
	  \caption{Neural architectures with softmax}
	  \label{Fig_OverallSetup_a}
  \end{subfigure}
  \begin{subfigure}[b]{0.16\textwidth}  
	  \centering	  
	  \includegraphics[scale=0.55]{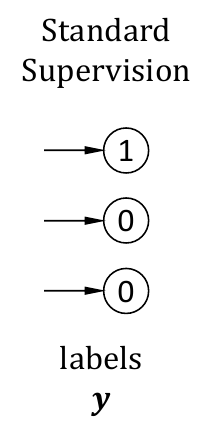}
	  \caption{Multi-class}	  
	  \label{Fig_OverallSetup_b}
  \end{subfigure}
  \begin{subfigure}[b]{0.22\textwidth}  
	  \centering	  
	  \includegraphics[scale=0.55]{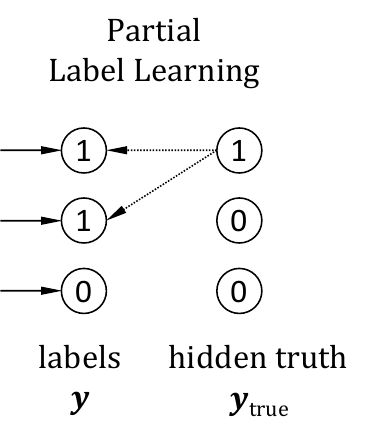}
	  \caption{PLL}	  
	  \label{Fig_OverallSetup_d}
  \end{subfigure}
  \begin{subfigure}[b]{0.13\textwidth}  
	  \centering	  
	  \includegraphics[scale=0.55]{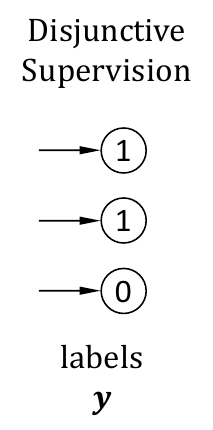}
	  \caption{DS}	  
	  \label{Fig_OverallSetup_c}
  \end{subfigure}

  \caption{We consider standard neural architectures with a softmax layer (i.e.\ the goal is to \emph{predict one output}) whose parameters $\params$
  are learned via supervision from labeled examples 
  $(\myinput, \mylabel)$ via a loss function $\Loss\big(\probs, \mylabel \big)$
  comparing the predictions $\probs$ against the labels $\mylabel$ (a).
  In the standard \emph{multi-class learning} scenario, exactly one correct label is supplied (b).
  In \emph{Partial Label Learning} (PLL), several alternative labels are supplied but only one among them is correct (c).
  In \emph{Disjunctive Supervision} (DS), 
  any one of the 
  different alternative labels is accepted (d).
  }
\label{Fig_OverallSetup}
\end{figure}

With multiple outputs labeled for a given input in the training set,
supervision for PLL/DS also resembles supervision for  multi-label
classification.  The key difference is that PLL/DS seek a function that produce
a single output
output as the answer. %
\Cref{tab:learning_tasks} shows a comparison between the tasks,
while \Cref{ex:learning_tasks} gives an example of each.

\begin{table}[t]
\centering
\small
  \caption{Comparing multi-class classification, 
  multi-label classification, PLL and DS.}
  \label{tab:learning_tasks}
  \begin{tabular}{l l l l}
    {\bf Learning task } & {\bf Supervision} & {\bf Prediction} & {\bf Interpretation} \\
    \toprule
    Multi-class & 1 & 1 & Single true label \\
    Multi-label & multiple & multiple & Several true labels \\
    PLL & multiple & 1 & Single (unknown) true label \\
    DS & multiple & 1 & Any one of the allowed labels is true \\
    \bottomrule  
  \end{tabular}
\end{table}

\begin{example}[Path Learning Scenarios]
\label{ex:learning_tasks}  
Consider a path finding problem in some dangerous environment: given
endpoints $A$ and $B$, we aim to find paths that take us safely from
$A$ to $B$. Standard multi-class learning is when there is a single
safe path between $A$ and $B$ and it is provided for each training
sample.  In PLL too, there is a single safe path for each pair of
endpoints, but it is not known for the training samples, only a set of
paths that contains the single safe one. In multi-label learning,
there are numerous safe paths and we aim to identify all of them. In
DS, there are several safe paths, and a valid model should identify 
one of them.
\end{example}

Given the same supervision, our work will be applicable to both scenarios.
We focus on
classifiers that output a probability distribution over
output space $\outputspace$, by application of a final \emph{softmax}
layer.  Consider partially/disjunctively labelled training samples of the form
$(\myinput, \mylabel)$ where $\mylabel \subseteq \outputspace$ is the
set of acceptable labels for input $\myinput$. The output of a
classifier $\probs = \classifier(\myinput)$ represents a probability
distribution over $\outputspace$. Much  of the literature on
PLL ( e.g. \cite{provablyconsistentpll,guu-etal}), 
uses a variation of the loss:
\begin{align}
\label{eq:NLL1}
	\Loss_{\nll}(\probs, \mylabel) &= - \log \Big(\sum_{i \in \mylabel} \probsscalar_i \Big)
\end{align}

This is simply
the negative log likelihood %
of obtaining a target in $\mylabel$ when
sampling from $\probs$, denoted $\nllloss$. However, for the softmax
architecture, we show that simply training with this loss -- whether in the PLL or DS scenario -- leads to an
undesirable property that some of the \emph{acceptable} labels would
be favoured over others, when trained using gradient descent. In fact,
in the absence of other supervision,
this leads to a \emph{winner-take-all} 
scenario, where
all the probability concentrates on only one of the acceptable labels.

As an alternative, we propose a novel loss function, the $\prploss$ (\cref{def:libraloss}),
whose updates preserve the ratios of the probabilities for the
acceptable labels in the absence of other supervision.
We show that such a loss function is unique up to
composition by a differentiable function 
under some natural technical
conditions. This more balanced loss leads to more stable training and increased success rate in finding a better optimum irrespective of the starting conditions.

\begin{example} 
\label{ex:figure}
  Let us examine a toy problem with $\indim=10$ inputs and
  $\outdim=100$ outputs. We assume a single training sample
  $(\myinput, \{A, B, C\})$, i.e., having $k=3$ allowed outputs.  We
  train a neural network that consists of a single dense layer with
  $100$ neurons and softmax nonlinearity, having $1100$ parameters
  altogether.  \Cref{fig:winner_takes_all_a} shows the behavior of the
  standard $\nllloss$, and \Cref{fig:winner_takes_all_b} our
  $\prploss$, both starting from the same initial condition.
  $\nllloss$ results in a distribution where the allowed output $A$
  with the highest initial probability accumulates all the probability
  mass. In contrast, $\prploss$ yields a balanced update and the ratio
  of the allowed outputs does not change.
\end{example}

\begin{figure}[thb]
  \centering
  \begin{subfigure}[b]{0.47\textwidth}  
	  \centering
	  \includegraphics[width=0.95\linewidth]{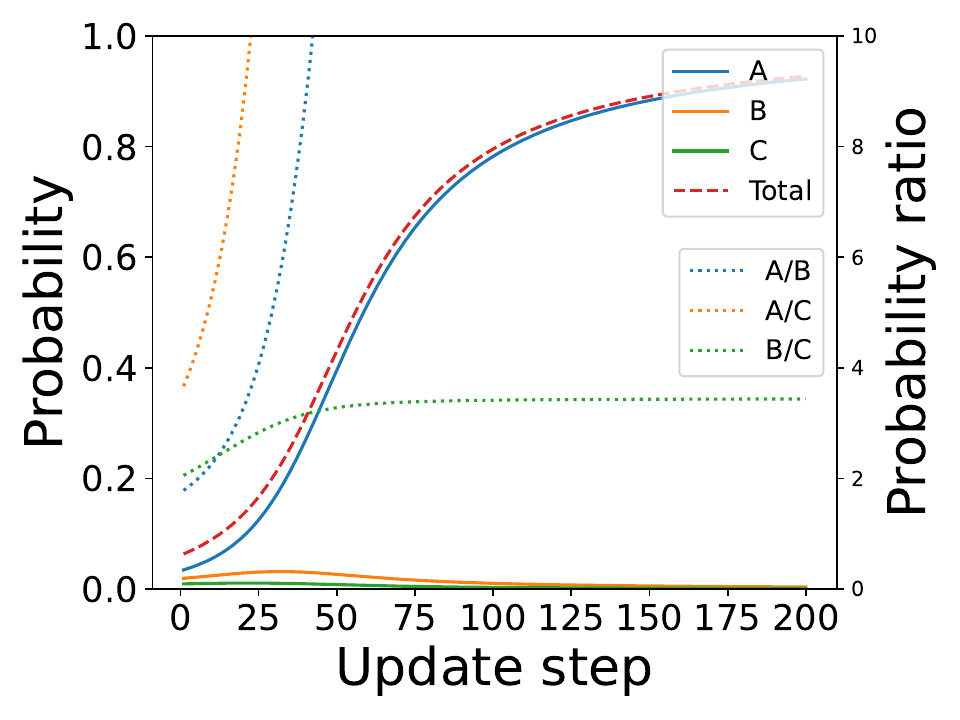}
	  \caption{$\nllloss$}
	  \label{fig:winner_takes_all_a}
  \end{subfigure}
  \begin{subfigure}[b]{0.47\textwidth}  
	  \centering	  
	  \includegraphics[width=0.95\linewidth]{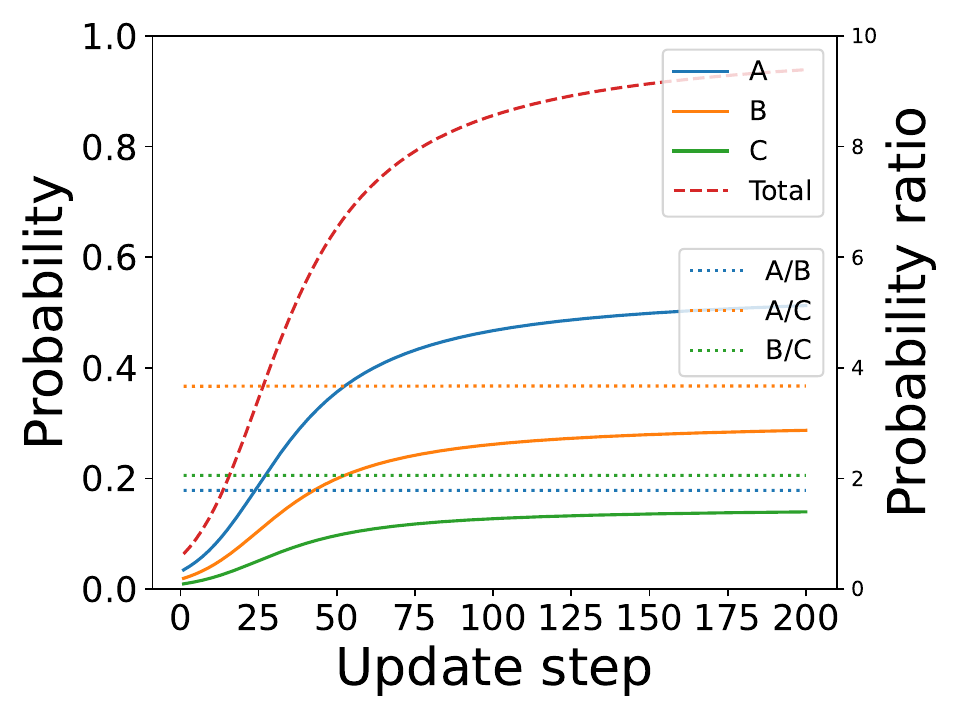}
	  \caption{$\prploss$}	  
	  \label{fig:winner_takes_all_b}
  \end{subfigure}
  \caption{\Cref{ex:figure}:
  Learning curves using a small classifier
    neural network  for
    $\nllloss$ (a) and $\prploss$ (b). Solid
    lines represent probabilities and dotted lines represent
    probability ratios of allowed outputs. 
    The dashed Total line is
	the sum of allowed probabilities, i.e., $A+B+C$.
	Notice how the relative ratio between allowed labels remain constant in (b), while the `winner-takes-all' in (a).
    }
\label{fig:winner_takes_all}
\end{figure}

\smallsection{Contributions}
The paper's contributions are as follows:

\begin{compactitem}

\item We describe a bias phenomenon for architectures ending in a
  softmax layer when learning from partially/disjunctively labelled
  datasets and using $\nllloss$.  We show (\Cref{thm:winner_take_all})
  that it prevents proper exploration of alternatives when optimising
  $\nllloss$.

\item We formulate a property to avoid the observed bias and derive
  from it the $\prploss$ function, whose updates maintain the ratios
  of probabilities for acceptable labels produced by the softmax.  
  We
  show that when loss functions are restricted to depend only on the
  predicted probabilities 
  of \emph{acceptable outputs}, $\prploss$ is uniquely defined (up to
  composition by differentiable functions).
  
\item We consider a stronger property that aims to avoid bias not
  only among acceptable labels, but \emph{also among unacceptable ones} and
  derive from it the $\biprploss$ function. We show that among all
  loss functions that can depend on both acceptable and unacceptable
  probabilities, $\biprploss$ is uniquely defined (again, up to composition by
  differentiable functions).
  
\item We compare several methods from the PLL literature
  experimentially both on synthetic and real-world datasets. These
  experiments demonstrate the performance and accuracy benefits of
  $\prploss$, while results related to $\biprploss$ are not
  conclusive.  In particular, we find that $\prploss$ is more robust
  than other variants when the learning task becomes harder, either
  because there are more labels in the label sets or because some
  distractor labels co-occur very often with the true label.

\item We provide novel DS datasets appropriate for \emph{rule learning} in a supervised context.

\item The entire codebase is available from the project webpage~\citep{webpage}.
\end{compactitem}

\smallsection{Organization}
We overview related work in \Cref{sec:related} and define our problem in \Cref{sec:prelims}.
\Cref{sec:biasprp} provides our key technical contributions: the formalization of the bias problem, and our solution using
probability-preserving loss functions.
\Cref{sec:experiments} is concerned with testing our approach experimentally.
We close with conclusion in \Cref{sec:conc}.
All proofs, as well as some details of the experimental set up  are in 
the Appendix.

\section{Related work} \label{sec:related}

\introparagraph{Partial Label Learning (PLL)}
Partial Label Learning has by now an extensive literature. See, for
example \cite{pllsurvey} for a recent survey.
One common approach is to dampen the loss
proportionally to an average of the overlap between the output
probability distribution and each acceptable label, possibly also
including a component that enhances the loss proportionally to an
average overlap with the unacceptable outputs. This approach has many
variations and goes under the heading of ``average-based methods''
\citep{pllleveraging,pllavgbased}. 
Another family of approaches
attempts to learn the noise model in combination with learning the
prediction.  These are sometimes referred to as
``identification-based'' methods
\citep{pllidentificationbased,pllanotheridentificationbased}. They
might use a strategy similar to \emph{expectation maximization} to alternate
between refining the model of the most likely true labels and
exploiting the model to make predictions.

We mentioned that there are other scenarios with the same weak
supervision as in PLL, but different assumptions: the multiple outputs
do not represent one true output corrupted with noise; rather they
represent multiple possible outputs and we are indifferent to which
one is selected (\cref{Fig_OverallSetup}). We call this scenario \emph{Disjunctive Supervision}
(DS). It has appeared in prior work,
for example in the literature on semantic parsing
\citep{guu-etal}. However, its connection to and distinction from PLL
have not yet been investigated. We treat the two setups in a unified
framework and show that the same optimization methods are applicable for
both, while requiring different evaluation protocols.

Many problems in the machine learning literature can be recast as
special constrained cases of PLL or DS.  For example, if one has a model
with latent variables, such as a \emph{Hidden Markov Model}~\citep{hmm}, any
value of the output can be generated by multiple valuations of the
hidden variables, thus the output can be considered partial
supervision over the possible latent variable values.
The underlying probabilistic
model constrains how partial supervision can be generated.  In contrast,  here we will have a
model on the underlying function class being learned, but no assumption on how the partial supervision is 
generated: thus prior techniques from the latent variable literature will not be applicable.

\introparagraph{Optimization}
In terms of optimization, our $\prploss$ function can be viewed as a
form of \emph{entropic regularization}~\citep{entropic}, 
with the notable difference
that we apply regularization to a truncated distribution of the output
that is different for each datapoint.

\introparagraph{Rule learning}
One of the applications of our loss function in the DS setup is in the
setting of a neural approach to rule learning.
Rule learning has been studied from both
a theoretical and practical perspective for many decades. The theory
includes complexity bounds within a number of learning models.
An
example is the complexity of finding a Horn sentence that entails 
a give set of statements, while contradicting (or merely failing to entail) another set of
sentences \citep{learninterp}. This problem has also been considered in the presence of a background theory
$\Sigma$: thus entailment is with respect to $\Sigma$.
Our setting is of this form, where the background theory consists of ground facts.
Like most variations of the problem,
this is known to be intractable even when the size of the rule bodies is
fixed.  Intuitively, one has to guess a rule or rules that fit the
data, and then verify via evaluating the body of the guessed clause.
For a formalization of this intuition, see the $\Sigma^p_2$-completeness
results in \cite{gottlobilphard}.

\introparagraph{ML for Rule Learning}
One response to the combinatorial hardness of rule learning is to consider a smooth
semantics for logical rules, aiming to make the loss amenable to neural methods.
An  example of this approach is Neural Theorem Proving~\citep{uclrules,ctp}, which looks for candidate
rules of a shape constrained by a template.
Atoms are scored using a smooth variant of unification, based on
a parameterized embedding of facts in Euclidean space.
 Scores are aggregated using the MIN function within a rule  and MAX across rules.
The score of a rule and the parameters of the embedding are then optimized
via
gradient descent. The MIN/MAX aggregation results in extremely
sparse gradients, leading to computational difficulties.
In addition, the sharpness of MIN/MAX boundaries makes it difficult to
move between alternatives, resulting in a ``closest-take-all''
behaviour, 
not unlike the ``winner-take-all'' behaviour of the
$\nllloss$, presented in our paper.  In \cite{imperialrules} each
possible rule is associated with a weight, and k-step forward
reasoning is performed to compute a score for supervised facts. When
aggregating scores of alternative proofs, the authors note that MAX
aggregation adversely affects gradient flow and use the probabilistic
sum $f_{\mathrm{agg}}(x,y) = x+y - x y$ instead. 
This makes the gradients
denser, but does not guarantee balanced gradients among alternatives.
The $\prploss$ function presented in this paper is designed
specifically to make the transition between alternative derivations as
smooth as possible, allowing better exploration.

\introparagraph{Symbolic supervision in ML}
Learning logical rules represents one application of our framework,
but there is a broader connection between DS and logic, in that
disjunctive supervision can be thought of as a special case of
symbolic supervision, where the supervision is given wholly or in part
by constraints.  The set up contrasts with much prior work on neuro-symbolic methods
\cite{semanticloss,neuroreg,harnessing,embeddingsymbolic}, which
focus on enforcing semantic information given by logical
constraints that are known to hold globally across all inputs,
including those outside the training set.
This prior work deals with logical constraints that are more complex than disjunctions, and
the loss functions that are introduced  (e.g.\ in \cite{semanticloss,neuroreg})
are themselves hard to compute in the worst case.
\cite{neuroreg} deals with a regularization term which is constraint-aware,
analogous to our loss function.
But entropy is being minimized 
to achieve sharper decision boundaries,
while in our case it is being maximized
to enhance exploration.

\introparagraph{Rule Learning for aligning heterogeneous data sources}
Our rule learning experiments are based on the RODI benchmark~\citep{rodi},
 aimed at comparing systems for aligning
relational sources with a target schema. Several such
systems are evaluated in \cite{rodi}.   However, the systems do not make use of
supervision, looking only at textual and structural similarities
between source and target. In contrast, we focus on learning the
alignment from supervision.
Nevertheless, we note that the success percentage of all
examined systems on the RODI challenges
ranges between $3-50\%$, much
lower than ours (see \Cref{subsec:rulelearningexp}). This highlights the
benefit of approaching the alignment problem via supervised machine
learning.

\section{Preliminaries and Problem Statement}
\label{sec:prelims}

Supervised classification is the task of learning a function that
conforms to a given set of samples $\dataset
=\big\{\big(\myinput^{(j)},
\mylabel^{(j)}\big)\big\}_{j=1}^{\nsamples}$ where $\myinput \in \R^d$
is the input and $\mylabel \in \{0,1\}^\outdim$ the one-hot encoded
desired output (i.e.\ exactly one entry is 1).  In \emph{Partial Label
Learning} (PLL) and \emph{Disjunctive Supervision} (DS), however,
there can be more than one allowed output, represented as $\mylabel$
having multiple entries being 1.  The difference between PLL
and DS is that the former assumes one single correct output among the
given 1 entries that is unknown at training time (thus the labels are
uncertain), while the latter assumes that each of the  entries are equally correct:
thus the labels are not uncertain but ``disjunctive''.  We at times
overload the notation and define $\mylabel$ as the set of allowed
labels as indexed by the binary vector.\footnote{We use $1$-indexing: For example $\mylabel = (1,
0, 1)$ has 2 acceptable outputs and could equally be written as
$\mylabel = \{1, 3\}$.}  We use $\numallowed$ to denote the number of
acceptable outputs associated with label $\mylabel$ in supervision.  We use
$\truetarget$ 
to denote the one-hot encoded unknown correct
output, in the setting of  PLL.

\smallsection{Partial Label Learning (PLL) vs.\ Disjunctive Supervision (DS)} 
PLL assumes a joint data
generating distribution $\distribution(\myinputvar, \truetargetvar,
\mylabelvar)$ 
on inputs $\myinputvar \in \R^d$, true one-hot outputs $\truetargetvar \in \{0,1\}^\outdim$,
and partial supervision $\mylabelvar \in \{0,1\}^\outdim$. 
In other words, the observed labels $\mylabelvar$ are a distorted representation of 
the true labels $\truetargetvar$ and the former always includes the later.
The goal is to learn a
function $\classifier$ in a given target class 
that maximizes:
\[
\mathbb{E}_{\distribution(\myinputvar,\truetargetvar,\mylabelvar)}[P(\classifier(\myinput)=\truetarget)] .
\] 
DS makes no assumption
about a single true output $\truetarget$. It assumes only a joint data
generating distribution $\distribution(\myinputvar, \mylabelvar)$ 
on
inputs $\myinputvar$ and partial supervision $\mylabelvar$. Our target
is to learn a function $\classifier$ that maximizes:
\[
\mathbb{E}_{\distribution(\myinputvar,\mylabelvar)}[P(\classifier(\myinput) \in \mylabel)] .
\]
As in supervised learning, we do not know the PLL/DS distribution
$\distribution$: instead, we assume a finite $\dataset
=\{(\myinput^{(j)}, \mylabel^{(j)})\}_{j=1}^{\nsamples}$ sampled
uniformly from $\distribution$ and we focus on optimising performance
on this set.  Our learning target class will be a statistical model
$\probs = \classifier_{\params}(\myinput)=\softmax
(\g_{\params}(\myinput))$, with input $\myinput$ and parameters
$\params \in \R^t$, interpreting its output as a probability
distribution over the output space.  The function $\g$ gives the
unnormalized output, called \emph{logits}, which we denote as
$\logits$.

Since the supervision is indistinguishable for PLL and DS
(only its interpretation),
technically, the same optimization methods are
applicable, and most of our
theoretical claims are relevant in both scenarios.\footnote{This, however, does not necessarily mean that the
same method is optimal for both problem classes.}

We aim to find the \emph{Maximum Likelihood Estimate (MLE)}, which maximizes
the joint probability of the observed data. For DS, this means the
probability, given $\myinput$, of observing an element $o$ such that
$o \in \mylabel$. While for PLL, it is the conditional probability
given $\myinput$ of observing an $o$ with $o \in \truetarget$.  For
computational reasons, one usually minimizes the negative logarithm of
this value:
\begin{align}
 -\log \bigg( \prod_{j=1}^{\nsamples}\classifier_{\params}(\myinput^{(j)})
	\cdot \mylabel^{(j)} \bigg) 
	&= - \sum_{j=1}^{\nsamples} \log \left(\probs^{(j)} \cdot
	\mylabel^{(j)} \right) 
	& \textrm{(DS)}
	\label{eq:mle_ds} \\
 - \log \bigg( \prod_{j=1}^{\nsamples}\classifier_{\params}(\myinput^{(j)}) \cdot
	\truetarget^{(j)} \bigg) 
	& = - \sum_{j=1}^{\nsamples} \log \left(\probs^{(j)} \cdot
	\truetarget^{(j)} \right)
	& \textrm{(PLL)}
	\label{eq:mle_pll}
\end{align}

\Cref{eq:mle_pll} cannot be optimized directly as $\truetarget$ is not
known for training samples. However, in the absence of any prior
preference over acceptable labels (i.e., assuming each have the same
probability of being the true one) the expected value of
\Cref{eq:mle_pll} only differs from \Cref{eq:mle_ds} by a
multiplicative constant.\footnote{The multiplicative constant is the
number of allowed outputs $\numallowed$.} Thus, we also consider
\Cref{eq:mle_ds} as a natural measure in the setting of PLL as
well. This measure is known as the \emph{Negative Logarithm of the
Likelihood}~\citep{dlbook} function,
which yields the following samplewise \emph{$\nllloss$}:
\begin{align}
	\label{eq:NLL}
	\Loss_{\nll}\big(\probs, \mylabel \big) &= - \log \left(\probs \cdot
	\mylabel \right) = - \log \bigg( \sum_{i=1}^{\outdim} \probsscalar_i
	\mylabelscalar_i \bigg)
\end{align}

The above formulation of the $\nllloss$ is a direct generalization of
the classical case with a single allowed output.
We use the $\nllloss$ as a baseline for optimization and argue that it
is not an ideal choice 
for PLL/DS due to its sensitivity to
initial configuration. The same applies to most identification-based
methods, 
such as $\lwsloss$ and $\rcloss$ (defined later in
~\Cref{subsec:competitors}).

\smallsection{Extension to sequential outputs} 
PLL and DS have
important applications in which the output space cannot be effectively
modeled as a set of unstructured objects. For instance, in the path
finding problem of \Cref{ex:learning_tasks}, there can be a huge
number of paths (even unbounded) and we may want our model to
generalise to unseen paths, not just to unseen endpoints. In such
scenarios, it is not tractable to explicitly compute a distribution
over the entire output space, as a standard classifier model would
do. \emph{Autoregressive models} provide a solution for such problems:
instead of producing the output one-shot, they build it incrementally:
given an input and a partially constructed output, an autoregressive
model predicts the next component of the output. Consequently, one has
to repeatedly evaluate such models to obtain the final prediction.

\begin{example}[\cref{ex:learning_tasks} continued]
  Returning to the path finding problem, each output can naturally be
  modeled as a sequence of atomic choices coming from a small fixed
  set, e.g.\ \{``north'', ``west'', ``south'', ``east''\}.
\end{example}

We extend PLL/DS for problems where outputs are represented as
sequences over a finite alphabet $\Sigma$ of $\outdim$
elements. Following the terminology of language modeling, we refer to
the elements of the alphabet as \emph{tokens}.  Our learning target
class will be a statistical model $\probs =
\classifier_{\params}(\myinput,\prefix)=\softmax
(\g_{\params}(\myinput,\prefix))$.  Thus, besides input $\myinput$,
the model receives an extra argument $\prefix \in \Sigma^*$ which is a
sequence of tokens from $\Sigma$.  Notice that while the size of
alphabet $\outdim$ is finite, the length of the sequences is not
necessarily so.  The output is a probability distribution over
$\Sigma$, interpreted as the distribution of the next token of the
output following $\prefix$.  Given input $\myinput$ and sequence $\seq
= s_1 \dots s_{\seqlen} \in \Sigma^*$ the model $\classifier$ can be
used to compute the predicted probability of $\seq$
as
\begin{equation}
  \label{eq:seqprob}
  P_{\params}(\seq | \myinput) =
  \prod_{i=1}^{\seqlen} \classifier_{\params}(\myinput, (s_1 \dots s_{i-1}))_{I(s_i)}
\end{equation}
where $I(s_i)$ refers to the index in the output of $\classifier$
that corresponds to token $s_i$.
Let $\seq^{(1)}, \seq^{(2)} \dots$ be an arbitrary, fixed ordering of
all sequences over $\Sigma$. 
Given input $\myinput$, let $\probs$
represent the (possibly infinite) 
vector of model predicted
probabilities, i.e.,
$$
\probsscalar_i = P_{\params}(\seq^{(i)} | \myinput)
$$

In dataset $\dataset =\{(\myinput^{(j)},
\mylabel^{(j)})\}_{j=1}^{\nsamples}$ with sequential output space,
$\mylabel^{(j)}$ is an indicator vector over finite sequences.  Note
that even if $\Sigma$ is finite, the set of all sequences may be
infinite. 
We restrict $\mylabel$ to have only finitely many $1$'s, so
that it is finitely representable. This way, the predicted
probabilities of allowed sequences can be computed according to
\Cref{eq:seqprob}. Hence, any method that directly optimises only the
probabilities of allowed outputs generalizes directly to the
sequential case. 
In particular, minimizing the $\nllloss$ is also
applicable and yields the maximum likelihood estimate for DS and a
natural proxy loss for PLL. Note, however, that the
probabilities of all the disallowed outputs cannot be effectively
computed, ruling out some optimization methods.

\Cref{tab:notation} summarizes the notation used throughout the paper.

\begin{table}[t]
  \caption{Summary of the notation used in the paper}
  \label{tab:notation}
\centering
\small
\begin{tabularx}{\linewidth}{@{\hspace{0pt}} >{$}l<{$} @{\hspace{2mm}}X@{}} %
  \hline
  \textrm{Symbol}		& Definition 	\\
  \hline
  \hline
  (\myinput^{(j)}, \mylabel^{(j)})	
  & Example training point with $\myinput \in \R^{\indim}$ and $\mylabel \in \{0,1\}^{\outdim}$ \\
  \truetarget^{(i)}				& Unknown one-hot true label $\truetarget^{(i)} \subseteq \mylabel^{(i)}$ with $|\truetarget^{(i)}|=1$ \\	
  \params					& Parameters to learn \\
  \nsamples & Number of samples \\
  \indim & Input dimension \\
  \outdim & Number of outputs \\
  \numallowed & Number of 1's in label $\mylabel$ \\
  \seqlen		& Length of sequential output \\	
  \distribution & Data generating distribution \\
  \dataset & Finite dataset, sampled uniformly from $\distribution$ \\
  \Loss & Loss function \\
  \g: \R^{\indim} \rightarrow \R^{\outdim} & Logit function \\
  \classifier: \R^{\indim} \rightarrow [0,1]^{\outdim} & Probabilistic classifier function \\	
  \logits & Unnormalized model prediction (``logits''): output of $\g$ \\
  \probs & Normalized model prediction (``probabilities''): output of $\classifier$ \\
  o & Element from the output space \\
  \Sigma & finite alphabet of $\outdim$ elements for in the sequential setup \\
  \hline
\end{tabularx}

\end{table}

\section{Addressing bias in partial label learning} \label{sec:biasprp}

We prove that, even for simple architectures, standard optimization
based on a direct generalization of the MLE, 
i.e.\ the $\nllloss$ in \cref{eq:NLL},
leads to biased
``winner-take-all'' learning. We then introduce our main contribution,
a novel property of loss functions, the $\prp$ property, which
formalizes the absence of learning bias. 
We provide a loss function,
the $\prploss$, that possesses this property, and also show that it is the unique loss function
satisfying the $\prp$ property, up to composition by differentiable
functions.
Next, we relate the $\prploss$ to entropy regularization~\citep{confidence_penalty}
and the $\nllloss$. Then, we introduce the $\biprp$ 
property, an
extension of the $\prp$ property and provide an analogous
characterization theorem based on a loss function called $\biprploss$.
We end the section with practical considerations.

\subsection{Stability of Probability Ratios of Allowed Outputs During Training} \label{subsec:stability}

Given a set of samples $\dataset$ and a function $\classifier$, let
$\classifier_{|\dataset}$ denote the function with domain restricted
to $\{ \myinput \mid (\myinput, \mylabel) \in \dataset \}$.  When the
supervision is \emph{total}, i.e., each $\myinput$ corresponds to a
single output, and $\mylabel = \truetarget$ is one-hot, 
then there is a
single optimal function $\classifier_{|\dataset}^*$ that fits
perfectly to $\dataset$, namely when
$\classifier_{|\dataset}^*(\myinput) = \mylabel$
for each
$(\myinput, \mylabel) \in \dataset$.\footnote{Note, however, that the
same optimal function $\classifier_{|\dataset}^*$ can have multiple
realizations in terms of $\params$.}  This does not hold when
supervision is partial/disjunctive: given that we have no direct
information (PLL) or preference (DS) about the true label
$\truetarget$, any output distribution that places all the probability
mass over the acceptable outputs can be considered as perfect fitting
to the training signal.  Other constraints -- such as regularization,
interaction among training points, or task-specific requirements --
might restrict this set of optima. However, we argue that it is very
important to avoid any prior bias in the learning algorithm towards
any of these optimal distributions.  Let $\probsscalar_i =
\classifier_{\params}(\myinput)_i$ denote the probability of the
$i^{\textrm{th}}$ dimension of the output distribution.  
The unwanted bias that we target  in this paper 
is ``winner-take-all''. That is if $(\myinput, \mylabel) \in
\dataset$ with $\mylabelscalar_i=\mylabelscalar_j=1$ and
$\probsscalar_i > \probsscalar_j$ at initialization, then the
optimization converges to $\probsscalar_i = 1$ and $\probsscalar_j=0$.
To see why such behaviour is undesirable, consider \cref{ex:3outputs}.

\begin{example}
\label{ex:3outputs}
Consider a problem with
$\outdim=3$ outputs: $A$, $B$ and $C$.   
Assume two samples with the same input $\myinput$: $(\myinput, \{A, B\})$ and
$(\myinput, \{A, C\})$.\footnote{Notice, that for this example we will use the set notation for partial supervision.}
Next, assume that at initialization we have $\classifier(\myinput)_B > \classifier(\myinput)_A$ and $\classifier(\myinput)_C > \classifier(\myinput)_A$.
Then the signals from the two samples work against each other if using $\nllloss$
trying to
increase the probability of $B$ and $C$, respectively, instead of
finding the joint optimum in $A$. This example is analyzed
in greater depth in \Cref{ex:fig4}, as well as in \Cref{fig:vectorfield_1AB_1AC,fig:winner_takes_all_interaction}
within \cref{sec:increasingmodelcomp}.
\end{example}

In general, any randomized initialization in the parameters can lead
to an initial bias among the outputs,  which may prevent the expected
interaction among different points.  Ideally, we would like the model
update operation to preserve an \emph{invariance property}: as we
increase the aggregate probability of a set of values, the
distribution within the set should not change:

\begin{postulate}[Ratio preservation]
\label{post:prp} 
For one training point with multiple allowed outputs, a single
optimization step that updates model parameters $\params$ should
preserve the \emph{ratio of probabilities} of the allowed outputs.
\end{postulate}

We formalize this for a parameterized distribution
$\classifier_{\params}$ with parameters $\params$. We assume that
training is done via \emph{gradient descent}, referred to as
\emph{Gradient-update}:

\begin{definition}[Gradient-update]
  \label{def:gupdate}
Given a parameter vector $\params$, an update operation is called a \emph{Gradient-update} if there exists some loss function $\Loss$ 
and \emph{learning rate $\lambda > 0$} such that the update on the $i^{\textrm{th}}$
parameter is
\[
	\paramsscalar'_{i} := \paramsscalar_i - \lambda \pd{\Loss}{\paramsscalar_i}
\]
\end{definition}

\subsection{Negative Log Likelihood ($\nllloss$) and Bias}
We now show that gradient descent on the $\nllloss$ from \cref{eq:NLL}
leads to a
``winner-take-all'' effect in the presence of partial supervision.
The intuitive explanation for this is that the easiest way to decrease
the $\nllloss$ is to increase the greatest probability:
the gradient of the $\nllloss$ that the logits receive
(through the softmax layer) is proportional to the output
probabilities.

Our formal results apply exactly to a simple class of classifiers
called \emph{softmax regression}~\citep{Tsoumakas2007}.

\begin{definition}[Softmax Regression]
  We refer to \emph{softmax regression} as the parametric model $\probs = \classifier_{\params}(\myinput)=\softmax (\params \cdot
  \myinput)$.
\end{definition}

\begin{restatable}[Winner-take-all]{theorem}{thmwta}
  \label{thm:winner_take_all}
  Consider the softmax regression model
  $\classifier_{\params}(\myinput)$. Fix a datapoint $(\myinput, \mylabel)$, and let $J$ be
  the set of acceptable outputs such that for every $j\in J$,
  $\probsscalar_j=\classifier_{\params}(x)_j$ is maximal among the allowed output
  probabilities. Then the Gradient-update 
  operation with $\Loss=\nllloss$ from \cref{eq:NLL}
  yields
  a limit distribution
  $$
  \probsscalar_j = \begin{cases}
    \frac1{|J|} 	~&\text{if} ~ j\in J	\\ 
    0 			~&\text{otherwise}
  \end{cases}
  $$
\end{restatable}

\Cref{thm:winner_take_all} states that the model converges to a
distribution in which all the probability mass is evenly distributed
among a subset $J$ of allowed outputs that initially had maximal
probability.  Under any realistic model and random
initialization, there is a single allowed output with maximal initial
probability, i.e., $J$ is a singleton and all the probability mass
converges to a single output. As we have illustrated
in \cref{ex:3outputs}, this ``winner-take-all'' behaviour is
harmful, as it can prevent the optimizer from fitting to other
points. The proof is provided in \Cref{app:winnertakeall}.

\begin{example}
\label{ex:3outputs_2}
Again, assuming arbitrary input dimension and $\outdim=3$ outputs: $A$, $B$ and $C$, 
we examine the optimization dynamics of $\nllloss$ with a single sample $(\myinput, \{ A, B \})$.
\Cref{fig:vectorfield_1AB_nll} visualizes the winner-take-all behaviour of $\nllloss$ in this case. 
We see that the model converges to $A$ or $B$ depending on which one has greater
initial probability.
\end{example}

\begin{figure}[thb]
  \centering

  \begin{subfigure}[b]{\textwidth}
    \centering
    \includegraphics[width=0.44\textwidth]{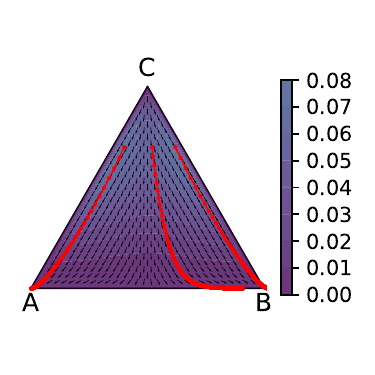}
    \includegraphics[width=0.44\textwidth]{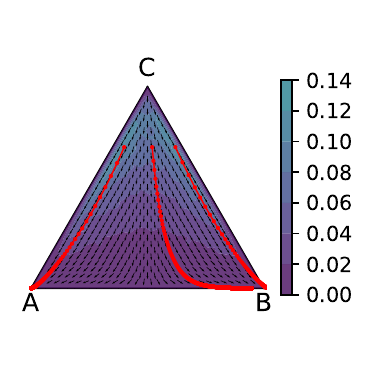}
    \vspace{-4mm}	
    \caption{ $\nllloss$ }
    \label{fig:vectorfield_1AB_nll}
  \end{subfigure}

  \begin{subfigure}[b]{\textwidth}
    \centering
    \includegraphics[width=0.44\textwidth]{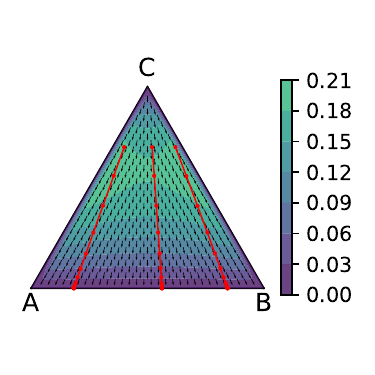}
    \includegraphics[width=0.44\textwidth]{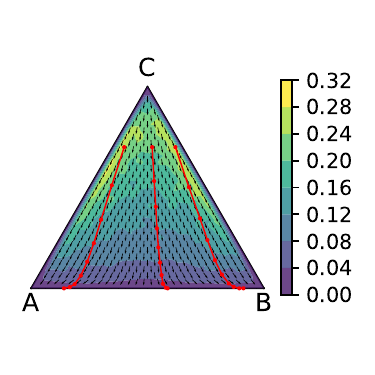}
    \vspace{-4mm}		
    \caption{ $\prploss$ }
    \label{fig:vectorfield_1AB_prp}
  \end{subfigure}

  \caption{\Cref{ex:3outputs_2}: Given $\outdim=3$ outputs $A$, $B$,
    $C$ and one sample $(\myinput, \{ A, B \})$, we show the direction (arrows)
    and magnitude (colors) of gradient updates with (a) the $\nllloss$, (b) the
    $\prploss$.
    We show a Softmax regression model (Left) and an MLP with 10 hidden layers (right).
    Red curves show real trajectories from
    three fixed starting points: $(\probsscalar_A, \probsscalar_B,
    \probsscalar_C) \in \{(0.25, 0.05, 0.7), (0.13, 0.17, 0.7), (0.05,
    0.25, 0.7)\}$, terminated when
    $\probsscalar_C < 0.0001$.  
    $\nllloss$ makes
    the model converge to either $A$ or $B$, while $\prploss$ only
    slightly distorts the initial probability ratios between $A$ and
    $B$.  Also notice the increased speed: while $\nllloss$ with softmax regression
    took $(8374, 8671, 9906)$ steps to converge with fixed learning rate, $\prploss$ took $(14, 14, 14)$ steps.
  }
\label{fig:vectorfield_1AB}
\end{figure}

\subsection{Loss Functions with the Probability Ratio Preserving (PRP) Property}
\label{subsec:prp}

Towards correcting this observed systematic bias of the
$\nllloss$, we present  the following formalization of \Cref{post:prp}:

\begin{definition}[$\prp$ property]
  \label{def:prp-property}
  Given a parametric model $\classifier_{\params}$, 
  a continuously differentiable function
  $\Loss(\probs, \mylabel)$ is said to satisfy the
  \emph{Probability Ratio Preserving } ($\prp$) property 
  for $\classifier_{\params}$ if any
  Gradient-update on $\classifier_{\params}$ with loss function $\Loss$ preserves the ratio of probabilities of all
  outputs $i$ with $\mylabelscalar_i=1$.
\end{definition}

Given any loss function $\Loss$, whether it satisfies the $\prp$
property depends on the model architecture.
\Cref{thm:winner_take_all}
demonstrates how $\nllloss$ introduces a strong preferential bias even
for a basic softmax regression model. Therefore, we focus on this
model class in our formal results:

\begin{definition}[$\prps$ property]
  \label{def:prps-property}
  A continuously differentiable function $\Loss(\probs, \mylabel)$ is
  said to satisfy the \emph{probability-preserving property for
  softmax regression} ($\prps$ property) if it satisfies the $\prp$
  property for the softmax regression model.
\end{definition}

There is a simple loss function that satisfies the $\prps$ property. Since the
loss function ``balances'' the probabilities of the different outputs,
we refer to it as the $\prploss$.

\begin{definition}[$\prploss$]
\label{def:libraloss}
  Let \emph{$\prploss$} denote the following function:
  $$
  \Loss_{\prplosssubscript}(\probs, \mylabel) = 
  \underbrace{\log \bigg(1- \sum_{i=1}^m \mylabelscalar_i \probsscalar_i \bigg)}_{\textup{Disallowed term}} -
  \underbrace{\frac{1}{k} 
  \sum_{i=1}^m \mylabelscalar_i \log(\probsscalar_i)}_{\textup{Allowed term}}
  $$ where $\numallowed = \sum_i \mylabelscalar_i$ is the number of
  allowed outputs. The first term is the positive log likelihood of
  selecting a disallowed label, while the second term is the average
  of the individual negative log likelihood losses for each allowed
  output.
\end{definition}

In \Cref{app:prploss}, we show that the $\prploss$ has the desired property:
\begin{restatable}{theorem}{thmprploss}
  \label{thm:prp_loss}
  The $\prploss$ function has the
  $\prps$ property.
\end{restatable}

$\prploss$ only depends on the probabilities of allowed outputs and is invariant under permutation of the output vector. We formalise this property as:

\begin{definition}[acceptable-dependent]
  \label{def:acceptable_dependent}
  A loss function $\Loss$ is said to be \emph{acceptable-dependent} if
  its value only depends on the $\probsscalar_i$ for which
  $\mylabelscalar_i=1$ and it is invariant under any permutation $\pi\in S_{\outdim}$ of the
  coordinates of the arguments of $\Loss$.  (i.e., $\forall \pi\in
  S_{\outdim}, \Loss(\pi\circ \probs, \pi\circ \mylabel) =
  \Loss(\probs, \mylabel)$).
\end{definition}

In fact, when we restrict attention to acceptable-dependent functions,
we do not have that much choice about how to satisfy the $\prps$
property.  We show that any acceptable-dependent loss function
satisfying the property can be obtained from the $\prploss$.

\begin{restatable}{theorem}{thmprpchar}
  \label{thm:prp_char}
  Let $\Loss$ be an acceptable-dependent function that has the $\prps$ property.
  Then there exists a function $h:\RR \times [\outdim] \to \RR$ that is continuously differentiable in its first argument such that 
  $\Loss(\probs, \mylabel) = h(\Loss_{\prplosssubscript}(\probs, \mylabel), \numallowed)$ where $\numallowed = \sum_i \mylabelscalar_i$.  
\end{restatable}

\Cref{thm:prp_char} is a central result of our work. It gives a
characterization of all acceptable-dependent functions that have the $\prps$ property,
which are the functions that avoid any systematic bias towards some of
the allowed outputs. Because this theorem is one core of our work, we add a quick intuitive proof sketch.

\begin{proof}[Proof sketch]
	The core of the argument considers an arbitrary loss function $\Loss$ with the $\prps$ property and a real value $z$, 
	and shows that on the
	set $H_z$ 
	of values where $\Loss_{\prplosssubscript}=z$, $\Loss$ is constant: thus $\Loss$ is a function of $\Loss_{\prplosssubscript}$, and once this is proven it is easy to show
	that the function is smooth.
	To prove smoothness, we first argue that $H_z$ is path connected. We then fix two points $a$ and $b$ in $H_z$  let $\gamma$
	be a path in $H_z$ connecting them, and  write $\Loss(b)-\Loss(a)$
	as a line integral of the gradient of $\Loss$ over that path.  We show that the gradient of $\Loss$ is a constant multiple of
	the gradient of $\Loss_{\prplosssubscript}$. But since $\Loss_{\prplosssubscript}$ is constant on $H_z$, hence constant on $\gamma$, its gradient must be $0$.
	We have thus shown that $\Loss(b)-\Loss(a)$ is $0$ as required.
	Details are in \Cref{app:prploss}.
\end{proof}

\begin{example}[\cref{ex:3outputs_2} continued]
  We see in \cref{fig:vectorfield_1AB_prp} that in our simple example
  of 3 possible outputs and a single sample $(\myinput, \{ A, B \})$
  the model does not necessarily converge to either $A$ or $B$, as it
  did for the $\nllloss$.  In fact, in the case of softmax regression
  (leftmost plot), the update operations strictly preserve the initial
  probability ratios between $A$ and $B$.  In the more general case,
  the output with the greater initial probability increases only
  moderately faster and only at the very end of training.
\end{example}

\subsection{Increasing Model Complexity}
\label{sec:increasingmodelcomp}

Our results about the winner-take-all property of $\nllloss$ and the
$\prp$ property of $\prploss$ assume a softmax regression
model. This can be seen in \cref{fig:winner_takes_all},
where we show learning curves during training of a single layer:
$\nllloss$ makes the model prediction collapse into a single
output, while $\prploss$ guarantees to keep to the initial
probability ratios of allowed outputs.
For more complicated networks with hidden layers, our results become approximations. 
In \cref{fig:vectorfield_1AB} we experimentally observe how
probability ratios change in a larger network. We find that the update
dynamics remain mostly unchanged as we increase the model
complexity and the training trajectories indeed do not converge towards one or the other side.

\begin{example}[\cref{ex:3outputs} continued]
\label{ex:fig4}	
  Let us return to the slightly more complex example where the same input is
  associated with two consistent label sets:  $(\myinput, \{A, B\})$ and
  $(\myinput, \{A, C\})$. The update dynamics of this setting is depicted in \cref{fig:vectorfield_1AB_1AC}.
  For both $\nllloss$ and $\prploss$, label $A$ constitutes the single
  attractor. However, in the case of $\nllloss$, reaching $A$ can take a long time when we start
  with very low probability assigned to $A$ and it can even lead to
  oscillation between $B$ and $C$ if the learning rate is not
  sufficiently small. On the other hand, $\prploss$ yields a smooth
  trajectory to $A$ from any starting configuration.
\end{example}

\begin{example}[\cref{ex:fig4} continued]
\label{ex:fig5}
We run simulations of the entire training process 
with random adversarial starting configurations
on this toy problem from \cref{ex:fig4}: 
we train the model
for 20 steps to approach the $B$-$C$ line (using sample $(\myinput, \{ B,
C \})$ and then train it for another 200 steps with the two samples 
$(\myinput, \{ A, B \})$ and $(\myinput, \{ A, C \})$. When there are only three outputs ($\outdim=3$), we
find that both losses make the model converge to $A$, although the
$\nllloss$ usually takes longer (\cref{fig:vectorfield_1AB_1AC}). 
However, as we
increase the output size $\outdim$
while keeping the samples -- i.e., we add
unrelated disallowed outputs -- bad local optima emerge and it becomes harder for
$\nllloss$ to find the optimum. When there are $\outdim=10$ possible outputs,
$\nllloss$ misses the optimum $20\%$ of the time (based on 30 trials)
and when there are 100 outputs, it misses the output $100\%$ of the
time (based on 10 trials). In the meantime, the $\prploss$ robustly
learns to select output $A$ that satisfies both
points. \Cref{fig:winner_takes_all_interaction} shows typical
learning curves for 100 outputs.
\end{example}

\begin{figure}[thb]
  \centering
  \begin{subfigure}[b]{0.32\textwidth}  
	  \centering
	  \includegraphics[width=0.95\textwidth]{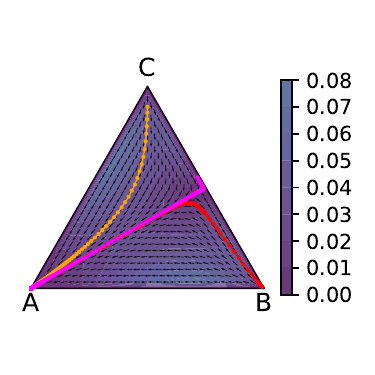} 
	  \vspace{-4mm}
	  \caption{$\nllloss$, $LR=1$}
	  \label{fig:vectorfield_1AB_1AC_a}
  \end{subfigure}
  \begin{subfigure}[b]{0.32\textwidth}
	  \centering
	  \includegraphics[width=0.95\textwidth]{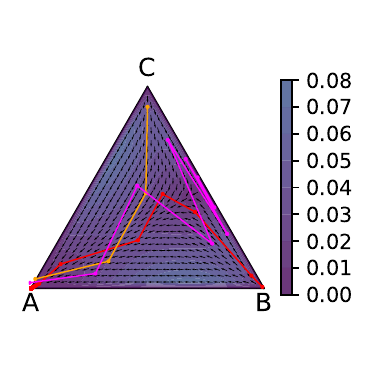} 
	  \vspace{-4mm}
	  \caption{$\nllloss$, $LR=10$}
	  \label{fig:vectorfield_1AB_1AC_b}
  \end{subfigure}
  \begin{subfigure}[b]{0.32\textwidth}  
	  \centering	  
	  \includegraphics[width=0.95\textwidth]{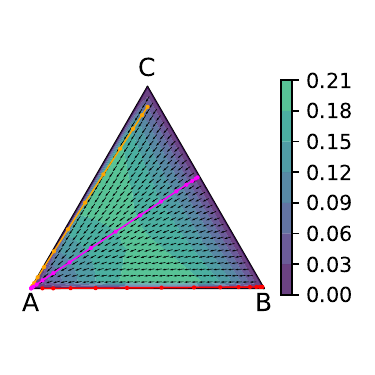} 
	  \vspace{-4mm}	  
	  \caption{$\prploss$, $LR=1$}	
	  \label{fig:vectorfield_1AB_1AC_c}
  \end{subfigure}

  \caption{\Cref{ex:fig4}: Given $\outdim=3$ outputs $A$, $B$, $C$ and two samples $(\myinput,
    \{ A, B \})$, $(\myinput, \{ A, C \})$, we show how softmax regression
    updates the probabilities. {\bf (a):} $\nllloss$ with small learning rate, {\bf (b):} $\nllloss$ with large learning rate, {\bf (c):} $\prploss$.
    $\nllloss$ can lead to oscillation, or
    may take a long time to reach the attractor.
    In contrast, $\prploss$ heads directly towards the attractor.
    We show three real trajectories from fixed starting points:
    $(\probsscalar_A, \probsscalar_B, \probsscalar_C) \in \{(0.003,
    0.99, 0.007), (0.05, 0.05, 0.9), (0.01, 0.44, 0.55)\}$. The
    trajectories are terminated when $\probsscalar_A > 0.9999$.
  }
  \label{fig:vectorfield_1AB_1AC}
\end{figure}

\begin{figure}[thb]
  \centering
  \begin{subfigure}[b]{0.47\textwidth}  
	  \centering
	  \includegraphics[width=0.95\textwidth]{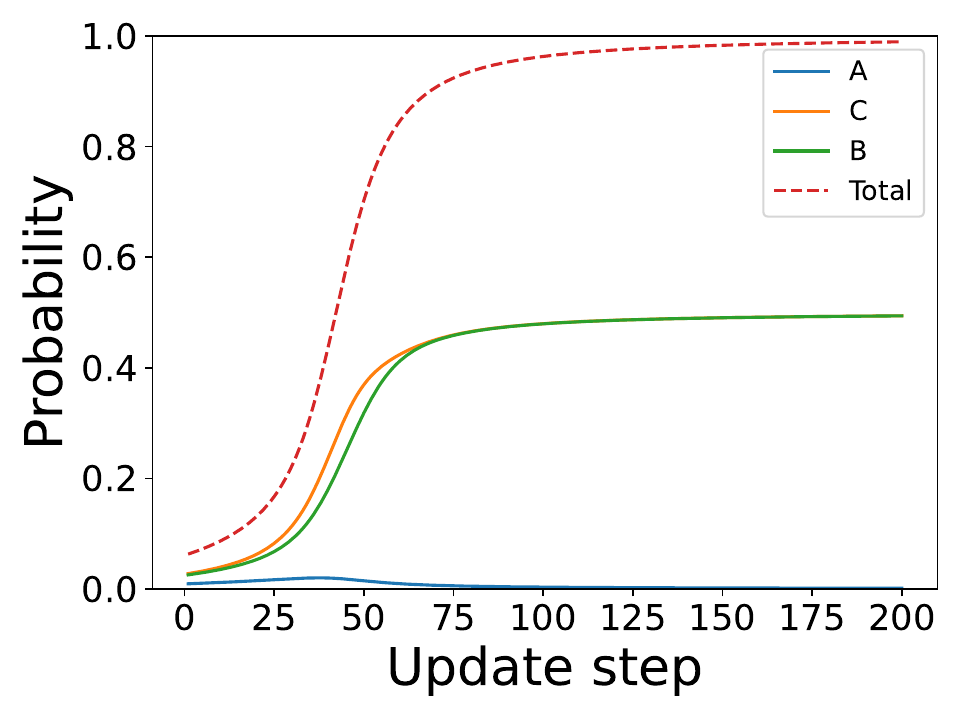} 
	  \caption{$\nllloss$}
	  \label{fig:winner_takes_all_interaction_a}
  \end{subfigure}    
  \begin{subfigure}[b]{0.47\textwidth}  
	  \centering
	  \includegraphics[width=0.95\textwidth]{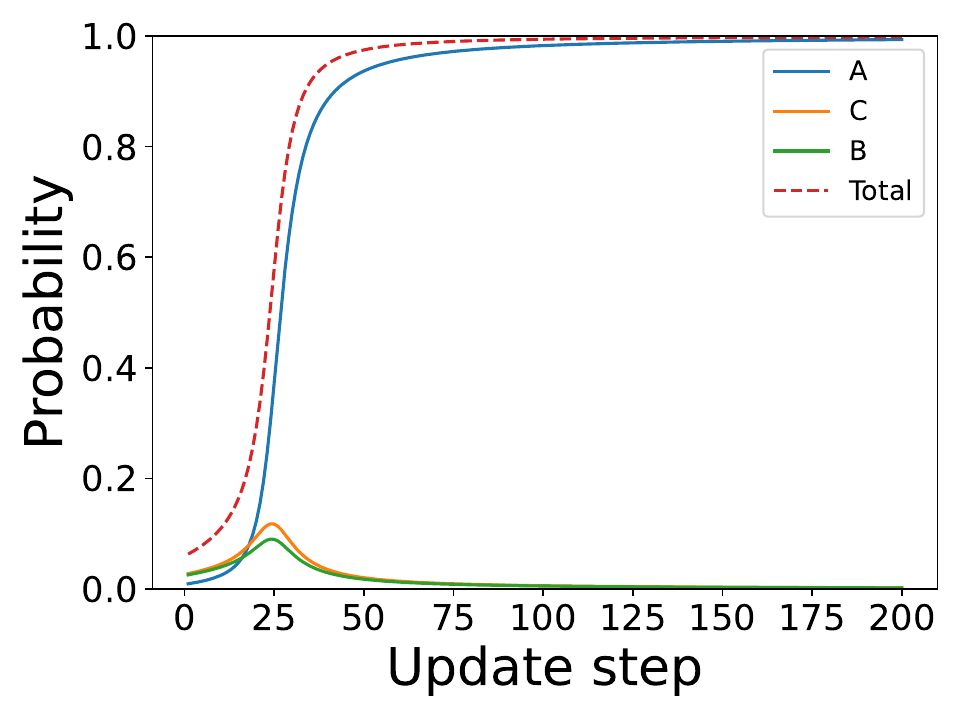} 
	  \caption{$\prploss$}	
	  \label{fig:winner_takes_all_interaction_b}
  \end{subfigure}
  \caption{
    \Cref{ex:fig5}:
    Learning curves using a classifier MLP with two layers and $\outdim=100$ outputs on a
    training set that consists of two samples: $(\myinput,\{ A, B
    \})$, $(\myinput, \{ A, C \})$, as described in
    \cref{ex:3outputs}. While $A$ is the optimal output, initially,
    $B$ and $C$ have higher probabilities.
    \emph{Total} shows $A+B+C$, i.e., the sum of probabilities of outputs in occurring in some label set.
	Starting from identical initialization, $\nllloss$ (a) gets stuck in a bad local optimum, while $\prploss$ (b) quickly recovers the global optimum.
    }
  \label{fig:winner_takes_all_interaction}
\end{figure}

\subsection{Connection with negative log likelihood loss}

A closer inspection of the $\prploss$ \cref{def:libraloss} 
reveals that it is a combination of
several different log likelihood losses.  The disallowed term
$\log(1 - \sum_i \mylabelscalar_i \probsscalar_i)$ is the \emph{positive} log
likelihood of selecting a disallowed output.
It has the same
monotonicity and optimum as the $\nllloss$ \cref{eq:NLL}. 
However, what is different
is its convexity: the more the model fits to a sample
(i.e.\ the
higher the sum of allowed probabilities $\sum_i \mylabelscalar_i \probsscalar_i$),
the flatter the $\nllloss$ becomes. On the other hand, $\log(1 - \sum_i
\mylabelscalar_i \probsscalar_i)$ becomes steeper as we start fitting  the
sample. This curvature, however, is compensated by the allowed term
$-\frac{1}{k} \sum_i \mylabelscalar_i \log(\probsscalar_i)$, which is the average
of the individual negative log likelihood losses for each allowed
output.

In general, loss components that reward the log probability of allowed
outputs (such as $\nllloss$ or the allowed term of $\prploss$) will
have vanishing gradients when we are close to fitting the allowed
labels. Analogously, loss components that penalize the log probability
of disallowed outputs (such as the disallowed term in $\prploss$) will
have vanishing gradients when we are far from fitting the allowed
labels.  $\prploss$ provides a ``perfect'' balance between these two
kinds of components.  The gradients are stable throughout the
optimization -- we show in \Cref{app:prploss} that they are always
$-\frac{1}{k}$ for the allowed outputs.

We visualize the balancing effect of the $\prploss$ in \Cref{fig:prp_vanilla} for the classical
supervised case. That is, when there is a single allowed
output, thus $k=1$. In this case $\prploss$ reduces to the \emph{log odds ratio}. 
Let $\trueprob$ 
denote the probability of the single allowed output, and thus
$\trueprob = \sum_{i} \mylabelscalar_i \probsscalar_i$.
Then the
$\prploss$ becomes
$$
	\Loss_{\prplosssubscript} =
	\log \left( \frac{1-\trueprob}{\trueprob} \right) =
	\underbrace{\log(1-\trueprob)}_{\textrm{positive log likelihood of disallowed}} + 
	\underbrace{-\log(\trueprob)}_{\textrm{negative log likelihood of allowed}}
$$

The derivative of the loss with respect to the single allowed logit
$\truelogit$ is:
\begin{align*}
  \pd{\Loss_{\prplosssubscript}}{\truelogit} & =
  \pd{\log(1-\trueprob)}{\truelogit} + \pd{-\log(\trueprob)}{\truelogit}
  = \left( \frac{-1}{1-\trueprob} + \frac{-1}{\trueprob} \right) \pd{\trueprob}{\truelogit}
  \\ & = \frac{-\trueprob - 1 + \trueprob}{\trueprob (1-\trueprob)} \trueprob (1-\trueprob)
  = -\trueprob -1 + \trueprob = -1
\end{align*}

We obtain that the $\prploss$ is linear in $\truelogit$,
with constant derivative $-1$.

\begin{figure}[htb]
  \centering
  \includegraphics[width=0.45\textwidth]{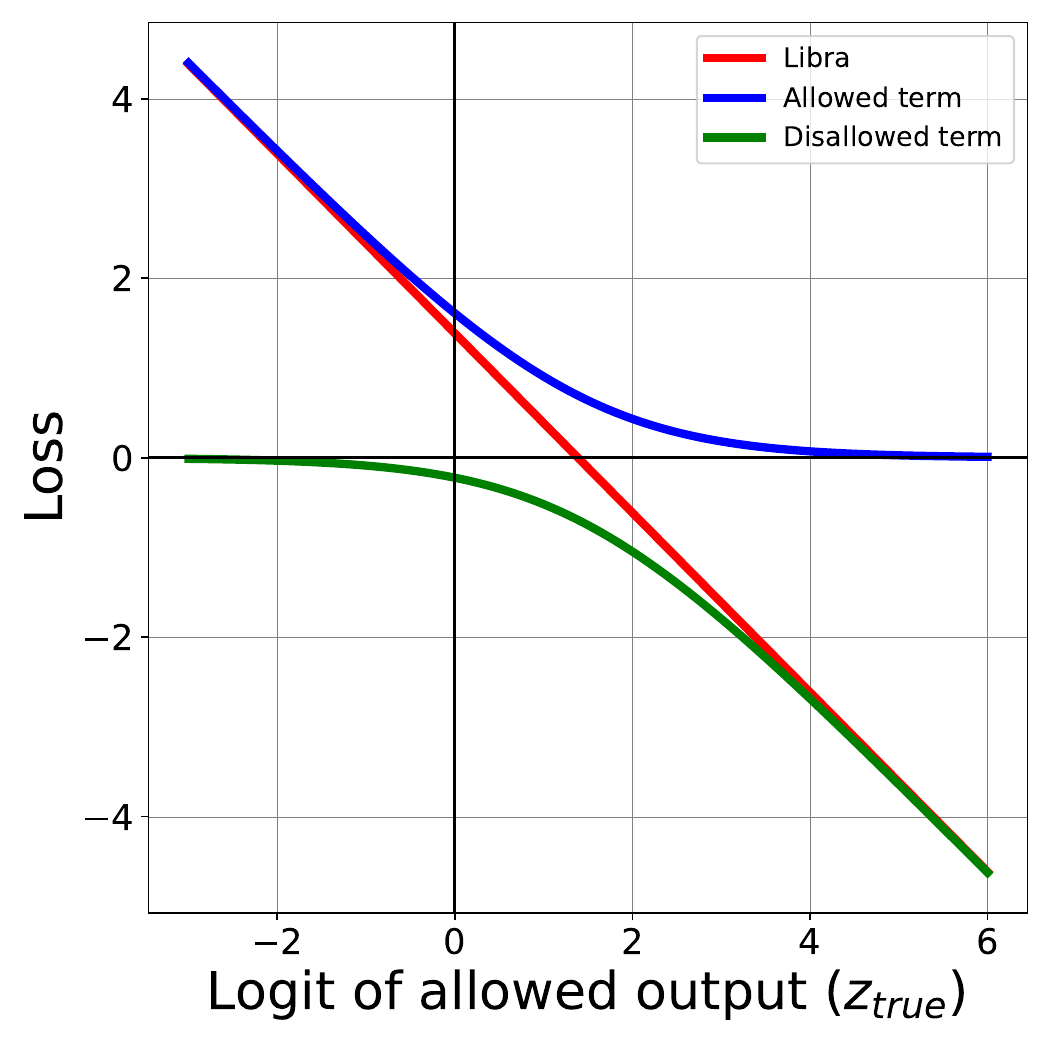}
  \caption{$\prploss$ and its two terms when there is a single allowed
    output, plotted against the single allowed logit $\truelogit$. 
	The
    derivatives of the two terms add up to $-1$ (i.e., the gradient of $\prploss$ is
    independent from the model prediction). This implies the $\prp$
    property. {\bf Blue line (Allowed term)}: negative log likelihood
    of the allowed output ($-\log(\trueprob)$, $\nllloss$), {\bf Green line
      (Disallowed term)}: positive log likelihood of the disallowed
    outputs ($\log(1-\probsscalar_0)$), {\bf Red line ($\prploss$)}: sum of the
    allowed and disallowed terms.
    }
  \label{fig:prp_vanilla}
\end{figure}

\subsection{Connection with entropy regularization}
The $\prploss$ implements a special, input-dependent form of entropy
regularization~\citep{confidence_penalty}, whose intuitive goal is to penalize distributions with low entropy.
As we have already seen, the
first loss term $\log(1-\sum_i \mylabelscalar_i \probsscalar_i)$ aims to minimise the likelihood of disallowed labels.
The allowed loss term can be rewritten as
\begin{align*}
  - \frac{1}{k} \sum_i \mylabelscalar_i \log(\probsscalar_i)
  & = - \frac{1}{k} \sum_i \mylabelscalar_i \left(\log(\probsscalar_i) 
  - \log \left(\frac{1}{k}\right) + \log \left(\frac{1}{k}\right) \right) = \\
  & = \sum_{\{i| \mylabelscalar_i=1\}} \frac{1}{k} \log \left(\frac{\frac{1}{k}}{\probsscalar_i} \right) 
  - \sum_{\{i| \mylabelscalar_i=1\}} \frac{1}{k} \log \left(\frac{1}{k} \right) = \\
  & = D_{\textup{KL}}(U_{\mylabel} \mid\mid \probs) + H(U_{\mylabel}))
  = H(U_{\mylabel}, \probs)
\end{align*}
where $D_{\textup{KL}}(\boldsymbol p \mid\mid \boldsymbol q) = \sum_i
p_i \log\left(\frac{p_i}{q_i}\right)$ is the Kullback-Leibler
divergence of distribution $\boldsymbol p$ from reference distribution
$\boldsymbol q$, $U_{\mylabel}$ is the uniform distribution over the
$\numallowed$ allowed outputs, $H(\boldsymbol p)$ is the entropy of
$\boldsymbol p$ and $H(\boldsymbol q, \boldsymbol p)$ is the cross
entropy of $\boldsymbol p$ relative $\boldsymbol q$. This rewriting
shows that the allowed term is a cross entropy loss, measuring the
distance between $U_{\mylabel}$ and the model output distribution
$\probs$. Minimising this term is equivalent to entropy regularization
(i.e.\ maximising entropy), restricted to the allowed outputs.  In
other words, it is minimal when $\probs$ is uniform on the allowed
outputs and zero elsewhere.

\subsection{Preserving both acceptable and unacceptable inputs}
\label{sec:biprploss}

The reader may have noticed that the $\prp$ property
requires that the loss is acceptable-dependent (\Cref{def:acceptable_dependent}).
It enforces constraints which concern preservation of  ratios between
outputs, but it does this only on the acceptable outputs.
It is  natural to drop the first requirement, 
allowing dependence on all outputs, but
replacing the constraints with a  stronger property that is symmetric in acceptable and unacceptable
outputs.
We give this analog of the $\prp$ property below:

\begin{definition}[$\biprp$ property]
  \label{def:biprp-property}
  Given a parametric model $\classifier_{\params}$,
  a continuously differentiable function
  $\Loss(\probs, \mylabel)$ is said to satisfy the
  \emph{$\biprp$ property} for $\model$ if any
  Gradient-update with loss function $\Loss$ preserves the ratio of probabilities
  of all outputs $i$ with $\mylabelscalar_i=1$, and also the ratio of probabilities of outputs with $\mylabelscalar_j=0$.
\end{definition}

As before, we focus on the $\biprp$ property for a softmax regression model:

\begin{definition}[$\biprps$ property]
  \label{def:biprps-property}

  A continuously differentiable function $\Loss(\probs, \mylabel)$ is
  said to satisfy the \emph{$\biprps$ property} if it satisfies the $\biprp$
  property for the softmax regression model.
\end{definition}

Again, we demonstrate that the property is not vacuous.
We define a loss function that performs ``balancing'' on both the acceptable and unacceptable loss. Contrasting
with the $\prploss$, we call this the \emph{Sagittarius loss}, abbreviated $\biprploss$.
\begin{definition}[$\biprploss$]
  \label{def:biprp-loss}
  $$
  \Loss_{\biprplosssubscript}(\probs, \mylabel) =
  \underbrace{\frac{1}{n - \numallowed} \sum_i (1-\mylabelscalar_i) \log(\probsscalar_i)}_{\textup{Disallowed term}}
  \underbrace{- \frac{1}{\numallowed}\sum_i \mylabelscalar_i \log(\probsscalar_i)}_{\textup{Allowed term}} + 
  $$
  The first term is the average of the individual positive log likelihood losses for each disallowed output. The second term -- which is identical to that of $\prploss$ -- is the average of the individual negative log likelihood losses for each allowed
  output. Also notice that both terms can be seen as cross entropies of $\probs$ relative to uniform distributions on the 1) allowed outputs (allowed term) and 2) disallowed outputs (disallowed term).
\end{definition}

We can show that $\Loss_{\biprplosssubscript}$ has the $\biprps$ property:

\begin{restatable}{theorem}{thmbiprploss}
  \label{thm:biprp_loss}
  The $\biprploss$ function has the $\biprps$ property and for any
  continuously differentiable family of $h_i:\RR\to \RR$ functions
  $\Loss(\probs, \mylabel)=h_k(\Loss_{\biprplosssubscript}(\probs,
  \mylabel))$ also satisfies the $\biprps$ property, where $k
  = \sum_i \mylabelscalar_i$.
\end{restatable}

  Furthermore, 
we get a characterization analogous to the one
of \cref{thm:prp_loss}.
\begin{restatable}{theorem}{thmbiprpchar}
\label{them:biprp_char}
Let $\Loss$ be a function that has the $\biprps$ property,
  invariant under the permutation of the input (i.e., $\forall \pi\in
  S_n, \Loss(\pi\circ \probs, \pi\circ \mylabel) = \Loss(\probs, \mylabel)$).  Then there
  exist $h_i:\RR\to \RR$ continuously differentiable functions such
  that $\Loss(\probs, \mylabel) = h_k(\Loss_{\biprplosssubscript}(\probs, \mylabel))$.
\end{restatable}

The proofs are provided in \Cref{app:biprploss}.

\subsection{Comparing $\prploss$ and $\biprploss$}

The $\prploss$ and the $\biprploss$ have many similarities and are
strongly related to $\nllloss$. They both factorize into an allowed
and a disallowed term, and the allowed terms are identical: the cross
entropy of $\probs$ relative to the uniform distribution on the
allowed outputs, which is also the average of the individual negative
log likelihood losses for each allowed output.  They differ in the
disallowed term. For $\prploss$ it is the positive log likelihood of
selecting a disallowed output, while for $\biprploss$ it is the cross
entropy of $\probs$ relative to the uniform distribution on the
disallowed outputs, or equally the average of the individual positive
log likelihood losses for each disallowed output.

The $\biprp$ property implies the $\prp$ property, hence the
$\biprploss$ satisfies the $\prps$ property. 
At this point the reader
may expect that the $\biprploss$, having a stronger property, should
be superior to $\prploss$.  Surprisingly, we will explain in 
\cref{sec:experiments} that this is not the case: the need to retain
balance on both acceptable and unacceptable outputs leads to some
undesirable effects. In particular, the magnitude of the logit vector
increases rapidly during learning, leading to numerical instability.

\section{Learning Mapping Rules via Partial Label Learning}
\label{sec:obda}

We introduce new sequential datasets with disjunctive
supervision. Extending Example \ref{ex:running}
in the introduction, these dataset  will
concern  learning rules, a topic that has gained considerable
interest in the AI community,
e.g. \citet{imperialrules,uclrules,rnnlogic}. More specifically, we
consider learning \emph{mapping rules} which relate data sources in a
\emph{source vocabulary} into some \emph{target vocabulary}. This is a
common approach in data integration, where the target vocabulary is
often standardized (an ``ontology'' \cite{owl}), optionally equipped
with additional logical constraints.  Although there are a vast number
of tools available for answering queries with known rules, determining
the mapping rules by hand is known to be a difficult even with domain
expertise \citep{rodi}. Thus a key challenge is to learn the mapping
rules from supervision on the target vocabulary -- we know some tuples
$\vec t$ that should or should not be inferred in the target
vocabulary, called \emph{positive} and \emph{negative} facts.

\begin{example} [\cref{ex:running} continued]
  \label{ex:obda}
Consider the setting with source relation
\verb|Person|, target relation \verb|Author|, and facts:
$Person(alice, 45, 1)$, $Person(bob, 34, 1)$, $Person(joe, 23,
2)$, $Person(lola, 12, 2)$. Supervision might consist of:

\begin{tabular}{l l}
  Positive: & \verb|Author(alice)|, \verb|Author(bob)| \\
  Negative: & \verb|Author(joe)|, \verb|Author(lola)|
\end{tabular}
\end{example}

In ~\Cref{ex:obda}, we are looking for mapping rules between source relation
\verb|Person| and target relation \verb|Author| that allow for
deducing that Alice and Bob are authors, and that cannot be used to
prove that Joe and Lola are authors.

The number of possible mapping rules is generally large, much larger
than what can be enumerated. In data integration, there are typical
patterns in which the source and the target may differ, and when
domain experts construct mappings by hand, they tend to try these
typical patterns to find the one that fits the task at hand. We
formalise these patterns as \emph{mapping templates}, e.g.:
$$T(x) \leftarrow S_1(x) \land S_2(x,y) \land S_3(y)$$ where variables
in the head (the variable $x$ above) are universally quantified and
the rest of the variables (the variable $y$ in the example above) are
existentially quantified. $T$ and $S_1, S_2, S_3$ are template
variables over predicate names in the target and source language,
respectively. Any instantiation of template variables yields a mapping
rule. Such templates are assumed in most prior work in the area,
e.g. \cite{imperialrules,uclrules}.
In particular, we support mapping templates of the form
\[
H(x_1 \ldots x_k)
\datalogarrow \bigwedge_{i \leq b} C_i(\vec y_i), \bigwedge_{j \leq k}
x_i =\tau_i
\]
Mapping rules are formed by replacing template variable $H$ by a
target predicate, template variables $C_i$ by source predicates and
variables $y_i$ by either variables or source constants. The terms
$\tau_i$ are formed from applying string concatenation to either
variables or strings. 

\begin{example}
  To illustrate the usage of string concatenation, we provide a real mapping rule from the NPD challenge (to be described below).
  
\begin{align*}
  & \mathrm{Agent}(x) \datalogarrow \\
  & C_1(y_1) \land C_2(y_2) \land C_3(y_3) \land C_4(y_4) \land \\
  & x = \textrm{CONCAT}(\textrm{http://sws.ifi.uio.no/data/npd-v2/baa/}, y_1, \textrm{/licensee/}, y_2, \textrm{/history/}, y_3, \textrm{/}, y_4)
\end{align*}

This mapping rule aligns target concept $\textrm{Agent}$. In the
target language, agents are represented as URL strings.  Components of these strings
are fixed for all agents, such as the prefix
$\textrm{http://sws.ifi.uio.no/data/npd-v2/baa/}$. Other parts
are derived from four source predicates $C_1, C_2, C_3, C_4$. The
source predicates correspond to columns in the database -- we omit
their description in the example.
\end{example}

We assume that mapping templates $\mt$ are provided by domain experts.
Our task is to find a subset $\m$ of the instantiations of the
templates such that the source database instance $I_{\source}$ and the
mapping rules $\m$ together imply all the positive facts $\possup$ and
none of the negative facts $\negsup$. When an exact solution is not
achievable, we can also consider a relaxation of the problem, i.e., we
want to cover ``as many as possible'' of the positive facts and ``as
few as possible'' of the negative facts.

The number of possible rules is infinite, due to the number of
possible strings in concatenation terms. However, we will only be
interested in the rules that can produce a given target fact. For any
positive or negative fact $F$ and mapping template $MT \in \mt$, we
define the \emph{candidates} of $F$ with respect to $MT$ to be the set
of all instantiations $M$ of $MT$ such that $M$, together with the
source instance implies $F$:

$$\mathrm{candidates}(F, MT) = \{M | M \in MT, (M \land I_{\source} \vDash F)\}$$ 

Taking into account that the source database is finite,
$\mathrm{candidates}(F, MT)$ is a finite set and typically small enough
so that it can be obtained via preprocessing.

\begin{example}[\cref{ex:running} continued]
  \label{ex:obda2}
In our example, the rules that derive \verb|Author(alice)|
are: $R_0 = Author(x) \leftarrow \exists a, t. Person(x, a, t)$,
$R_1 = Author(x) \leftarrow \exists a. Person(x, a, 1)$, $R_2 =
Author(x) \leftarrow \exists t. Person(x, 45, t)$. $R_0$ and $R_1$
also derive \verb|Author(bob)| while $R_2$ does not. We obtain the
following candidate sets:

\begin{tabular}{l l | l l}
  \verb|Author(alice)| & $\{R_0, R_1, R_2\}$ &
  \verb|Author(bob)| &  $\{R_0, R_1, R_3\}$ \\
  \verb|Author(joe)| & $\{R_0, R_4, R_5\}$ &
  \verb|Author(lola)| & $\{R_0, R_4, R_6\}$
\end{tabular}
\end{example}

In \Cref{ex:obda2}, each fact has three candidates and we have seven rules in total. R0
proves all facts, R1 proves all positives and none of the negatives,
R2 and R3 prove some of the positives, R4 proves all the negatives, R5
and R6 prove some negatives. Clearly, R1 is the optimal choice as a
single rule.

Let us consider a function $\classifier_{\params}: \mathrm{fact}
\rightarrow \mathrm{rule}$ that assigns to each fact a correct mapping
rule. Approximating this function via learning can greatly reduce the
labor cost of data integration. Given a set of positive facts
$\{\possup^{(i)}\}_{i=1}^{\nsamples_p}$ we can compute the corresponding
candidate rule sets $\{\mathrm{candidates}(\possup^{(i)},
MT)\}_{i=1}^{\nsamples_p}$, which together constitute a partially labelled
dataset for learning $\classifier$.

Analogously, we can use negative facts
$\{\negsup^{(i)}\}_{i=1}^{\nsamples_n}$ to extract a \emph{negative}
partially labelled dataset for learning $\classifier$.  Negative supervision
represents global constraints and requires special treatment. Given a
negative sample $(\myinput, \mylabel)$, the labels in $\mylabel$ are
explicitly forbidden for \emph{any} input. Theoretically, this is
equivalent to a partial labelling that excludes globally these
outputs, however, producing complementer sets of forbidden label sets
can be problematic in practice when the output space is large. Let
$A_{neg} = \{\mylabel | (\myinput, \mylabel) \mbox{ is a negative example} \}$ be
the set of all label sets that appear in some negative example. Given
a loss function $\Loss$ for positive disjunctive supervision, we introduce a new loss
term $\Loss_{neg}(\probs) = \sum_{\myinput \in A_{neg}'} -\Loss(\probs, \myinput)$ where $A_{neg}' \subseteq A_{neg}$ is $50$ samples selected uniformly at random from $A_{neg}$ for each update step. This term quantifies the extent to which negatives are violated and it is,
weighted by a hyperparameter $\gamma$\footnote{$\gamma$ represents the
tradeoff between fitting to positive and negative datapoints.}
added to the loss function:
$$\Loss'(\probs, \myinput) = L(\probs, \myinput) + \gamma \Loss_{neg}(\probs)$$

Recall that the input space is the set of all possible atoms
expressible in the target language, while the output space is all
possible mapping rules. 
Although there are only finitely many options when conditioned on the
supervision and the source database -- when we only consider
mapping rules that derive some fact -- even in this case the output
space remains huge. It can easily reach hundreds of thousands of
rules. Directly training a model with so many outputs is challenging
and such an approach would neglect similarities across rules. For this
reason, we instead represent inputs and outputs as text, i.e., as
sequences of tokens, yielding a sequence-to-sequence language
modelling task with disjunctive supervision. As discussed in
\Cref{sec:prelims}, autoregressive models can be used to model
problems with sequential outputs: model predicted probabilities
$\probs$ can be calculated in a sequence of evaluations. Consequently,
any loss function that takes $\probs$ and label $\mylabel$ as input
can be applied directly, without modification for optimization,
independent of the architecture.  In the following we describe the
novel datasets that we extracted from the \cite{rodi} benchmark and
that are used in the experiments presented in
\Cref{subsec:rulelearningexp} to compare various loss functions.

\subsection{RODI Challenges} \label{subsec:rodi}

The RODI dataset was introduced in \cite{rodi,rodi2} as a benchmark for
systems that integrate a set of source relational schemas into a
target graph schema.  Each challenge provides a target schema
consisting of unary and binary relations and a source relational
database. The task is to find mapping rules that define concepts in
the target using query expressions over the source database.

The challenges are synthetically generated starting from an
instance of the target schema, generating a source schema. The target
schema consists of binary relations (\emph{properties}) and unary
relations (\emph{classes}).  The source schema generation involves one or
a combination of typical -- real life inspired -- distortions that
make the alignment nontrivial. For competition purposes, RODI
provides the target schema (without data), the source data, and a list
of \emph{translation pairs} (source query, target query) that can be
used for evaluation. In each pair one is a SPARQL~\citep{sparql} query against the
target and the other is an SQL query against the source database. In case of
correct mapping, the two queries have to return the same result.  The
target schemas (\emph{ontologies}) are based on three conference management systems: CMT,
SIGKDD and CONFERENCE. RODI uses the distortions described in \Cref{tab:distortions} (see \cite{rodi} for more details):

\begin{table}
  \label{tab:distortions}
\caption{Synthetic distortions applied to make the alignment task harder.}
\small
\begin{tabular}{p{0.15\textwidth} p{0.75\textwidth}}
  {\bf Distortion} & {\bf Description} \\
  \toprule
  renaming & Classes and properties have different names in the
  ontology and the database \\
  cleaning & Foreign keys in the database are removed,
  making it harder to join tables. \\
  restructuring & Class hierarchies are represented using
  attributes indicating subclass membership. \\
  denormalising & Correlated information is jointly stored in the same
  table, redundantly.\\
  \bottomrule
\end{tabular}
\end{table}

For each predicate of each challenge, we sample $\nsamples$ positive
tuples that satisfy the predicate and $\nsamples$ negative tuples that
do not satisfy it. The positive tuples are sampled uniformly from the
tuples returned by the provided SQL query for that predicate. For
sampling negatives, we use random constants for each tuple position,
selected uniformly from the constants of the database with matching
type and ensuring no overlap with the positives.

We obtain 5 datasets for each domain (one without distortion and four
with one of the above distortions) that contain 1500-2000 positive
samples and a maximum of 55 candidates for each input. We find that
the different domains yield no new insights and preliminary
experiments suggest similar performance. Hence, we focus on the CMT
system and experiment with the 5 challenges associated with it in
\Cref{subsec:rulelearningexp}. Our distribution contains the extracted
CMT datasets, as well as code to generate datasets for any domain.

\subsection{NPD Challenge}
Besides the synthetically generated challenges, \cite{rodi} provide a
real world dataset related to the Norwegian Petroleum Directorate
(NPD) FactPages~\citep{npd}. The source data and the target schema
were constructed from publicly available data and the translation
pairs were built from real use cases from end users of the
FactPages. The source database contains ~40MB data and has a rather
complex structure with 70 tables, ~1000 columns and ~350 foreign
keys. The target schema has ~300 classes and ~350 properties.
Existing tools (e.g. \cite{bootox,incmap}) for this task rely
completely on the structure of the source and target, and are unable
to infer any relationships in a challenge like this.

Positive facts are sampled uniformly, just like for RODI. For sampling
negatives, however, we find that uniform sampling yields facts that
have extremely small probability of being provable by the rules
required to prove positives, making it rather easy to avoid
negatives. This is because the rules required to align NPD are much
more complex than those for RODI. For this reason, negative tuples are
sampled uniformly not from the entire database, but only from
constants appearing in positive tuples of other predicates. We
observe that this way of sampling negatives makes aligning NPD
harder, since many of the candidates of positive facts have to be
eliminated as they also prove some of the negatives.

We end up with a dataset consisting of $34965$ positive facts,
using $421$ target predicates. Over $98\%$ of the facts have less than
$1000$ candidates and we truncate the set of allowed candidates to
$1000$ for computational reasons. Our distribution includes the
extracted dataset, as well as code to generate a new dataset.

\section{Experiments}
\label{sec:experiments}

Our experiments aim to provide a quantitative overview of how
different loss functions perform on learning from partially labelled
data both in the PLL and DS settings, as well as to demonstrate the
practical benefit of the newly introduced $\prploss$. We employ three
types of datasets:

\begin{compactenum}
\item Synthetic inputs, synthetic outputs (PLL): These experiments,
  presented in \Cref{subsec:synexp}, examine extremely simple scenarios
  aimed at highlighting failure cases of various loss functions.
\item Real inputs, synthetic outputs (PLL): This is the setup typically used
  to evaluate PLL methods in the literature. We present two experiments
  in \Cref{subsec:realsynexp} based on the CIFAR10 and CIFAR100 datasets.
\item Real inputs, real outputs (PLL and DS): This is the most challenging and most
  important scenario. We experiment with a novel rule learning dataset
  for DS in \Cref{subsec:rulelearningexp}, as well as a collection of
  standard benchmarks for PLL in \Cref{subsec:pll_real}.
\end{compactenum}

Before moving on to the experiments, we provide an overview of the
loss functions from the literature that we use as competitors in
\Cref{subsec:competitors}. We end the section with a discussion of the
results in \Cref{subsec:discussionexp}.

\subsection{Competitors}
\label{subsec:competitors}
We overview the alternative approaches from the literature we compete with in the following experiments.

\introparagraph{Negative Log Likelihood loss (NLL)}
The $\nllloss$, defined in \Cref{sec:prelims} is the standard example of an average-based loss that
appears in the literature, often under different names. For example, it
is called the \emph{maximum marginal likelihood (MML) } loss in
\cite{guu-etal} and the \emph{classifier consistent (CC)} loss in
\cite{provablyconsistentpll}. We repeat the definition:
$$
\Loss_{\nll}\left(\probs, \mylabel \right) 
	= - \log \left(\probs \cdot \mylabel \right) 
	= - \log \left( \sum_i \mylabelscalar_i \cdot \probsscalar_i \right)
$$

\introparagraph{Uniform loss}
A very simple baseline is to compute the negative log likelihood of each allowed output and optimize their sum:
$$
\Loss_u(\probs, \mylabel)  = - \sum_i \mylabelscalar_i \log(\probsscalar_i)
$$
This is an average-based method
and it differs from the disallowed term of the $\prploss$ only by a multiplicative factor of $\frac{1}{\numallowed}$.
This loss has a single optimum, when the prediction is
uniform on the allowed outputs and zero elsewhere.  We refer to this
as $\uloss$.

\introparagraph{$\beta$-Meritocratic loss}
Recall that \cite{guu-etal} consider the semantic parsing application of  DS, overviewed in
\cref{ex:semparse}.
They propose the $\bmeritloss$:
$$
\Loss_{\bmerit}(\probs, \mylabel) = - \sum_i w(\beta)_i \log(\probsscalar_i),
$$ where each output $i$ is associated with a weight $w(\beta)_i =
\frac{\mylabelscalar_i \cdot \left( \probsscalar_i / P_{acceptable}
  \right)^\beta}{\sum_q \mylabelscalar_q \cdot \left( \probsscalar_q /
  P_{acceptable} \right)^\beta} $ with $P_{acceptable} = \sum_j
\mylabelscalar_j \probsscalar_j$.  A technical caveat is that the
dependence of $\vec w$ on the model output is disregarded during
optimization, i.e., no gradients are propagated through it. This holds
for $\vec w$ in all other loss functions inroduced below.

Notice, that the $\beta$ parameter provides one possible smooth
interpolation between two losses: $\nllloss = - \log \left( \sum_i
\mylabelscalar_i \probsscalar_i \right)$ and $\uloss = - \sum_i
\mylabelscalar_i \log(\probsscalar_i)$.  More specifically, the
$\bmeritloss$ has the same gradient as $\nllloss$ when $\beta=1$,
since the denominator of $w(\beta)_i$ becomes $1$ thus can be ignored.
On the other hand, $\bmeritloss$ with $\beta=0$ is equivalent to
$\uloss$, which is minimized where the entropy on the acceptable
outputs is maximal, i.e., when each of the $k$ acceptable outputs has
probability $\frac{1}{k}$. All three losses focus solely on the
probabilities of the acceptable outputs, since $w(\beta)_i=0$ when
$y_i = 0$.  \cite{guu-etal} observe that while there is no universal
$\beta$ across datasets, tuning this hyperparameter can greatly
increase convergence speed and slightly improve final accuracy. In our
experiments, we report the extreme values as $\uloss$ and $\nllloss$
and let $\bmeritloss$ refer to the best performing $\beta$ for the
given task from the set $\{0.25, 0.5, 0.75\}$.

Note that this interpolation is similar in spirit to the
$\prploss$, which has two terms: one similar to $\nllloss$ 
and has the winner-take-all property, while the other is an entropy
regularizer and pushes the probabilities towards uniform
distribution. What is different is that $\prploss$ does not have an
extra $\beta$ parameter: the strength of the two loss terms depends
implicitly on how well the model fits to the sample. In particular, $\prploss$ 
is similar to $\nllloss$ when most of the probability mass has accumulated on the 
acceptable outputs, which happens towards the end of training. At the beginning 
of training, however, the entropy regularizer term has a stronger effect.

\introparagraph{Leverage-weighted loss (LW)} \cite{pllleveraging}
introduce \emph{leverage weighted} loss, as a family of loss
functions based on the unnormalized model outputs or logits $\logits$ and focus in
particular on the following loss:

$$\Loss_{LW}(\logits, \mylabel) = \sum_i \mylabelscalar_i w_i \sigmoid(\logitsscalar_i) + \beta \sum_i (1-\mylabelscalar_i) w_i \sigmoid(-\logitsscalar_i)$$

where $\sigmoid(t) = \frac{1}{1+e^t}$. The loss has two terms, one for
allowed outputs ($\mylabelscalar_i = 1$) and one for disallowed outputs
($\mylabelscalar_i = 0$) and the leverage hyperparameter $\beta$ controls
their relative importance. The authors achieve best empirical results
with $\beta=1$ most of the time and sometimes with $\beta=2$. The
results presented in our experiments use the best performing value
from $\{0.5, 1, 2\}$, which turns out to be $\beta=1$ in all cases.

Each output $i$ is associated with an input dependent
weight $w_i$, which is defined as the likelihood assigned to the
output by the model, normalized so that weights for allowed and
disallowed outputs both add up to one:

$$w_i = \begin{cases}
  \frac{e^{\logitsscalar_i}}{\sum_j \mylabelscalar_j e^{\logitsscalar_j}} ~& \text{if} ~ \mylabelscalar_i = 1 \\
  \frac{e^{\logitsscalar_i}}{\sum_j (1-\mylabelscalar_j) e^{\logitsscalar_j}} ~& \text{if} ~ \mylabelscalar_i = 0 \\
\end{cases}$$

This is a typical identification-based loss: the model predicted $w_i$
values are used to ``identify'' how much an allowed/disallowed output
should be rewarded/penalized for fitting. We refer to this as
$\lwsloss$.

\introparagraph{Risk-consistent loss (RC)}
A similar identification-based approach is provided in \cite{provablyconsistentpll}, using loss function

$$\Loss_{RC}(\probs, \mylabel) = - \frac{1}{2} \sum_i \mylabelscalar_i w_i \log(\probsscalar_i)$$

For each allowed output $i$ the negative log likelihood loss ($-\log(\probsscalar_i)$) is weighted by 

$$w_i = \frac{\probsscalar_i}{\sum_j \mylabelscalar_j \probsscalar_j}$$

which is the model predicted probability of output $i$, normalized to
the allowed outputs.  \cite{provablyconsistentpll} refer to this as risk-consistent loss (and we abbreviate as $\rcloss$).
We will not discuss risk-consistency  -- a  property of partial labelling losses that was established for $\rcloss$ under a
particular noise model: see  \cite{provablyconsistentpll} for a definition and discussion.

\subsection{Synthetic Experiments}
\label{subsec:synexp}

\introparagraph{Small consistent synthetic dataset} Recall that \Cref{ex:3outputs}
presented an extremely simple situation with $\outdim=3$ outputs and
$\nsamples = 2$ samples with the same input $\myinput$: $(\myinput,
\{A, B\})$ and $(\myinput, \{A, C\})$, i.e., each sample having
$\numallowed=2$ allowed outputs. To scale this example up, let us
consider a problem with $\outdim = 100$ possible outputs and a dataset
of $\nsamples=10$ samples, each having the same input vector
$\myinput$ and $\numallowed=10$ allowed outputs. In each of the $10$
samples $\mylabel^{(1)}, \ldots \mylabel^{(10)}$ allows output $o_0$
together with $9$ different values from among  $o_1 \dots
o_{10}$.\footnote{E.g.\ $\mylabel^{(1)} = \{o_0, o_2, o_3, \ldots,
o_{10}\}$, $\mylabel^{(2)} = \{o_0, o_1, o_3, \ldots, o_{10}\}$, \ldots
$\mylabel^{(10)} = \{o_0, o_1, o_2, \ldots, o_9\}$.} In this dataset,
there are $10$ outputs $o_1 \dots o_{10}$ that are ``almost good'' in
the sense that they are acceptable for $9$ out of $10$ samples and
there is a single output $o_0$ that is acceptable in all samples.
Hence, the only consistent solution is to select $o_0$. This example
highlights the challenge of identifying the correct label when some
alternative label has a large ``support'', i.e., when it is acceptable
by many samples, while not all of them.

\begin{figure}[htb]
  \centering
  \includegraphics[width=0.70\textwidth]{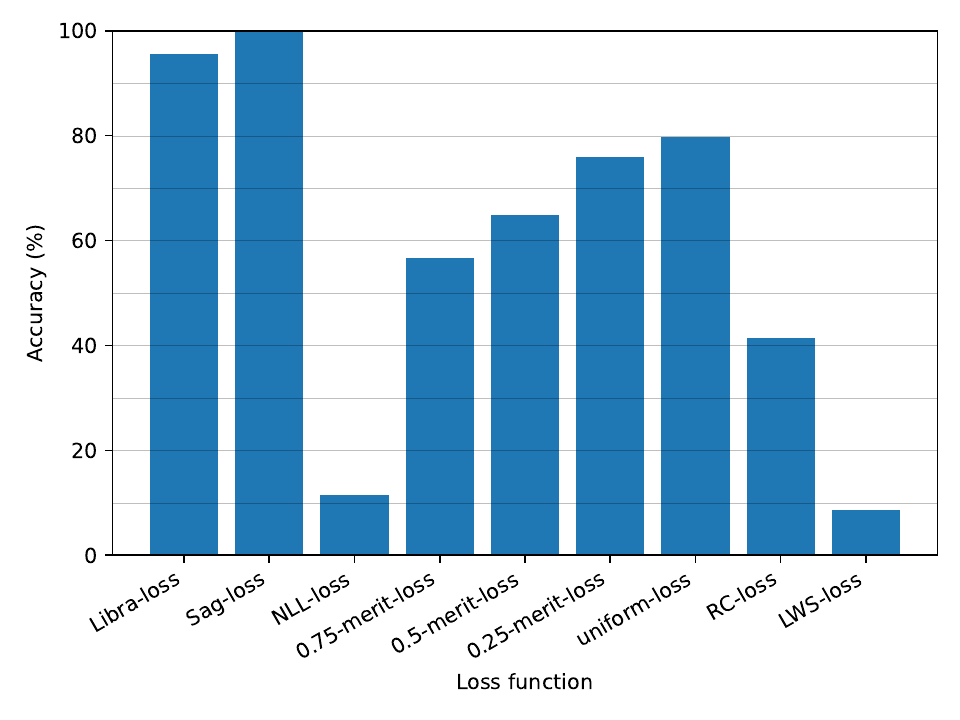}
  \caption{Average training accuracy on the small consistent dataset
    over 1000 random initializations for various loss functions.
    $\prploss$ and $\biprploss$ almost always find the optimal
    solution ($95.5\%$ and $99.9\%$), while $\nllloss$, $\rcloss$ and
    $\lwsloss$ perform extremely poorly on this task. $\bmeritloss$
    alleviates this weakness and reaches $79.7\%$ in the extreme case
    of $\beta=0$ ($\uloss$).
  }
  \label{fig:small_consistent_dataset}
\end{figure}

\paragraph{Network}
For each loss considered, we train an MLP with a single hidden layer of
$50$ neurons on this dataset, with $1000$ different random seeds.  We
employ Glorot~\citep{glorot} initialization.

\paragraph{Results}
We report average accuracy on the training set in
\Cref{fig:small_consistent_dataset}.  The models trained with
$\prploss$ and $\biprploss$ robustly find the output that is
consistent with all samples. However, for the other losses this is
often not the case. Depending on random initialization, some of the
suboptimal outputs can have higher initial probability, resulting in
them getting greater gradients, even greater than $o_0$ which is
promoted by all samples during optimization. This behaviour arises
when the strongly ``supported'', yet suboptimal, output has higher
initial probability than the single optimal output.
On the project webpage, we drill down to provide visualizations of $10$
randomly selected learning curves for each loss function.

We have seen earlier that the winner-take-all dynamics of $\nllloss$
makes it extremely sensitive to initialization. This, however, also
holds for the identification based approaches: $\rcloss$ and
$\lwsloss$. These methods weigh the loss for each output with the
model's own prediction: i.e., when allowed label $A$ is predicted to
be more likely than allowed label $B$, $A$ will be promoted more,
making the probability gap between $A$ and $B$ even greater. We argue
that this phenomenon is what makes these three losses perform so
poorly. $\bmeritloss$ reduces the winner-take-all effect and we get
better results as we decrease $\beta$. $\uloss$, which is completely
insensitive to the current model configuration and merely tries to
reach uniform distribution on the allowed outputs performs
surprisingly well, although still consistently worse than $\prploss$.

\introparagraph{Large consistent synthetic dataset} While the previous
example is useful to intuitively understand the harmful
``winner-take-all'' behaviour of $\nllloss$, it is very restrictive,
since it assumes a setting with multiple competing samples for the
same input vector.  In a more realistic scenario there 
are few (or no) samples with the same input and hence the
interaction among points is more subtle. More specifically, learned
models are functions that display some degree of smoothness.  As a
result, samples with similar features, i.e., similar input vectors,
will get similar predictions, affecting each other's prediction
accuracy.  In our next experiment, we aim to simulate this by building
a large synthetic dataset with partial labels. Our dataset has
$\nsamples=100,000$ samples, $\indim=100$ input dimension and
$\outdim=100$ possible outputs.  First, we produce a set of synthetic
input vectors with their corresponding true labels, as follows: We
uniformly sample $\outdim$ corners of a hypercube in
$\mathbb{R}^{100}$, i.e., from $\{0, 1\}^{100}$ which will function as
our cluster centroids. Each cluster will correspond to one true label,
ensuring that samples that have similar input will likely share their
true output.  Then, we utilize a mixture of $\outdim$ Gaussian
distributions (having standard deviation $1$) with our selected
centroids, and sample $\nsamples=100,000$ input vectors. Each input
vector $\myinput^{(i)}$ is assigned a true output $\truetarget^{(i)}$
corresponding to the Gaussian from which it was sampled.

With the input samples defined, we randomly select partial/distractor labels for each sample. Distractor selection is controlled by the following two parameters:

\begin{definition}[$\prate$]
In the context of a random  PLL dataset as above,
the Distractor pool fraction ($\prate$) 
is the  fraction of the output labels that can appear as distractors for any given true 
label.
\end{definition}

For example, if $\prate = 0.2$ and there are $\outdim = 100$ outputs
in total, then for each true label $c \in [\outdim]$ we select
(uniformly at random) $100 \cdot 0.2 - 1= 19$, other labels, which --
along with $c$ -- form the distractor pool $D(c)$. For each input
$\myinput^{(i)}$ the partial labels are constrained to be from $D(\truetarget^{(i)})$. The second parameter controls the \emph{strength of distraction}:

\begin{definition}[$\srate$]
The Distractor co-occurrence fraction ($\srate$) is the fraction
  of  inputs that are affected by any particular distractor from the
  distractor pool. More precisely, for any label $c$ and potential
  distractor $c' \in D(c)$ the fraction of inputs with true label $c$
  and distractor $c'$ is $\srate$.
\end{definition}

For example, if $\srate = 0.1$ and there are $1000$ inputs with true
label $c$, then distractor $c' \in D(c)$ will be present $1000 \cdot 0.1 =
100$ times as a distractor in the label sets of inputs with true label
$c$. A high $\srate$ means that the distractors are strongly ``supported'', 
i.e., are almost indistinguishable from true labels.

In the preceding example (small consistent synthetic dataset), $\srate
= 0.9$ since each distractor occurs in $9$ out of $10$ samples and
$\prate = 0.11$, since $11$ out of the $100$ possible outputs appear
in the label sets.

We note that $\prate$ and $\srate$ are just two of the many possible
ways of characterising this noise model. $\prate$ was motivated by the
observation that all losses are very sensitive to the number of
distractors and the motivation for $\srate$ comes from observing that
in the real world rule learning datasets, high $\srate$ made learning
much harder (see \Cref{subsec:rulelearningexp}).
The employed noise
model is \emph{instance-independent}, meaning that partial label
$\mylabel$ is independent from input $\myinput$ given true label
$\truetarget$.

\paragraph{Network}
We alter $\prate$ and $\srate$ and train models with various loss
functions. As underlying network, We use the same MLP model from
\cite{pllleveraging}, having $5$ layers and $333,108$ parameters. We
run each experiment 9 times, using 3 seeds for dataset generation and
3 seeds for training.

\paragraph{Results}
\Cref{fig:synthetic_large_consistent} shows model accuracies for
different loss functions, as well as $\prate$ and $\srate$ values. On
all plots, we see a clear downward trend in performance as we increase
$\prate$, with the exception of $\prploss$ and $\uloss$. We argue that
this is due to the winner-take-all behaviour: as we increase $\prate$,
there are more and more distractors, so the chance of one of them
getting significantly greater initial probability than the true label
increases, which makes it impossible to recover the true label. This
trend is greatly exacerbated by increasing $\srate$: when $\srate$ is
high, distractors are ``almost as good'' as the true label, so it gets
easy to confuse them. $\prploss$ and $\uloss$ demonstrate extreme
resistance against this kind of distraction. As in the previous
experiment, $\prploss$ performs consistently better than $\uloss$.

\begin{figure}[thb]
  \centering
  \includegraphics[width=0.32\textwidth]{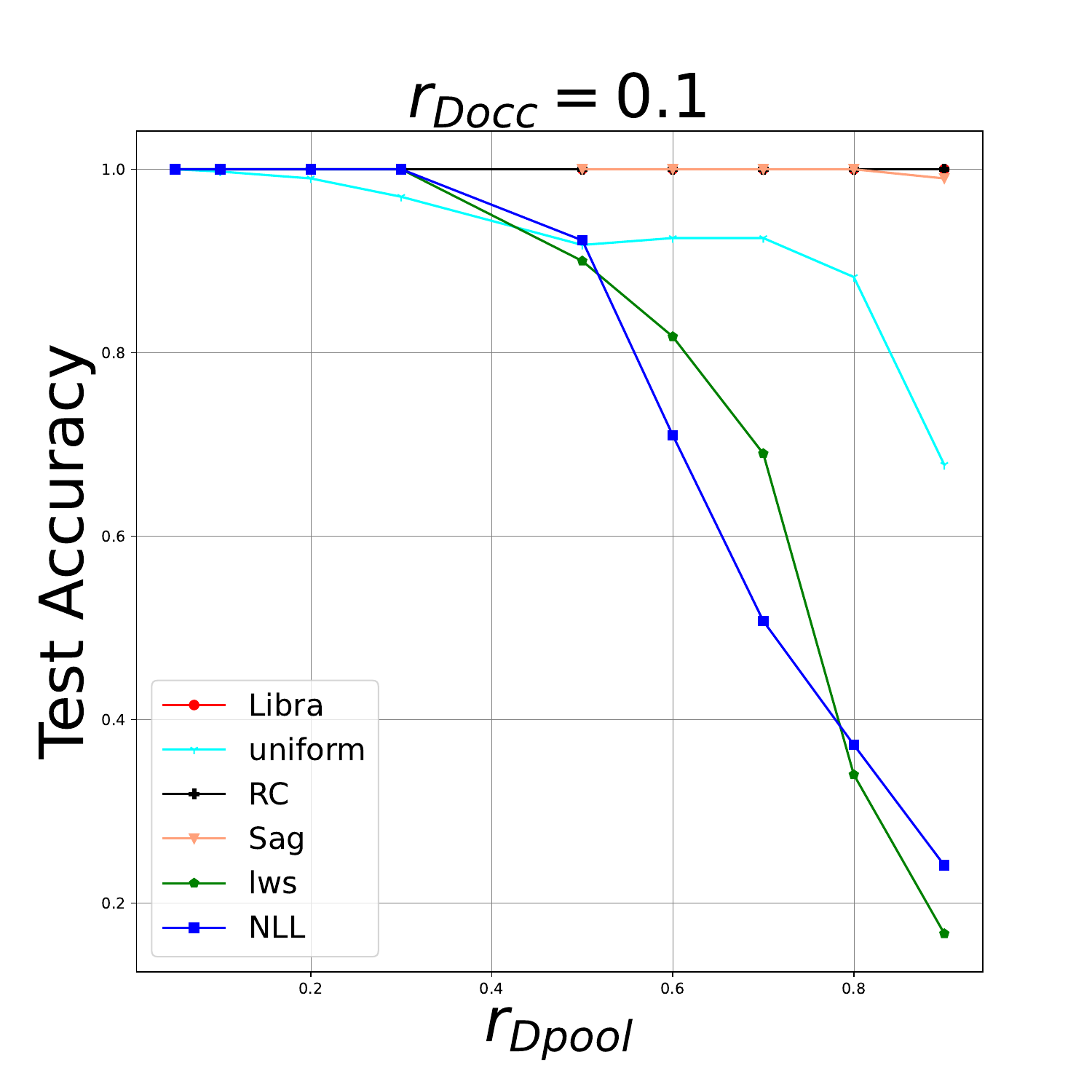}
  \includegraphics[width=0.32\textwidth]{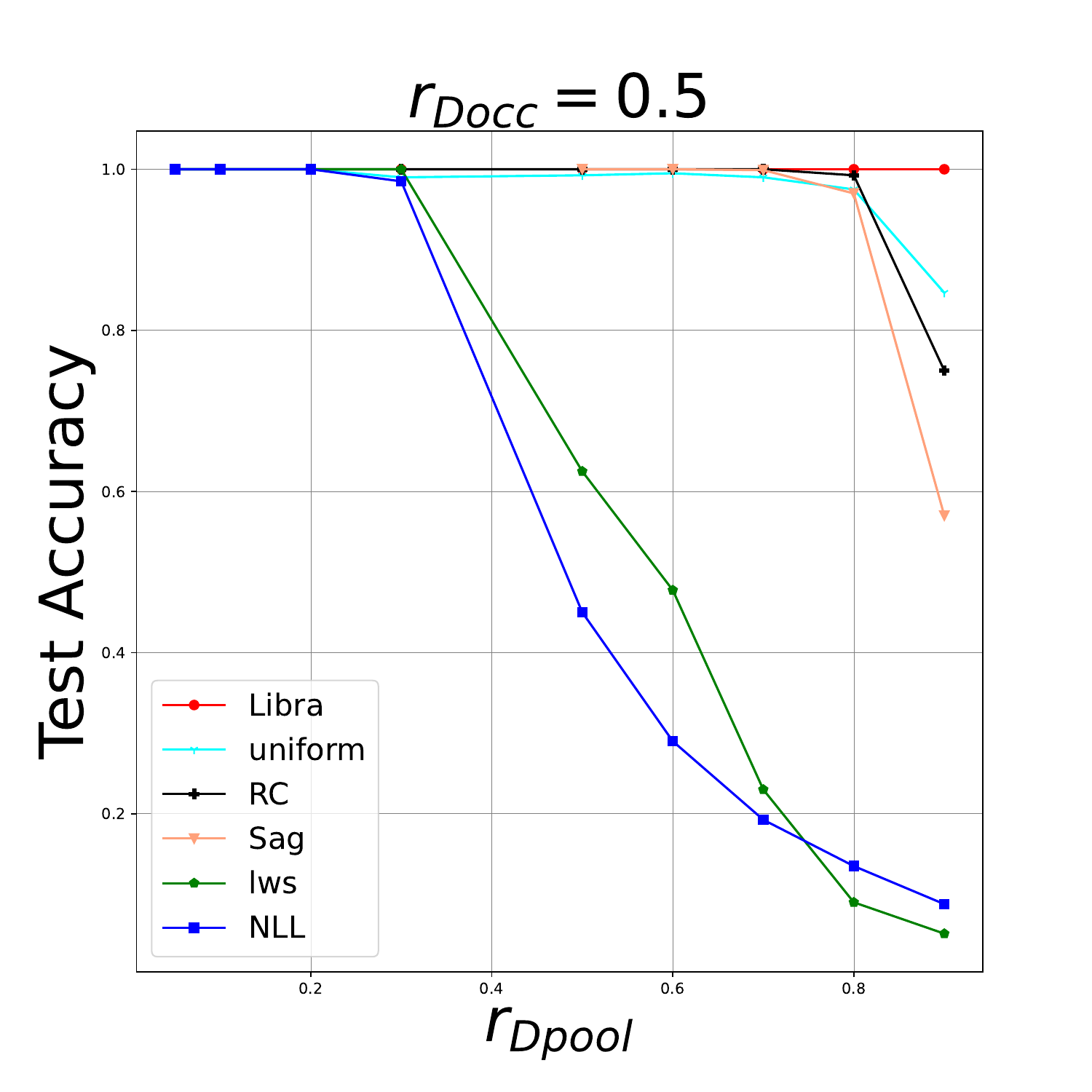}
  \includegraphics[width=0.32\textwidth]{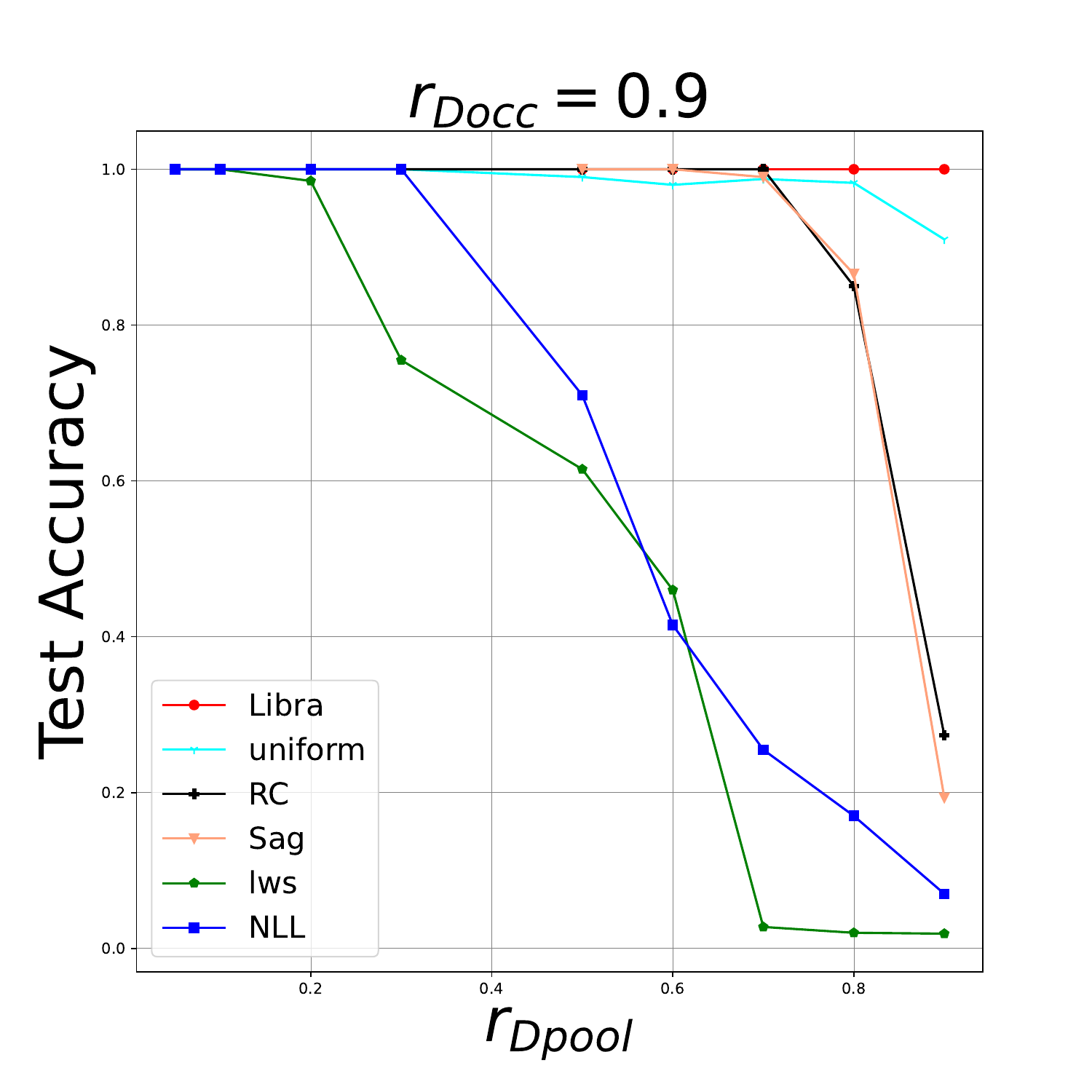}
  \caption{Test accuracy on the large consistent synthetic dataset for
    different combinations of $\prate$ and $\srate$ and different loss
    functions. We show mean values over 9 trials, using 3 seeds for
    dataset generation and 3 seeds for training.
  }
  \label{fig:synthetic_large_consistent}
\end{figure}

\subsection{Experiments with Real Datasets and Synthetic Distractors}
\label{subsec:realsynexp}

To better understand the practical value of learning methods for PLL
and DS, we can start from a real fully-labelled dataset instead of a
synthetic one, and generate distractor labels according to the noise
model. This approach is often taken in the literature to evaluate PLL
methods. We use the setup from \cite{pllleveraging}, starting from the
CIFAR10~\citep{cifar} image classification benchmark and apply various
true label dependent (instance-independent) noise models. \cite{pllleveraging} define three
cases, to which we add two harder ones and refer to them as ``Case 1''
\dots ``Case 5''. The noise models corresponding to these $5$ cases are
described in detail on the project webpage.
CIFAR10 has 10
possible outputs and out of the $9$ non-correct labels the expected
number of distractors is $0.5$, $0.6$, $1.8$, $4$ and $7.1$ for the
$5$ cases, respectively.

\paragraph{Network}
We train on this dataset the CNN model from \cite{pllleveraging}, that
has 9 convolutional layers and $4,434,570$ parameters.

\paragraph{Results}
\Cref{fig:cifar10_cases} shows the performance of several loss
functions trained on these datasets. Unsurprisingly, performance
decreases as the distraction is stronger, however, the only loss that
shows catastrophic collapse is $\lwsloss$.  Also note that while
$\uloss$ performs very well on purely synthetic inputs, it is clearly
inferior to the other competitors in this setup.  Some initial
experiments with $\prploss$ show easy overfitting, requiring careful
early stopping to avoid a drop in final accuracy. We overcome this by
introducing a weight $w_{\prplosssubscript} = 1-\sum_i y_i p_i$ that
makes the loss vanish as the model gets close to fitting. This weight
is used in all subsequent experiments.  Experiments with $\biprploss$
reveal that it is rather unstable. The explicit loss term that
penalizes each disallowed label makes the average of the logits
$\logits$ tend to minus infinity and training quickly reaches a
configuration that yields numerical instability. We managed to
overcome this by adding an extra L2 regularization term to the loss
that penalizes the magnitude of the logit vector:
$$
\Loss_{\textrm{logit}} = \gamma_{\logits} \sum_i \logitsscalar_i^2
$$ where $\gamma_{\logits}$
is a hyperparameter determining the importance of this loss term and
it is set to $0.01$ in our experiments. This regularization
successfully stabilised learning with the $\biprploss$, however, we
find that it performs consistently worse than $\prploss$. All later experiments with $\biprploss$ makes use of this regularization term.

\begin{figure}[thb]
  \centering
  \includegraphics[width=0.5\textwidth]{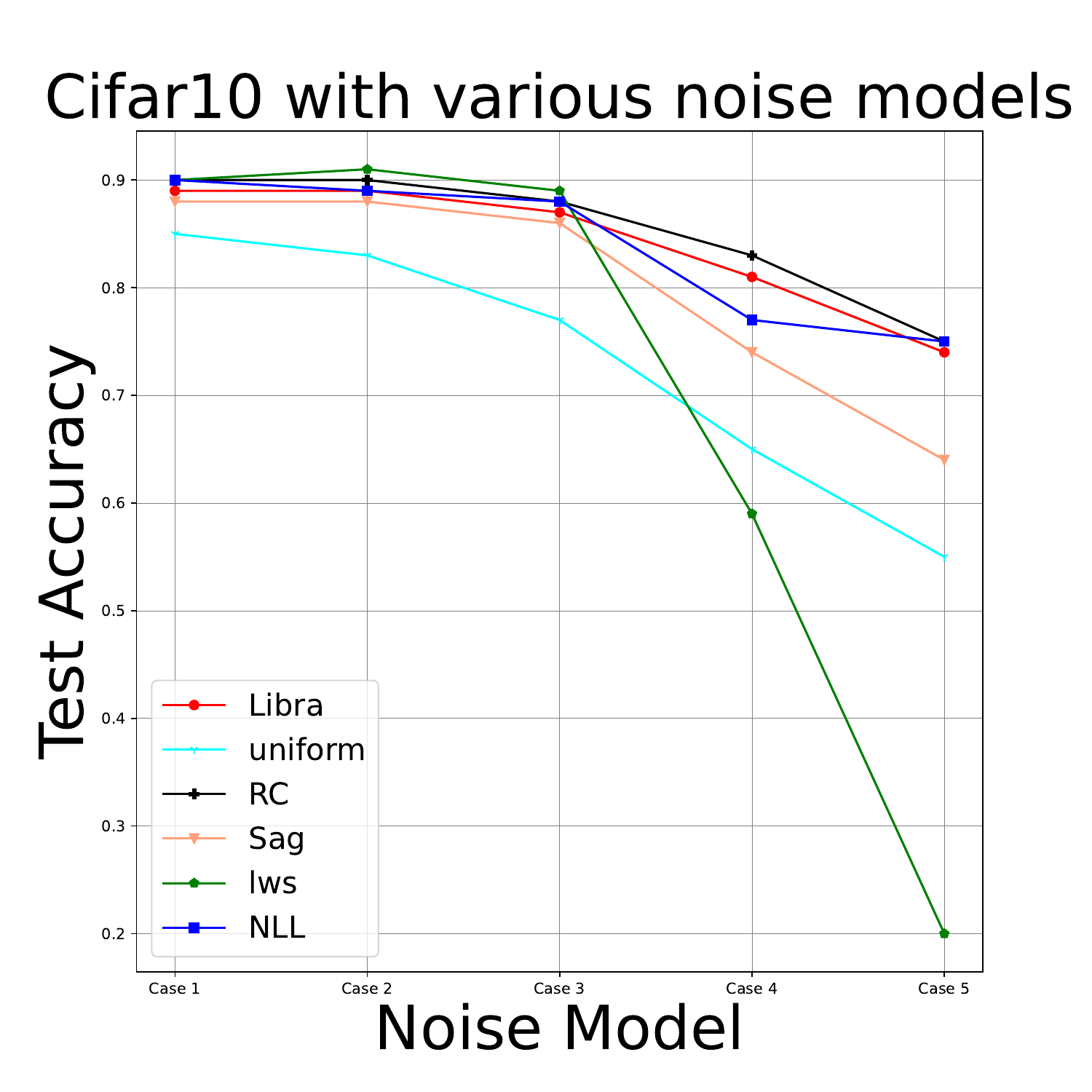} 
  \caption{Performance of various loss functions on a PLL dataset extracted from CIFAR10 and various noise models applied.}
  \label{fig:cifar10_cases}
\end{figure}

In our next experiment we evaluate the effect of changing $\prate$ and
$\srate$ on the much harder CIFAR100 dataset, which has 100 labels.

\paragraph{Network}
We use the $18$-block residual network from ~\cite{resnet18}, as
implemented in \cite{cifar100_models}. This model has $11,220,132$ parameters.

\paragraph{Results}
Figure~\ref{fig:cifar100} shows the same trends as observed on
Figure~\ref{fig:synthetic_large_consistent}: performance degrades as
$\prate$ (number of distractors) and $\srate$ increase (strength of
distraction) increase. However, $\prploss$ shows remarkable
robustness.

\begin{figure}[thb]
  \centering
  \includegraphics[width=0.32\textwidth]{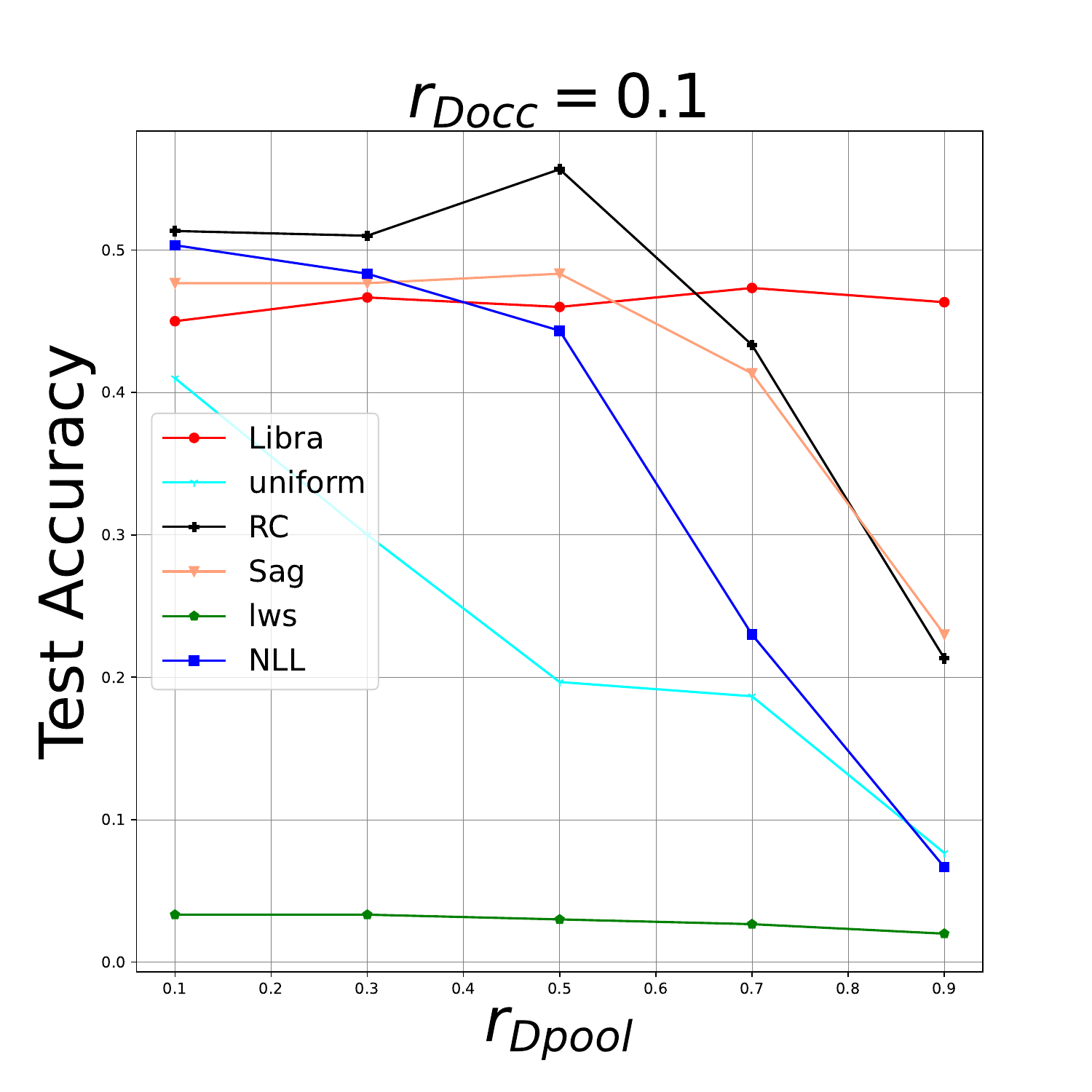}
  \includegraphics[width=0.32\textwidth]{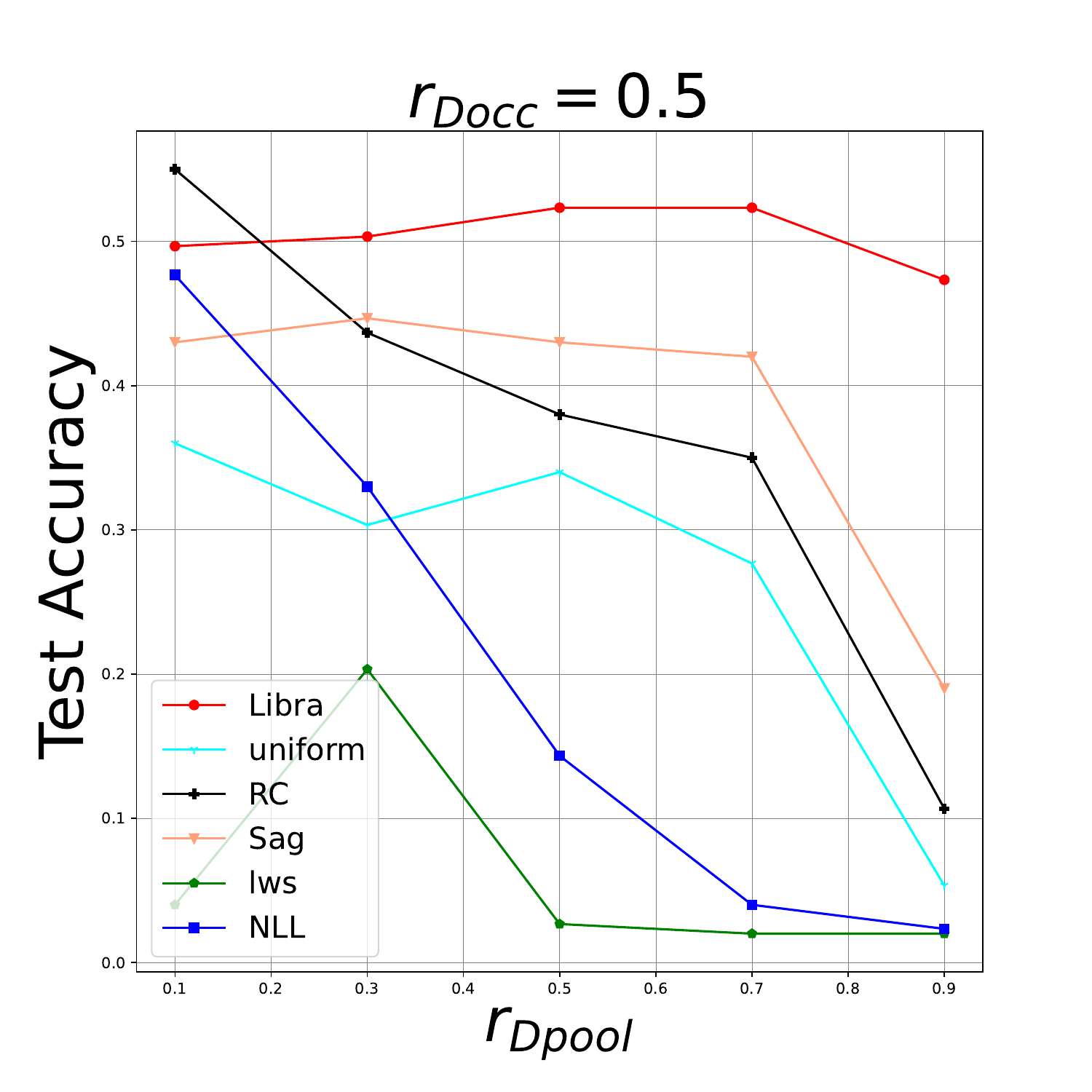}
  \includegraphics[width=0.32\textwidth]{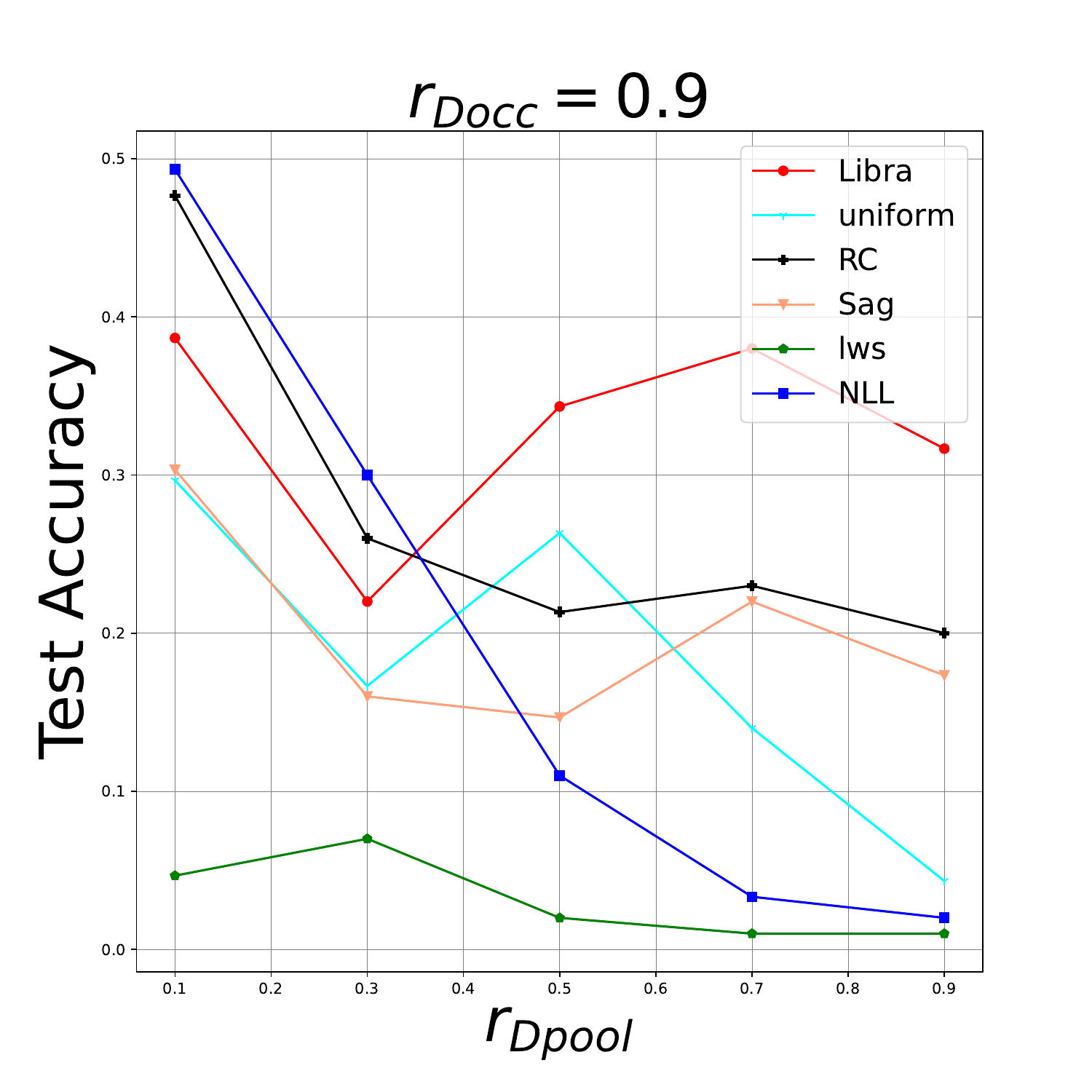}
  \caption{Test accuracy on CIFAR100 for different combinations of
    $\prate$ and $\srate$ and different loss functions. We show mean values over 3 trials, using 3 seeds for dataset generation.}
  \label{fig:cifar100}
\end{figure}

\subsection{Rule Learning Experiments}
\label{subsec:rulelearningexp}

In the following we experiment with the partially labelled rule
learning datasets, introduced in \Cref{sec:obda}. We remind the reader
that these datasets contain negative samples, which are handled as
described in \Cref{sec:obda}. We also recall that these are
sequence-to-sequence datasets, i.e., both the input and the output are
represented as sequences of tokens. As described in
\Cref{sec:prelims}, we use an autoregressive model
$\classifier_{\params}(\myinput,\prefix)$ that outputs a distribution
over single tokens in one step, conditioned on the preceding
tokens. By sequentially evaluating all tokens in a sequence, we obtain
the model predicted probability of the sequence. We recall that the
output space of sequences is huge and we cannot compute the
probability of all sequences, even if we employ some length
limit. Even computing the probabilities of allowed sequences (via
positive supervision) and explicitly forbidden sequences (via
negative supervision) is computation heavy due to the sequential
nature of evaluation. Consequently, we cannot use loss functions that
depend on the probabilities of all possible outputs, such as
$\biprploss$ or $\lwsloss$.

We treat all datasets as disjunctive supervision, i.e., we do not assume
a single unknown correct output.  Given a sample $(\myinput,
\mylabel)$, our primary evaluation metric is the probability of the
model outputting an allowed output:
$$
P_{pos} = \sum_i y_i \probsscalar_i
$$
Furthermore, we compute the probability of the model returning an output from any of the label sets of the negative samples (including training and test samples). Let $I_{neg} = \{i | (\myinput, \mylabel) \mbox{ is a negative example and }\mylabelscalar_i = 1\}$ denote the set of indices of all forbidden sequences. Then, the probability of selecting one of them is:
$$
P_{neg} = \sum_{i \in I_{neg}} \probsscalar_i
$$
We are also interested in $H@k$ metrics, which is the ratio of inputs for
which the $k$ highest scoring outputs according to the model include
either 1) an allowed (positive $H@k$) output or 2) a forbidden
(negative $H@k$) output. Exactly determining the $k$ highest scoring outputs
is not tractable, as it would require evaluating all possible
outputs. Thus we approximate this with beam-search, employing
beamsize $10$. All experiments employ a (70\%, 15\%, 15\%) train-validation-test split.

\introparagraph{CMT challenges} We experiment with the CMT challenges, described earlier
in  \Cref{sec:obda},
and train sequence-to-sequence models.

\paragraph{Network}
We use an encoder-decoder transformer
architecture~\citep{transformer} with embedding
dimension 128 and 4 encoder/decoder layers, having 2.5M
parameters.

The largest rule in the training set contains 17 tokens and the model
generated rules are restricted to 20 tokens.  The $\gamma$
hyperparameter that controls the tradeoff between positive and
negative samples is empirically set to $3$.  A single experiment lasts
for around $7 $hours on a single Nvidia A100 GPU.

\paragraph{Results}
\Cref{tab:rodi_cmt} shows our experimental results. Given the
different metrics, it is hard to come up with an unambiguous ordering of the
loss functions. Nevertheless, $\prploss$ clearly seems to perform best
in terms of predicting allowed outputs for the test samples and
$\bmeritloss$ is second best. $\uloss$ is weaker, but performs
consistently, while $\nllloss$ and $\rcloss$ are overall quite weak
and sometimes extremely weak. As for avoiding forbidden labels,
$\uloss$ tends to perform best, however, this becomes somewhat vacuous
given its mediocre performance on the allowed labels.

\begin{table}[htb]
\small
  \caption{$P_{pos}$, $P_{neg}$, H@1 and H@5 scores on the evaluation set of CMT
    datasets. For $\bmeritloss$, $\beta=0.5$, which provided the
    best results based on a grid search with $\beta \in \{0.25, 0.5,
    0.75\}$.}
  \label{tab:rodi_cmt}
  \centering
  \begin{tabular}{l l | l l l | l l l}
    & & \multicolumn{3}{c |}{\bf Positive} & \multicolumn{3}{c}{\bf
      Negative} \\
    Loss & Distortion & $P_{pos}$ & H@1 & H@5 & $P_{neg}$ & H@1 & H@5 \\
    \toprule
    $\prploss$ & - & {\bf 0.97} & 99\% & 100\% & {\bf 0.03} & 4\% & 7\% \\
    $\nllloss$ & - & 0.83 & 83\% & 85\% & 0.17 & 17\% & 17\% \\
    $\rcloss$ & - & 0.71 & 0.71 & 71\% & 0.2 & 20\% & 20\% \\
    $\bmeritlossB$ & - & 0.79 & 92\% & 100\% & 0.12 & 11\% & 25\% \\
    $\uloss$ & - & 0.69 & 93\% & 99\% & 0.08 & 22\% & 29\% \\
    \midrule
    $\prploss$ & renaming & {\bf 0.94} & 93\% & 99\% & 0.11 & 9\% & 28\% \\
    $\nllloss$ & renaming & 0.77 & 77\% & 79\% & 0.15 & 15\% & 15\% \\
    $\rcloss$ & renaming & 0.52 & 52\% & 52\% & 0.15 & 15\% & 15\% \\
    $\bmeritlossB$ & renaming & 0.78 & 90\% & 100\% & 0.13 & 8\% & 29\% \\
    $\uloss$ & renaming & 0.68 & 94\% & 100\% & {\bf 0.08} & 15\% & 31\% \\
    \midrule
    $\prploss$ & restructuring & {\bf 0.93} & 93\% & 100\% & {\bf 0.03} & 2\% & 26\% \\
    $\nllloss$ & restructuring & 0.32 & 32\% & 32\% & 0.13 & 13\% & 13\& \\
    $\rcloss$ & restructuring & 0.18 & 18\% & 18\% & 0.7 & 7\% & 7\% \\
    $\bmeritlossB$ & restructuring & 0.8 & 86\% & 99\% & 0.07 & 12\% & 21\% \\
    $\uloss$ & restructuring & 0.76 & 96\% & 100\% & 0.06 & 6\% & 27\% \\
    \midrule
    $\prploss$ & cleaning & 0.87 & 89\% & 98\% & 0.14 & 10\% & 29\% \\
    $\nllloss$ & cleaning & 0.71 & 71\% & 72\% & 0.10 & 10\% & 10\% \\
    $\rcloss$ & cleaning & 0.53 & 54\% & 54\% & 0.1 & 10\% & 10\% \\
    $\bmeritlossB$ & cleaning & {\bf 0.88} & 91\% & 100\% & 0.17 & 18\% & 30\% \\
    $\uloss$ & cleaning & 0.66 & 89\% & 98\% & {\bf 0.07} & 17\% & 29\% \\
    \midrule
    $\prploss$ & denormalising & {\bf 0.98} & 100\% & 100\% & 0.14 & 16\% & 20\% \\
    $\nllloss$ & denormalising & 0.32 & 32\% & 32\% & 0.08 & 8\% & 8\% \\
    $\rcloss$ & denormalising & 0.25 & 25\% & 25\% & 0.12 & 12\% & 12\% \\
    $\bmeritlossB$ & denormalising & 0.85 & 88\% & 100\% & 0.1 & 9\% & 17\% \\
    $\uloss$ & denormalising & 0.75 & 91\% & 100\% & {\bf 0.04} & 6\% & 20\% \\
    \bottomrule
  \end{tabular}
\end{table}

\introparagraph{NPD challenge} The NPD rule learning challenge is much
harder than the CMT challenges, mostly due to the larger number of
candidates (see \Cref{sec:obda} for details).

\paragraph{Network}
We train transformer models with embedding dimension 32 and 3
encoder/decoder layers, having 1M parameters.\footnote{We had to scale
down the model size compared to that in the CMT experiments because
candidate sets are larger and sequences are longer and we had to fit
into the memory of a single Nvidia A100 GPU.}

The largest rule in the training set contains 44 tokens and the model
generated rules are restricted to 50 tokens.  The $\gamma$
hyperparameter that controls the tradeoff between positive and
negative samples is empirically set to $0.001$.  A single experiment
lasts for around 23 hours on a single Nvidia A100 GPU.

\paragraph{Results}
A particularity of this dataset is that it contains inputs that 
share the same predicate, while having disjoint
labels, forcing the model to attend both to the predicates and the
constants in the input.
\Cref{tab:npd} shows that $\prploss$ performs best in terms of
predicting allowed outputs in the evaluation set and is only
marginally surpassed by $\rcloss$ in avoiding forbidden outputs, which
in turn completely fails to predict allowed labels. $\uloss$ is
competitive for allowed outputs, but performs rather poorly in terms
of avoiding forbidden labels.  The results also show that the
alignment produced by our solution is still far from perfect. However,
we know of no other tools that can detect the rules in the NPD dataset
with or without supervision.

\begin{table}[htb]
\small
  \caption{$P_{pos}$, $P_{neg}$, H@1 and H@5 scores on the evaluation set of the NPD dataset.}
  \label{tab:npd}
  \centering
  \begin{tabular}{l | l l l | l l l}
    & & \multicolumn{2}{c |}{\bf Positive} & \multicolumn{2}{c}{\bf Negative} \\
    Loss & $P_{pos}$ & H@1 & H@5 & $P_{neg}$ &  H@1 & H@5 \\
    \toprule
    $\prploss$ & {\bf 0.44} & {\bf 44\%} & 50\% & 0.02 & 2\% & 10\% \\
    $\nllloss$ & 0.1 & 10\% & 11\% & 0.02 & 2\% & 2\% \\
    $\rcloss$ & 0.06 & 6\% & 9\% & {\bf 0.01} & {\bf 1\%} & {\bf 1\%} \\
    $\bmeritlossB$ & 0.27 & 33\% & 45\% & 0.05 & {\bf 1\%} & 19\% \\
    $\uloss$ & 0.35 & 42\% & {\bf 69\%} & 0.19 & 26\% & 26\% \\
    \bottomrule
  \end{tabular}
\end{table}

\subsection{PLL Experiments with Real Datasets}
\label{subsec:pll_real}

To conclude our  experiments, we  adopt five real-world PLL datasets, each targeting a different task: 
Lost~\citep{pllavgbased}, Soccer Player~\citep{zeng2013learning}, and Yahoo!News~\citep{guillaumin2010multiple} 
for automatic face naming from video frames or images, 
MSRCv2~\citep{pllanotheridentificationbased} for object classification 
and BirdSong~\citep{briggs2012rank} for bird song classification. 

\paragraph{Network}
We perform experiments with two different models: the first Linear and
the second a $3$-layer MLP.  We use learning rate $0.1$ and weight
decay with parameter $10^{-3}$.  We train for $300$ epochs using
Stochastic Gradient Descent with batches of size $256$.  All
experiments are performed using Pytorch.

\paragraph{Results}
For this experiment we apply $10$-fold cross validation to evaluate all losses, and we report 
the accuracy along with the standard deviation.
We observe that $\prploss$ achieves the top performance for almost all datasets, 
with the exception of Yahoo!News.
In particular, $\prploss$ is the winner by a large margin for three out of five datasets, namely Lost, MSRCv2 and SoccerPlayer, 
while for BirdSong, it closely follows the winner.

\begin{table}[htb]
	\small
  \caption{Classification accuracy (mean$\pm$std) for five real-world datasets. 
          {\bf Soccer} and {\bf Yahoo} stand for SoccerPlayer and Yahoo!News benchmarks.
          For each dataset, the best method is indicated with \textbf{bold} 
          and the second best with \underline{underline}.}
  \label{tab:lostetc}
  \centering
  \begin{tabular}{l | l | l | l | l | l }
    Loss, Model & \bf Lost & \bf MSRCv2 &  \bf BirdSong & \bf Soccer &  \bf Yahoo \\
    \toprule
    $\nllloss$, linear & 61.6$\pm$3.2\% & 41.3$\pm$2.2\% & 70.9$\pm$1.5\% & 53.2$\pm$0.6\% & \textbf{64.7$\pm$0.4\%} \\
    $\prploss$, linear & \textbf{69.8$\pm$3.1\%} & 42.8$\pm$1.7\% & 65.8$\pm$1.3\% & \textbf{55.3$\pm$0.5\%} & 60.1$\pm$0.7\% \\
    $\biprploss$, linear & \underline{65.1$\pm$2.3\%} & 40.7$\pm$2.0\% & 62.1$\pm$1.3\% & 49.1$\pm$0.3\% & 47.6$\pm$0.6\% \\
    $\lwsloss$, linear & 39.4$\pm$4.9\% & 28.0$\pm$4.2\% & 57.3$\pm$2.1\% & 49.0$\pm$0.0\% & 46.6$\pm$0.7\% \\
    $\rcloss$, linear & 63.1$\pm$2.7\% & 40.9$\pm$2.1\% & 70.8$\pm$1.5\% & \underline{54.0$\pm$0.5\%} & \underline{64.7$\pm$0.5\%} \\
    \toprule
    $\nllloss$, MLP & 53.1$\pm$2.4\% & \underline{48.9$\pm$1.8\%} & 69.8$\pm$1.3\% & 52.3$\pm$0.5\% & 60.1$\pm$0.8\% \\
    $\prploss$, MLP & 59.7$\pm$2.4\% & \textbf{51.0$\pm$1.7\%} & \underline{72.4$\pm$1.0\%} & 52.7$\pm$0.4\% & 60.4$\pm$1.0\% \\
    $\biprploss$, MLP & 55.9$\pm$2.4\% & 48.7$\pm$1.7\% & 71.5$\pm$1.5\% & 53.7$\pm$0.5\% & 57.3$\pm$0.7\% \\
    $\lwsloss$, MLP & 51.7$\pm$3.0\% & 47.3$\pm$1.5\% & 67.3$\pm$1.1\% & 50.3$\pm$0.4\% & 53.9$\pm$1.0\% \\
    $\rcloss$, MLP & 53.5$\pm$2.6\% & 48.6$\pm$2.1\% & \textbf{72.4$\pm$1.0\%} & 52.8$\pm$0.4\% & 60.4$\pm$0.7\% \\
    \bottomrule
  \end{tabular}
\end{table}

\subsection{Discussion of the Experimental Results}
\label{subsec:discussionexp}

We draw together some of the main takeaways from our analysis.

\begin{itemize}
\item \emph{$\nllloss$ tends to perform poorly in the presence of a
softmax layer.} This is due to the winner-take-all bias. This
  loss works well in easier situations, however, when there are many
  distractors (high $\prate$) or some distractors are present in many
  samples (high $\srate$), performance drops steeply in more complex settings.
\item \emph{Identification-based methods are also susceptible to
winner-take-all bias.} This is because initially incorrect predictions
  can make erroneous labels being promoted more than the correct
  one. Analogously to $\nllloss$, this effect is exacerbated as
  $\prate$ and $\srate$ values increase.
\item \emph{$\bmeritloss$ improves over $\nllloss$.} Decreasing $\beta$
  reduces the effect of winner-take-all. Often, it is best to push it
  to the extreme, which is $\uloss$.
\item \emph{The most extreme antidote to winner-take-all is $\uloss$,
which can perform surprisingly well.}  $\uloss$ is the opposite of
  identification-based methods, and completely avoids
  winner-take-all. Indeed, it always explores multiple options
  equally, even after developing some experience and signal.  We find
  that it performs well on synthetic datasets
  (\Cref{fig:small_consistent_dataset,fig:synthetic_large_consistent}),
  where being cautious is useful. But it does not perform well on real
  datasets where it is important to ``exploit'' as well as explore
  (\Cref{fig:cifar10_cases,fig:cifar100},
  \Cref{tab:rodi_cmt,tab:npd}).
\item \emph{$\prploss$ tends to perform best, especially on harder
challenges.} $\prploss$ overcomes the winner-take-all bias by design,
  allowing more balanced exploration of alternatives. On the other
  hand, it is more flexible than $\uloss$, as it can adapt to
  experience accumulated during training. This ability is less
  important in synthetic datasets, but yields large performance
  difference in real datasets.
\item \emph{$\biprploss$ performs decently, but it can easily become
unstable.} This is because the magnitude of the logit vector increases
  quickly during learning, leading to numerical instability. This
  problem can be overcome with L2 regularization on the logits.
  However, not even the regularized variant ever performs better than
  $\prploss$.  
\end{itemize}

\section{Conclusion} \label{sec:conc}

In this paper we identify a bias phenomenon that emerges in partial label learning based on neural architectures
with a softmax layer.
We provide a loss function which is tailored towards addressing the situation, and argue
that it is, up to a differentiable transformation,  canonical.
We also give an experimental evaluation of its performance.
We discuss some of the  issues left over from this work.

\paragraph{Winner-take-all and characterization theorems} We have proven our main theoretical
results in a restricted setting, both in terms of the normalization function  (softmax) and the update mechanism -
(gradient descent).
It remains to investigate winner-take-all results and the $\prp$ property in more general settings.

\paragraph{Loss functions}  We have looked at loss functions that focus on combating a certain bias phenomenon; but there
are obviously many other desiderata within learning. It remains to investigate how properties like $\prp$ can be incorporated
in the setting where there are additional objectives in play.

\paragraph{Rule learning}
Disjunctive supervision is a special case of symbolic supervision in the form of logical constraints. In a 
learning framework where supervision is intermediated by the presence of logical constraints, more general forms of symbolic
supervision can emerge, not merely disjunctions of literals and their negations as in our application.
We will investigate the broader question of symbolic supervision in the future; we think our work shows
promise in tailoring loss functions for supervision intermediated by  more general formulas.

\paragraph{Evaluating loss functions} We have highlighted the distinction between PLL and DS, noting that evaluation of DS is much less explored experimentally.
We hope the rule learning benchmark we provide can be useful, but certainly a more diverse and extensive evaluation regime
for DS is needed.
 In the setting of PLL, there have been several datasets proposed, but we identified a number of shortcomings:
many of them allowed very few allowed outputs per input, and most provided little control over the strength of relationships
between true outputs and noise. As with DS, we hope that the synthetic PLL benchmarks we use here can improve the situation.

\acks{
This work has been supported by Hungarian National Excellence Grant 2018-1.2.1-NKP-00008, 
the Hungarian Artificial Intelligence National Laboratory (RRF-2.3.1-21-2022-00004), ELTE
    TKP 2021-NKTA-62 funding scheme. It has also been supported by the UK's Engineering and Physical Science
Research Center under Oxford’s EPSRC Impact Acceleration Account Award EP/R511742/1
as well as EPSRC EP/T022124/1. We thank Varun Kanade for his guidance and feedback on preliminary versions of this work.
}

\newpage
\onecolumn
\appendix

\section{Winner-take-all Theorem}
\label{app:winnertakeall}

In this section we prove the winner-take-all property of the
$\nllloss$ as stated in Theorem~\ref{thm:winner_take_all}.  Throughout
the section, we assume a softmax regression model $\probs =
\classifier_{\params}(\myinput)=\softmax (\params \cdot
\myinput)$. Furthermore, since the theorem deals with convergence on a
single sample, we assume without loss of generality that input
$\myinput$ is given in one-hot representation, i.e, the logit vector
$\logits$ is $\params_j$ for some $j$, i.e., directly updateable. We
use $\hat{p} = \sum_i \mylabelscalar_i \probsscalar_i$ to denote the sum of
probabilities of acceptable outputs.

We begin by calculating the gradients of the $\nllloss$ with respect
to the logits.

\begin{lemma}
  \label{lem:nll_gradients}
  The gradient vector $grad = \pd{\Loss_{\nll}(\probs, \mylabel)}{\logits}$ of the
  $\nllloss$ with respect to the logit vector $\logits = \params \cdot \myinput$ is given by
  $$grad_j = \frac{\probsscalar_j}{\hat{p}} \left( -\mylabelscalar_j + \hat{p} \right)$$
\end{lemma}

\begin{proof}

  \begin{align*}
    grad_j & = \sum_{\{i | \mylabelscalar_i=1\}} \pd{\Loss_{\nll}(\probs, \mylabel)}{\probsscalar_i} \pd{\probsscalar_i}{\logitsscalar_j}
    = \sum_{\{i | \mylabelscalar_i=1\}} - \frac{1}{\sum_{\{k | \mylabelscalar_k=1\}} \probsscalar_k}  \probsscalar_i (\delta_{ij} -
    \probsscalar_j)
    \\ & = \frac{1}{\sum_{\{k | \mylabelscalar_k=1\}} \probsscalar_k} \left( -\mylabelscalar_j \probsscalar_j + \sum_{\{\mylabelscalar_i=1\}} \probsscalar_i \probsscalar_j \right)
    = \frac{\probsscalar_j}{\sum_{\{k | \mylabelscalar_k=1\}} \probsscalar_k} \left( -\mylabelscalar_j + \sum_{\{\mylabelscalar_i=1\}} \probsscalar_i \right)
    \\ & = \frac{\probsscalar_j}{\hat{p}} \left( -\mylabelscalar_j + \hat{p} \right)
  \end{align*}
\end{proof}

Next, we compare the ratio of probabilities of two allowed outputs and
show that the ratio changes monotonically during training.

\begin{lemma}
  \label{lem:nll_increasing}
  Let $m, n$ be two acceptable outputs, i.e. $\mylabelscalar_m = \mylabelscalar_n = 1$. Let $\probs'$ denote the updated probability vector after a Gradient-update operation with some positive learning rate $\lambda$. Then it holds
  that $\frac{\probs'_m}{\probs'_n} > \frac{\probsscalar_m}{\probsscalar_n}$ exactly when $\probsscalar_m > \probsscalar_n$.
\end{lemma}

\begin{proof}
  Since we know that $\mylabelscalar_m = \mylabelscalar_n = 1$ the gradients computed in \Cref{lem:nll_gradients}  reduce to

  \begin{align*}
    grad_m & = \frac{\probsscalar_m}{\hat{p}} (-1 + \hat{p}) \\
    grad_n & = \frac{\probsscalar_n}{\hat{p}} (-1 + \hat{p})
  \end{align*}

  After the update step, the ratio of model predicted probabilities are:

  \begin{align*}
    \frac{\probs'_m}{\probs'_n} & = \frac{\softmax(\logits')_m}{\softmax(\logits')_n}
    = \frac{e^{\logits'_m}}{e^{\logits'_n}}
    = \frac{e^{\logitsscalar_m - \lambda \frac{\probsscalar_m}{\hat{p}} (-1 + \hat{p})}}{e^{\logitsscalar_n - \lambda \frac{\probsscalar_n}{\hat{p}} (-1 + \hat{p})}}
    \\ & = \frac{e^{\logitsscalar_m}}{e^{\logitsscalar_n}} \frac{e^{ - \lambda \frac{\probsscalar_m}{\hat{p}} (-1 + \hat{p})}}{e^{ - \lambda \frac{\probsscalar_n}{\hat{p}} (-1 + \hat{p})}}
    = \frac{\probsscalar_m}{\probsscalar_n} e^{\lambda (\probsscalar_m-\probsscalar_n) \frac{1-\hat{p}}{\hat{p}}}
  \end{align*}

  \noindent Since $0 < \hat{p} < 1$, the exponent has the same sign as $\probsscalar_m -
  \probsscalar_n$. From this it follows that the ratio increases exactly when
  $\probsscalar_m > \probsscalar_n$ and remains the same when $\probsscalar_m = \probsscalar_n$. This concludes
  our proof.
\end{proof}

We can now prove Theorem~\ref{thm:winner_take_all}, which we restate here:

\thmwta*

\begin{proof}
  The model probability of all outputs $j \in J$ is the same
  initially, and it follows from \Cref{lem:nll_increasing} that
  their ratios remain $1$ during training.  Let $I$ denote the set of
  acceptable outputs and let $I^c = \outputspace \setminus I$ denote its
  complement.  Recall that we
  update the $k^{th}$ logit  $\logitsscalar_k$ as

  $$\logitsscalar_k = \logitsscalar_k - \lambda\left(\probsscalar_k - \frac{\probsscalar_k \mylabelscalar_k}{\hat{p}}\right)$$

  In any state where none of the $\logitsscalar_i$ are $\pm\infty$, the gradient is
  nonzero and hence that state cannot be a convergence point. Consequently,
  $\logits$ can only converge to a state where at least one logit is $\pm
  \infty$.
  Note that the sum of logits is constant because the sum of the gradients at
  each step is zero:

  $$\sum_i \left( - \probsscalar_i + \frac{\probsscalar_i\mylabelscalar_i}{\hat{p}} \right) = -\sum \probsscalar_i + \frac1{\hat{p}}\sum_i \mylabelscalar_i\probsscalar_i = 0$$

  Given that the sum of $\logitsscalar_i$ is constant and that there is some logit
  that converges to $\pm \infty$, there must be a logit which
  converges to $\infty$. The disallowed logits are decreasing, so an
  allowed logit must converge to $\infty$.

  If $J = I$, i.e., all acceptable outputs have the same initial
  probability, then we are done, since $\logitsscalar_k$ for $k \in I$ are
  increasing and $\logitsscalar_k$ for $k\in I^c$ are
  decreasing and this only stops if $\hat{p} =1$, so the limit is uniform distribution over $J$.
  So we can assume that $J \neq I$. After $T$ update steps, the value
  of the $k^{th}$ logit with $k\in I$ will be

  $$\logitsscalar_k(T) = \logitsscalar_k(0) + \lambda \sum_{t=0}^{T-1} \left( \frac1{\hat{p}(t)} - 1 \right)\probsscalar_k(t)$$
  
  Let $j\in J$, $\iota \in I\setminus J$ and $c =
  \frac{\probsscalar_j(0)}{\probsscalar_\iota(0)}$.  Due to our assumption that $J$ contains
  all allowed logits with maximal probabilities at time $t=0$, we have
  that $c>1$.  Furthermore, we know from
  \Cref{lem:nll_increasing} that $\probsscalar_j$ grows faster than
  $\probsscalar_\iota$ in every update step, hence $\frac{\probsscalar_j(t)}{\probsscalar_\iota(t)}\ge
  c$ for every $t\ge 0$. This  gives us a lower bound on $\logitsscalar_j(T)$:
  
  \begin{align*}
    \logitsscalar_j(T) &\ge \logitsscalar_j(0) + \lambda c\sum_{t=0}^{T-1} \left( \frac1{\hat{p}(t)} -1 \right) \probsscalar_\iota(t)  
    = \logitsscalar_j(0) - c \logitsscalar_\iota(0) + c \logitsscalar_\iota(T)
  \end{align*}
  
  If $\logitsscalar_\iota \to \infty$, then the above calculation shows that $\logitsscalar_j
  \to \infty$ and that 

  $$\frac{\probsscalar_\iota^\infty}{\probsscalar_j^\infty} = \lim_{t \to \infty} e^{\logitsscalar_\iota(t) - \logitsscalar_j(t)} \le \lim_{t \to \infty} e^{C + (1-c) \logitsscalar_\iota(t)} = 0$$

  where $p^\infty$ denotes the limit distribution and $C = c \logitsscalar_\iota(0) -  \logitsscalar_j(0)$ is some constant.
  The limit goes to $0$ because $1-c < 0$, so $(1-c)\logitsscalar_\iota(t) \to -\infty$.

  We now consider the possibility that $\logitsscalar_\iota$ does not go to infinity
  as the number of updates increases.  Thus there exists $\delta$ such
  that $\logitsscalar_\iota(t)<\delta$ for all $t$ and

  $$\frac{\probsscalar_\iota^\infty}{\probsscalar_j^\infty} = \lim_{t \to \infty} e^{\logitsscalar_\iota(t) - \logitsscalar_j(t)}
  \le \lim_{t\to \infty} e^{\delta - \logitsscalar_j(t)} = 0$$ 

  This is because we showed previously that some allowed logit must
  converge to infinity and $\logitsscalar_j$ is greater than any other allowed
  logit, hence $\logitsscalar_j \to \infty$.

  We conclude that $\frac{\probsscalar_\iota^\infty}{\probsscalar_j^\infty}= 0$ in the limit
  state. Therefore, we showed that all probabilities in $J^c$
  converge to $0$.

  We know that $\logitsscalar_{j_1}(t) = \logitsscalar_{j_2}(t)$ for all $j_1, j_2 \in J$
  throughout the training because we apply the same gradient at each
  step. From this it follows that $p$ converges to a uniform
  distribution over $J$.
\end{proof}

\section{Theorems related to the $\prploss$}
\label{app:prploss}

In this section we prove the characterization theorems for loss functions satisfying the $\prps$ property. We recall the two theorems:

\thmprploss*

\thmprpchar*

We also recall the definition of the $\prploss$:

$$
\Loss_{\prplosssubscript}(\probs, \mylabel) = \log\left(1- \sum_i y_i p_i\right) -
\frac{1}{k} \sum_i y_i \log(p_i)
$$

Before beginning the proofs, we give a property of loss function with the $\prps$ property that
will be easier to work with.

\begin{restatable}{theorem}{thmgradform}
  \label{thm:prp_grad_form}
  Let $\Loss:\RR^n\times\{0,1\}^n \to \RR$ be a differentiable loss function. Then the $\prps$ property holds for $\Loss$ if and only if $\Loss$ satisfies the
  following system of equations for all index pair $(m, n)$ such
  that $\mylabelscalar_m=\mylabelscalar_n=1$:
  $$ \sum_i \pd{\Loss(p,y)}{\probsscalar_i} \probsscalar_i(\delta_{im} - \probsscalar_m) = \sum_i
  \pd{\Loss(p,y)}{\probsscalar_i} \probsscalar_i(\delta_{in} - \probsscalar_n)$$
\noindent where $\delta_{ij}$ is the Kronecker function.
\end{restatable}

\begin{proof}[Proof of \Cref{thm:prp_grad_form}]
  
  Let us compute the updated probabilities $\probs'$:
  
  \begin{align*}
    grad & := \pd{\Loss(\probs,\myinput)}{\logits} \\
    p' & = \softmax(\logits') = \softmax(\logits - \lambda grad) \\    
    p'_i & = \frac{e^{\logits'_{i}}}{\sum e^{\logits'_{j}}} = \frac{e^{\logitsscalar_{i}- \lambda
        grad_i}}{\sum e^{\logits'_{j} - \lambda grad_j}} \\
    \frac{\probsscalar_m}{\probsscalar_n} & = \frac{e^{\logitsscalar_{m}}}{e^{\logitsscalar_{n}}} \\
    \frac{\probs'_m}{\probs'_n} & = \frac{e^{\logitsscalar_{m}- \lambda grad_m}}{e^{\logitsscalar_{n}- \lambda
        grad_n}} = \frac{\probsscalar_1}{\probsscalar_2} \frac{e^{\lambda grad_n}}{e^{\lambda
        grad_m}} 
    \end{align*}

  Here we recall that $\lambda$ is the learning rate.  The last
  equation above shows that the ratios remain the same if
  $\frac{e^{\lambda grad_n}}{e^{\lambda grad_m}} = 1$, i.e., $grad_m =
  grad_n$. Let us decompose the gradient using the chain rule:

  \begin{align*}
    grad_j & = \pd{\Loss}{\logitsscalar_{j}} = \pd{\Loss}{\probs} \pd{\probs}{\logitsscalar_{j}}
    = \sum_i \pd{\Loss}{\probsscalar_i} \pd{\probsscalar_i}{\logitsscalar_{j}}
    = \sum_i \pd{\Loss}{\probsscalar_i} \probsscalar_i(\delta_{ij} - \probsscalar_j)
  \end{align*}

    For each pair of allowed outputs $(m, n)$, the loss function has to
    satisfy the differential equation $grad_m = grad_n$, i.e.:

    \begin{equation}
    \sum_i \pd{\Loss}{\probsscalar_i} \probsscalar_i(\delta_{im} - \probsscalar_m) = \sum_i \pd{\Loss}{\probsscalar_i} \probsscalar_i(\delta_{in} - \probsscalar_n) \nonumber
    \end{equation}
    This concludes our proof.
\end{proof}

\begin{proof}[Proof of Theorem~\ref{thm:prp_loss}.]

  Recall that the $\prps$ property assumes a softmax regression model and that the logit vector $\logits$ is a parameter vector.  
  In order to show that $\Loss_{\prplosssubscript}$ has the
  $\prps$ property, we need to show that partial derivatives in
  $\logitsscalar_j$ are equal whenever $\mylabelscalar_j=1$.  First, we compute
  the partial derivatives in the probabilities $\probsscalar_j$.

  \begin{align*}
    \pd{\Loss_{\prplosssubscript}}{\probsscalar_j} &= \pd{\left[\log(1-\sum_i \mylabelscalar_i \probsscalar_i) -\frac1{k}\sum_i \mylabelscalar_i\log \probsscalar_i \right]}{\probsscalar_j}
    \\&= \begin{cases}
      -\frac1{1-\sum_i \mylabelscalar_i \probsscalar_i}-\frac1{k}\frac1{\probsscalar_i} \quad & \text{if} ~ j\in I
      \\0 & \text{otherwise}
    \end{cases}
  \end{align*}
  Now, we compute the partial derivatives in $\logitsscalar_j$. 
  Recall that $\pd{\probsscalar_i}{\logitsscalar_j}$ is the partial derivative of the softmax function which is $\probsscalar_i(\delta_{ij} -\probsscalar_j)$. 
  \begin{align*}
    \pd{\Loss_{\prplosssubscript}}{\logitsscalar_j} &= \sum_{i} \pd{\Loss}{\probsscalar_i}\pd{\probsscalar_i}{\logitsscalar_j}
    = \sum_i -\mylabelscalar_i \left(\frac1{1-\sum_{i'} \mylabelscalar_{i'} \probsscalar_{i'}} + \frac1{k\probsscalar_i}\right) \probsscalar_i ( \delta_{ij}-\probsscalar_j)
    \\&= -\frac1{1-\sum_i \mylabelscalar_i \probsscalar_i} \sum_i \mylabelscalar_i \probsscalar_i(\delta_{ij}-\probsscalar_j) -\frac1{k}\sum_i \mylabelscalar_i (\delta_{ij} - \probsscalar_j)
    = -\frac{\mylabelscalar_j\probsscalar_j - \probsscalar_j\sum_i \mylabelscalar_i \probsscalar_i}{1-\sum_i \mylabelscalar_i \probsscalar_i} - \frac1{k}\mylabelscalar_j + \probsscalar_j
    \\&=\frac{(1-\mylabelscalar_j)\probsscalar_j}{1-\sum_i \mylabelscalar_i \probsscalar_i} -\frac{\mylabelscalar_j}{k}
    =\begin{cases}
    -\frac1{k} \quad &\text{if} ~ \mylabelscalar_j=1
    \\\frac{\probsscalar_j}{1-\sum_i \mylabelscalar_i \probsscalar_i} \quad &\text{if} ~ \mylabelscalar_j=0
    \end{cases}
  \end{align*}
  As we can see, the gradients of the logits with $\mylabelscalar_i=1$ are equal,
  hence $\Loss_{\prplosssubscript}$ has the $\prps$ property.
\end{proof}

\begin{proof}[Proof of \Cref{thm:prp_char}.]

  For given $h_k$ continuously differentiable functions, let $\Loss(\probs, \mylabel) = h_{|\mylabel|}(\Loss_{\prplosssubscript}(\probs, \mylabel))$. 
  We showed earlier in \Cref{thm:prp_grad_form} that $\Loss$ has the $\prps$ property if it satisfies a linear differential equation 
  and we have shown that $\Loss_{\prplosssubscript}$ satisfies it. 
  We know that $\pd{\Loss}{\probsscalar_i} = h_{|\mylabel|}'(\Loss_{\prplosssubscript}(\probs, \mylabel)) \pd{\Loss_{\prplosssubscript}}{\probsscalar_i}$.
  Multiplying the equations in \Cref{thm:prp_grad_form} for $\Loss_{\prplosssubscript}$ with
  $h_{|\mylabel|}'(\Loss_{\prplosssubscript}(\probs, \mylabel))$ yields the equations for $\Loss$. Therefore $\Loss$ also has the $\prps$ property.

  The proof of the converse statement consists of several steps, which we will label for better transparency.
  
  \textbf{1.} Consider a loss function $\Loss$ that has the $\prps$ property and satisfies the technical assumptions in the statement of the theorem. 
  According to Theorem~\ref{thm:prp_grad_form}, the $\prps$ property is equivalent to a differential equation which is an invariant of $\Loss$ at any given set of labels $\mylabel$.
  Therefore, we only consider the case when $\mylabel$ is fixed such that the first $\numallowed$ outputs are the allowed ones,
  i.e., $\mylabelscalar_i=1 \leftrightarrow i \leq k$. Since $\mylabel$ is fixed, we can treat $\Loss$ and $\Loss_{\prplosssubscript}$ as functions over $\probs$, omitting $\mylabel$ from its domain. 
  According to Theorem~\ref{thm:prp_grad_form}, the partial derivatives of $\Loss$
  satisfy the system of equations:

  $$\sum_i \pd{\Loss}{\probsscalar_i} \probsscalar_i(\delta_{im} - \probsscalar_m) = \sum_i \pd{\Loss}{\probsscalar_i}
    \probsscalar_i(\delta_{in} - \probsscalar_n)
  $$

  for all $m, n \leq k$.  To understand these equations better, we
  define the parameterized matrix $A\in \RR^k \to \RR^{k\times k} $ such that
  $A_{i,j}=\probsscalar_j(\delta_{ij}-\probsscalar_i)$ where $i, j \leq k$, i.e., we only
  consider rows and columns corresponding to allowed outputs. We will
  use the apostrophe notion for denoting the Jacobian matrix of a
  smooth function, i.e., $\Loss'$ 
is the vector of partial
  derivatives of $\Loss$ with respect to the logits of allowed outputs.  Note
  that, at any input value in $\RR^k$,  the $m^{th}$ entry of $A \Loss'$ is the left-hand side of the above
  equation:
  
  $$(A \Loss')_m = \sum_i \probsscalar_i(\delta_{im}-\probsscalar_m) \Loss'_i$$
  
  Therefore the above system of equations is equivalent to the value
of $A \Loss'$ at any input being a constant vector. That is, $A \Loss'$ is  of the form 
$\langle \cfunct \ldots \cfunct \rangle$ for some function
$\cfunct:\RR^k\to \RR$.
At any input value, the corresponding matrix $A$ is invertible  if and only if $\det A \neq 0$. 
  Below, we show, by direct calculation,  that $\det A = \left(1 - \sum \probsscalar_i\right)\prod \probsscalar_i$. 
\begin{lemma}
  \label{lem:A_inv}
  Let $A\in \RR^{k \times k}$ denote the matrix such that
  $A_{i,j}=\probsscalar_j(\delta_{ij}-\probsscalar_i)$ and let $v = A^{-1} \underline{1}$. 
  Then $\det A = \left(1 - \sum \probsscalar_i\right)\prod \probsscalar_i^{-1}$.
\end{lemma}

\begin{proof}
  First, note that $A = B \cdot \text{diag}(\probsscalar_1, \dots \probsscalar_k)$, where
  $B=(\delta_{ij}-\probsscalar_i)_{i,j}$. The inverse of $\text{diag}(\probsscalar_1, \dots \probsscalar_k)$ is $\text{diag}(\probsscalar_1^{-1}, \dots \probsscalar_k^{-1})$.
The determinant of a diagonal matrix is just the product of the diagonal entries.
  So we only need to show that the determinant of $B$ is $1- \Sigma \probsscalar_i$.
Subtracting the last column
  from any other will not change the determinant, but will simplify
  the calculation
  \begin{align*}
    & B = \begin{pmatrix}
      1-\probsscalar_1 & -\probsscalar_1  & -\probsscalar_1 & \ldots & -\probsscalar_1
      \\-\probsscalar_2  & 1-\probsscalar_2 & -\probsscalar_2 & \ldots & -\probsscalar_2
      \\\vdots & &  & \ddots
      \\-\probsscalar_k & -\probsscalar_k & -\probsscalar_k & \ldots & 1-\probsscalar_k
    \end{pmatrix} \rightarrow
    & B_1 = \begin{pmatrix}
      1 & 0  & 0 & \ldots & -\probsscalar_1
      \\0  & 1 & 0 & \ldots & -\probsscalar_2
      \\\vdots & & & \ddots 
      \\-1 & -1 & -1 & \ldots & 1-\probsscalar_k
    \end{pmatrix}
  \end{align*}
  Now using the definition of determinant, $\det B = \sum_{\pi} \prod_{i} B_{i,\pi(i)}$, 
  where $\pi$ goes over every permutation, we see that the only non-zero products are $1-\probsscalar_k$ and  $-\probsscalar_1, -\probsscalar_2, \dots -\probsscalar_{k-1}$. 
  Hence $\det B = \det B_1 = 1 - \probsscalar_1 -\dots - \probsscalar_{k}$. 
This completes the derivation.
\end{proof}

  In particular, the lemma above tells us that $A$ is invertible over any non-degenerate probability distribution. %
  If $A$ is invertible, then we can simply calculate
$\Loss'=\cfunct A^{-1} \underline{1}$. 
  Let $v$ denote $A^{-1} \underline{1}$. Thus $\Loss'=\cfunct v$.
Let $\kappa_{\prplosssubscript}$ be the value of $\cfunct$ for $\Loss_{\prplosssubscript}$.
We can show:

\begin{claim} \label{clm:notzero}
$\kappa_{\prplosssubscript}$ is never $0$, assuming $\sum \probsscalar_i$ is neither $0$ nor $1$.
\end{claim}
\begin{proof} At a point where $\kappa_{\prplosssubscript}$ is $0$, we have
$\Loss'_{\prplosssubscript}$ is $0$, since $\Loss'_{\prplosssubscript}= \kappa_{\prplosssubscript} v_{\prplosssubscript}$.
But above we have calculated that $\pd{\Loss_{\prplosssubscript}}{\logitsscalar_j}$ is
    $-\frac1{k}$ if ~ $\mylabelscalar_j=1$ and
    $\frac{\probsscalar_j}{1-\sum_i \mylabelscalar_i \probsscalar_i}$ if $\mylabelscalar_j=0$.
Clearly this is not $0$ when $\vec p$ is nontrivial.
\end{proof}

From the claim it follows that at every point, $\Loss'$ and $\Loss'_{\prplosssubscript}$ only
differ by a constant multiple. Of course, we are not interested in the derivatives of the loss functions, but in the functions
themselves.

  Before we move on with the remainder of the proof, here is an outline of the steps.
  \begin{enumerate}
    \item We argue that $\Loss' = \cfunct v$ for some $\cfunct:\RR^k \to \RR$. 
and that $\Loss' = d \cdot L'_{\prplosssubscript}$ for some constant function $d$.

    \item The sets $H_z = \Loss_{\prplosssubscript}^{-1}(\{z\})$  are path-connected.
    \item We argue that $\Loss$ is constant on $H_z$ for any $z$. Restated, this means that the function $h:\RR\to\RR$ as
required by theorem (but not necessarily smooth) exists.  This will make use of the first items above.
    \item The $h$ function is continuously differentiable.
  \end{enumerate}

We have already shown the first item above, modulo the gap of showing
the determinant of $A$ is nonzero, and also that $\kappa_{\prplosssubscript}$ is never
$0$.

  \textbf{2.} Note that in this item, we are only reasoning about
$\Loss_{\prplosssubscript}$, and not the generic loss function $\Loss$.
Let $H_z = \Loss_{\prplosssubscript}^{-1}(\{z\}) = \{ \probs | \Loss_{\prplosssubscript}(\probs, \myinput) = z \}$ be the preimage of $z$. 
  Let $\mathcal{P}=\{(\probsscalar_1, \dots, \probsscalar_k)| \probsscalar_i\in (0,1), \sum_j \probsscalar_j < 1\}$ denote the space of the projection onto the first $k$ coordinates of the non-degenerate probability distribution over $\outdim$ categories.

  Note that $\mathcal{P}$ is an open path-connected subset of $\RR^k$ and hence a differentiable manifold. 
  At the same time, the range of $\Loss_{\prplosssubscript}$ is $\RR$, which is also a differentiable manifold.
  Thus we can view $\Loss_{\prplosssubscript}$ as a smooth map between manifolds $\mathcal{P}$ and $\RR$. 
Our next goal will be:
\begin{claim} \label{clm:preimagemanifold} Each $H_z$
is a differentiable manifold.
\end{claim}
  
\begin{proof}
  Let $X,Y$ be two differentiable manifolds and $f:X\to Y$ a smooth map between them.
  We say that $y\in Y$ is regular if for every $x\in f^{-1}(y)$ the map $df_x: T_xX \to T_yY$ is surjective, where $T_xX$ is the tangent space of $X$ in $x$.

  We will use the following elementary result about differentiable manifolds. 

\begin{fact} \label{fact:regularpoint}
  If $y\in Y$ is a regular value of $f$, then $f^{-1}(y)$ is a differentiable submanifold of $X$.%
\end{fact}

  We want to show that $d\Loss_{\prplosssubscript}$ is surjective everywhere, in order to argue, using
the fact above, that  the pre-image of a single point is a differentiable
manifold.

  Since $\RR$ is one-dimensional, $d\Loss_{\prplosssubscript}$ is not surjective precisely when $d\Loss_{\prplosssubscript}$ is the zero map.
  Equivalently the gradient is zero; it follows from \Cref{clm:notzero}
that this
can only occur on the 
boundary of $\mathcal{P}$. 
  Therefore, any $z$ is a regular value of $H$, 
and consequently $H_z$ is a differentiable submanifold of $\mathcal{P}$.
\end{proof}

For any probability distribution $\probs$ over $\outdim$ categories with $\sum_{j=1}^k \probsscalar_j=1$ we assign a line that goes through $\probs$ and $0$, let $l_{\probs} = \{ \omega \probs | \omega \in (0,1) \}$ denote this line.
Note that $l_{\probs}$ does not contain either $(\probsscalar_1, \dots \probsscalar_k)$ or $0$ and it lies in $\mathcal{P}$, i.e., $l_{\probs}\subset \mathcal{P}$.
Informally, $l_{\probs}$ represents the possible ways of ``scaling down'' some target distribution that assigns all the mass to acceptable elements.
We
make the following claim, where again $\vec y$ is fixed to sum to $k$.
\begin{claim} \label{clm:everyvalue}  $\Loss_{\prplosssubscript}$ takes every value precisely once on $l_{\probs}$. 
\end{claim}

\begin{proof}
To prove the claim, observe that the loss is 
\begin{align*}
  \log(1 - \omega \sum_i \mylabelscalar_i \probsscalar_i) - \frac{1}{k} \sum_i \mylabelscalar_i \log(\omega \probsscalar_i)
    &= \log\left(1- \omega\right) -\log(\omega) -
    \frac{1}{k} \sum_i \mylabelscalar_i \log(\probsscalar_i)
  \end{align*}
  where we used that $\sum_i \mylabelscalar_i \probsscalar_i = 1$ and $\sum_i \mylabelscalar_i = k$.
  It is clear that when $\omega \to 0$ it converges to $\infty$ and when $\omega \to 1$ it converges to $-\infty$.
  Now, we show that the above mapping is monotonically strictly decreasing in $\omega$, and consequently $\Loss_{\prplosssubscript}$ takes every value of $\RR$ precisely once on $l_{\probs}$.
  It is sufficient if the derivative with respect to $\omega$ is less than zero. 
  The derivative is $ -\frac{1}{1- \omega} - \frac1{\omega} $ which is clearly less than zero. Hence the claim is proven.
\end{proof}

Recall that we are interested in showing path connectedness
of the set $H_z$, the pre-image of singletons under $\Loss_{\prplosssubscript}$.
By the claim above, we know that as we vary the lines
$l_{\probs}$, $\Loss_{\prplosssubscript}$ always hits $H_z$ exactly once on the line, but the point at which it hits $H_z$ varies with $\probs$.

  Let $\pi:\mathcal{P} \to S^{k-1}$ the projection given by $\pi(x) = \frac{x}{\left\lVert x \right\rVert}$. Here $\lVert\cdot\rVert$ is the $2$-norm.
  Note that the preimage of a point under $\pi$ is precisely an $l_{\probs}$ line for some $\probs$.
  Since we have shown above that $\Loss_{\prplosssubscript}$ takes every value once over a fixed $l_{\probs}$ line, we conclude that $\pi$ is a bijection between $\pi(\mathcal{P})$ and $H_z$.

We will now use the fact that $H_z$ is a manifold by \Cref{clm:preimagemanifold}.
  It is known that for manifolds, connected and path-connected are equivalent properties. 
  If $H_z$ were not connected, then there would be $U,V$ disjoint
non-empty open sets such that $H_z\subset U\cup V$. 
  Observe that $\pi$ is an open map,  $\pi(\mathcal{P})$ is connected, and $\pi{\mathcal{P}} = \pi(H_z) \subset \pi(U\cup V)$. We
cannot have two disjoint open sets covering the connected set $\pi(\mathcal{P})$. Thus the sets $\pi(U)$ and $\pi(V)$ must overlap:
  there are $u\in U, v\in V$ points such that $\pi(u) = \pi(v)$. 
Thus there are two distinct
points in the pre-image of $\pi$ with the same value. Since
the pre-image is an $l_{\probs}$ line, this
contradicts \Cref{clm:everyvalue}.

\textbf{3.} We show that $\Loss$ is constant on $H_z$ for any $z$,
and that a $h:\RR\to\RR$ function exists such that $\Loss = h(\Loss_{\prplosssubscript})$.
The idea will be that for any $a \neq b \in H_z$ we show
$\Loss(b)-\Loss(a)=0$.  We do this by computing $\Loss(b)-\Loss(a)$ as an integral of
a quantity, over a path in $H_z$ between $a$ and $b$, using the fact
that $H_z$ is path-connected. The quantity will involve a dot product
with the derivative of $\Loss$, and we will use part (1) to argue that this dot
product is always $0$.
We will make use of the following result from multi-variable
calculus
\begin{proposition} \label{prop:pathint}
  For any $F:\RR^m\to \RR$ continuously differentiable function and 
$\gamma:[0,1] \to \RR^m$ differentiable path from $\gamma(0)=a$ to $\gamma(1)=b$,
  we have $F(b) - F(a) = \int_\gamma \langle F', d\gamma \rangle$.
\end{proposition}

 Applying this to $\Loss$, we get
  \begin{align*}
    \Loss(b) - \Loss(a) &= \int_\gamma \langle \Loss', d\gamma\rangle
\end{align*}

Applying what we showed about $\Loss'$ in part (1), we have that this integral
simplifies as follows:
\begin{align*}
    &= \int_\gamma \langle \cfunct v, d\gamma \rangle
    = \int_\gamma \frac{\cfunct}{\cfunct_{\prplosssubscript}}\langle \Loss_{\prplosssubscript}', d\gamma \rangle
  \end{align*}
In the last line, we used the assumption that $\kappa_{prp}>0$, so we
can divide by it.
We now use another fact from calculus:
\begin{proposition} \label{prop:orthogonal}
For any smooth $H$, the gradient  $H'$ is orthogonal to the tangent plane of 
a constant surface $H_z$.  
\end{proposition}

Now note that $\gamma$ lies in $H_z$, so $d\gamma$ is in the tangent
plane of $H_z$. So the inner product $\langle \Loss', d \gamma
\rangle=0$, for every point of $\gamma$.  And since
$\Loss_{\prplosssubscript}'$ is a constant multiple of $\Loss'$ by
part (1), we have $\langle \Loss_{\prplosssubscript}', d\gamma \rangle
= 0$ for every point of $\gamma$.  This implies that $\Loss(a) =
\Loss(b)$ and that there exists some $h:\RR\to\RR$ function such that
$\Loss = h(\Loss_{\prplosssubscript})$, though it is not necessarily
differentiable or even continuous.

\textbf{4.} We claim that $h$ should be differentiable.  Let $d$ be a
vector.
By $\partial_d \Loss_{\prplosssubscript}(p) \neq 0$, then the directional derivative of $\Loss$ is
  \begin{align*}
    \partial_d \Loss(\probs) &= \lim_{\epsilon\to 0} \frac{\Loss(\probs+ \epsilon d) - \Loss(\probs)}{h}
    \\ &= \lim_{\epsilon \to 0} \frac{h(\Loss_{\prplosssubscript}(\probs + \epsilon d))- h(\Loss_{\prplosssubscript}(\probs))}{\Loss_{\prplosssubscript}(\probs +  \epsilon d)- \Loss_{\prplosssubscript}(\probs)} \cdot \frac{\Loss_{\prplosssubscript}(\probs + 
      \epsilon d)- \Loss_{\prplosssubscript}(\probs)}{\epsilon}
    \\ &= \partial_d \Loss_{\prplosssubscript}(\probs) \lim_{\epsilon \to 0} \frac{h(\Loss_{\prplosssubscript}(\probs + \epsilon d))- h(\Loss_{\prplosssubscript}(\probs))}{\Loss_{\prplosssubscript}(\probs +  \epsilon d)- \Loss_{\prplosssubscript}(\probs)}
  \end{align*}
  By assumption $\Loss_{\prplosssubscript}'$ and $\Loss'$ exist and they are continuous, therefore the above limit also exists which is just the derivative of $h$ at 
  $\Loss_{\prplosssubscript}(\probs)$. That means that $h$ is indeed continuously differentiable on the domain of $\Loss_{\prplosssubscript}$, which is $\RR$.
  Note that we fixed $y$ at the very beginning. There are only finitely many such $y$ over a set of $\outdim$ outputs, so we have a $h$ function for every $y$, and putting these
  together gets the $h$ that we want.
\end{proof}

\section{Theorems related to the $\biprploss$}
\label{app:biprploss}

In this section we prove the characterization theorems for loss functions satisfying the $\biprps$ property. We recall the two theorems:

\thmbiprploss*

\thmbiprpchar*

We also recall the definition of the $\biprploss$:

$$\Loss_{\biprplosssubscript}(\probs, \mylabel) =
\underbrace{- \frac{1}{\numallowed}\sum_i \mylabelscalar_i \log(\probsscalar_i)}_{\textrm{Allowed term}} + 
\underbrace{\frac{1}{\outdim - \numallowed} \sum_i (1-\mylabelscalar_i) \log(\probsscalar_i)}_{\textrm{Disallowed term}}
$$

As before, let $\numallowed=\sum_i \mylabelscalar_i$, and $\outdim-\numallowed = \sum_i (1-\mylabelscalar_i)$ denote the number of acceptable and unacceptable labels, respectively.

\begin{proof}[Proof of \Cref{thm:biprp_loss}.]

  Recall that the $\biprps$ property assumes a softmax regression model and that the logit vector $\logits$ is a parameter vector.  
  In order to show that $\Loss_{\biprplosssubscript}$ has the
  $\biprps$ property, we need to show that partial derivatives in
  $\logitsscalar_j$ are equal whenever $\mylabelscalar_j=1$ and they are also equal whenever $\mylabelscalar_j=0$. First, we compute
  the partial derivatives in the probabilities $\probsscalar_j$.

  \begin{align*}
    \frac{\partial}{\partial \probsscalar_i} \Loss_{\biprplosssubscript}(\probs, \mylabel) &= -\frac{1}{\numallowed} \frac{\mylabelscalar_i}{\probsscalar_i} + \frac{1}{\outdim - \numallowed} \frac{1-\mylabelscalar_i}{\probsscalar_i}
    = \begin{cases}
      -\frac{1}{\numallowed} \frac{1}{\probsscalar_i}~&\text{if}~\mylabelscalar_i=1
      \\\frac{1}{\outdim- \numallowed} \frac{1}{\probsscalar_i}~&\text{otherwise}
    \end{cases}
  \end{align*}

  Now, we compute the partial derivatives in logit $\logitsscalar_j$. 
Recall that $\pd{\probsscalar_i}{\logitsscalar_j}$ is the partial derivative of the softmax function which is $\probsscalar_i(\delta_{ij} -\probsscalar_j)$. 
\begin{align*}
  \pd{\Loss_{\biprplosssubscript}}{\logitsscalar_j} &= \sum_{i} \pd{\Loss_{\biprplosssubscript}}{\probsscalar_i}\pd{\probsscalar_i}{\logitsscalar_j}
  = \sum_i \left(-\frac{1}{\numallowed} \frac{\mylabelscalar_i}{\probsscalar_i} + \frac{1}{\outdim - \numallowed} \frac{1-\mylabelscalar_i}{\probsscalar_i}\right) \probsscalar_i ( \delta_{ij}-\probsscalar_j)
  \\&= \sum_i \left(-\frac{1}{\numallowed} \mylabelscalar_i + \frac{1}{\outdim - \numallowed} (1-\mylabelscalar_i)\right) ( \delta_{ij}-\probsscalar_j)
  \\&= \left(-\frac{1}{\numallowed} \mylabelscalar_j + \frac{1}{\outdim - \numallowed} (1-\mylabelscalar_j)\right) - \probsscalar_j \sum_i \left(-\frac{1}{\numallowed} \mylabelscalar_i + \frac{1}{\outdim - \numallowed} (1-\mylabelscalar_i)\right)
  \\&= -\frac{1}{\numallowed} \mylabelscalar_j + \frac{1}{\outdim - \numallowed} (1-\mylabelscalar_j)
  = \begin{cases}
    -\frac{1}{\numallowed}~&\text{if}~\mylabelscalar_i=1
    \\\frac{1}{\outdim - \numallowed}~&\text{otherwise}
  \end{cases}
\end{align*}
We can observe that the gradients of the logits with $\mylabelscalar_i=1$ and those with $\mylabelscalar_i=0$ are equal, indicating that the loss function satisfies the $\biprps$ property.

\end{proof}

\begin{proof}[Proof of \Cref{thm:prp_char}.]

The proof follows along the same lines as in the the $\prps$ case.
The revised outline is just as before:
  \begin{enumerate}
    \item We argue that $\Loss' = \cfunct v$ for some $\cfunct:\RR^k \to \RR$. 
and that $\Loss' = d \cdot \Loss'_{\biprplosssubscript}$ for some constant function $d$.
    \item We show that the sets $H_z = \Loss_{\biprplosssubscript}^{-1}(\{z\})$  are path-connected.
    \item We argue that $\Loss$ is constant on $H_z$ for any $z$. Restated, this means that the function $h:\RR\to\RR$ as
required by theorem (but not necessarily smooth) exists.  
    \item The $h$ function is continuously differentiable.
  \end{enumerate}

  \textbf{1.} Consider a loss function $\Loss$ that has the $\biprps$ property and satisfies the technical assumptions in the statement of the theorem. 
Let $\Loss'_{accept}$ denote the gradient restricted to acceptable inputs, and $\Loss'_{unaccept}$ the restriction
to unacceptable outputs. We let $\Loss'_{\biprplosssubscript,accept}$ and $\Loss'_{\biprplosssubscript,unaccept}$ denote the special case
where the loss is the $\biprploss$.
First, we show that $\Loss'$ and $\Loss'_{\biprplosssubscript}$ are scalar multiples of one another at any $\vec p$. 
Based on our assumption that the ratios of gradients for acceptable and unacceptable inputs are equivalent, 
we can infer that $\Loss'_{accept}$ and $\Loss'_{\biprplosssubscript,accept}$ are scalar multiples of each other, 
as are $\Loss'_{unaccept}$ and $\Loss'_{\biprplosssubscript,unaccept}$.
However, we still need to prove that the constants for both pairs are identical.

Let $v_{accept}$ and $v_{unaccept}$ denote the gradients of $\Loss$ with respect to the acceptable and unacceptable logits.
Similarly, for  $\Loss_{\biprplosssubscript}$, we use  $v_{\biprplosssubscript,accept}, v_{\biprplosssubscript,unaccept}$.
Furthermore, we will use $v$ for the gradients of a general $\Loss$ (with respect to \emph{logits}), without restricting to particular outputs.
We similarly use $v_{\biprplosssubscript}$ for the full gradient vector of $\Loss_{\biprplosssubscript}$, with respect to logits.
Since $\Loss$ satisfies the $\biprps$ property, we have $v_{accept} = \kappa_{accept} \underline{1}$ and $v_{unaccept} = \kappa_{unaccept} \underline{1}$ for some $\kappa_{accept},\kappa_{unaccept}$ scalars.
For any $\Loss$, the gradients on the logits add to $0$, since:
\begin{align*}
  \sum_j \frac{\partial \Loss}{\partial \logitsscalar_j} &= \sum_j \sum_i \frac{\partial \Loss}{\partial \probsscalar_i} \frac{\partial \probsscalar_i}{\partial \logitsscalar_j}
  = \sum_i \frac{\partial \Loss}{\partial \probsscalar_i} \sum_j \frac{\partial \probsscalar_i}{\partial \logitsscalar_j}
  = \sum_i \frac{\partial \Loss}{\partial \probsscalar_i} \sum_j \probsscalar_i(\delta_{ij}-\probsscalar_j) 
  \\ &= \sum_i \frac{\partial \Loss}{\partial \probsscalar_i} \probsscalar_i \sum_j (\delta_{ij}-\probsscalar_j) 
 = 0
\end{align*}
The last equality follows because the $\probsscalar_j$ form a probability distribution, hence for any fixed $i$, $\sum_j (\delta_{ij}-\probsscalar_j)=0$.

Since the gradients on the logits add to zero, we have
\begin{align*}
  0 &= v_{accept} + v_{unaccept} = \numallowed \kappa_{accept} + (\outdim - \numallowed)\kappa_{unaccept}  
  \\0 &= v_{\biprplosssubscript,accept} + v_{\biprplosssubscript,unaccept} = \numallowed \kappa_{\biprplosssubscript,accept} + (\outdim - \numallowed) \kappa_{\biprplosssubscript,unaccept}  
\end{align*}
For this, it is easy to see that the ratios $\kappa_{accept}:\kappa_{\biprplosssubscript,accept}$ and $\kappa_{unaccept}:\kappa_{\biprplosssubscript,unaccept}$ have to be equal. Thus we have derived the following result:

\begin{proposition}  \label{prop:partiallogit}
$v$ and $v_{\biprplosssubscript}$ are scalar multiples of one another. 
\end{proposition}

Recall that the goal of part (1) of the proof is to show that $\Loss' = d\cdot \Loss'_{\biprplosssubscript}$, i.e., the gradients with respect to the probabilities of $\Loss$ and $\Loss_{\biprplosssubscript}$ are scalar multiple of one another.
\Cref{prop:partiallogit} shows the analog for the gradients with respect to the logits.
But because of the chain rule, the gradients with respect to the probabilities and the logits are connected by a linear transformation. 
We define the vector to vector function  $A$ by
\[
A_{i,j}=\probsscalar_j(\delta_{ij}-\probsscalar_i)
\]
This is quite similar to the function $A$ in the earlier proof of \Cref{thm:prp_char}, but this time $i, j$ range over all inputs,
not just acceptable ones.
The  equality
\[
\frac{\partial \Loss}{\partial \logitsscalar_j} =
 \sum_i \frac{\partial \Loss}{\partial \probsscalar_i} \probsscalar_i \sum_j (\delta_{ij}-\probsscalar_j) 
\]
can be expressed in matrix multiplication terms as 
\[
v = A\Loss'
\]

If $A$ were invertible, then $\Loss' = A^{-1}v$, and it would follow that 
$\Loss' = d\cdot \Loss'_{\biprplosssubscript}$.  Unfortunately, this is not true. 
From the fact that the function uses all inputs, which sum to $1$,
we can infer that $\det A = 0$, and so we cannot take the inverse of $A$ over the entire input space.

Let $V$ be the orthogonal complement of $\underline{1}$. This is all real vectors whose dot product with 
$\underline{1}$ is $0$; that is, vectors whose sum is $0$.  We claim  that $V$ is invertible when we restrict to these vectors:
\begin{claim} \label{clm:invertibleoversubspace} $A$ is invertible over $V$.
\end{claim}

We mentioned above that the gradients sum to $0$, and
the gradient with respect to the logits -- that is, a $v$ above --  must be in $V$.
Thus, from \Cref{clm:invertibleoversubspace} we are able to take an inverse of $A$ over the relevant vectors, and derive that the partials with respect to the probabilities are scalar multiples, as before.
We now turn to the proof of \Cref{clm:invertibleoversubspace}.

\begin{proof}
  
Recall that $A$ is the Jacobian of the softmax function, which is a surjective function from $\RR^n$, the space of logits,  to the
space of probability distributions over $\outdim$ categories. The latter is an $\outdim-1$ dimensional subspace of $\RR^n$.
We already showed that $\underline{1}$ is in the kernel of $A$. Let $f$ denote the softmax function, then $A:= df$. Since $f$ is a smooth and surjective function,
the rank of $df$ is equal to the dimension of the codomain, i.e. the space of probability distributions, which has dimension $\outdim-1$.
It follows that $\dim \ker A = \outdim - \text{rank}~df = 1$, consequently $\ker A$ is generated by $\underline{1}$ and so $A$ is invertible over $V$, as required.
\end{proof}

  \textbf{2.}
  Analogous to what we did in the $\prp$ case, we argue for path-connectedness 
of
$\Loss_{\biprplosssubscript}$.
Let $H_z = \Loss_{\biprplosssubscript}^{-1}(\{z\}) = \{ \probs | \Loss_{\biprplosssubscript}(\probs, \mylabel) = z \}$ be the preimage of $z$. 
  Let $\mathcal{P}$ be the set of distributions with each probability non-zero and neither the acceptable
nor the unacceptable outputs sum to $1$.

  Note that $\mathcal{P}$ is an open path-connected subset of $\RR^n$ and hence a differentiable manifold. 
  At the same time, the range of $\Loss_{\biprplosssubscript}$ is $\RR$, which is also a differentiable manifold.
  Thus we can view $\Loss_{\biprplosssubscript}$ as a smooth map between manifolds $\mathcal{P}$ and $\RR$. 
  We will show the analogous claim as for $\prp$:
\begin{claim} \label{clm:preimagemanifoldbi} Each $H_z$
is a differentiable manifold.
\end{claim}
  
\begin{proof}
  Let $X,Y$ be two differentiable manifolds and $f:X\to Y$ a smooth map between them.
  We say that $y\in Y$ is regular if for every $x\in f^{-1}(y)$ the map $df_x: T_xX \to T_yY$ is surjective, where $T_xX$ is the tangent space of $X$ in $x$.

We again use that fact 
 that if $y\in Y$ is a regular value of $f$, then $f^{-1}(y)$ is a differentiable submanifold of $X$.
  We show that $d\Loss_{\biprplosssubscript}$ is surjective everywhere, in order to argue, using
the fact above, that  the pre-image of a single point is a differentiable
manifold.

  Since $\RR$ is one-dimensional, $d\Loss_{\biprplosssubscript}$ is not surjective precisely when $d\Loss_{\biprplosssubscript}$ is the zero map.
  Equivalently the gradient is zero, which can only occur on the 
boundary of $\mathcal{P}$. %
  Therefore, any $z$ is a regular value of $H$, 
and consequently $H_z$ is a differentiable submanifold of $\mathcal{P}$.
\end{proof}

  For any probability distribution $\probs=(\probsscalar_1 \ldots \probsscalar_n)$ over $\outdim$ categories, we let $D_\probs$
denote all distributions that agree with $\probs$ on both the ratios of acceptable values, as well as on  the ratio of
unacceptable values, with both of these nonzero. That is, $D_\probs$ is the subset of $\mathcal{P}$ that we get 
by fixing the ratios for both acceptable and unacceptable values. 

We again proceed analogously to the $\prp$ case:
\begin{claim} \label{clm:everyvaluebiprp}  For each fixed $\mylabel$ having $1$
on entries for $\vec a$ and $0$ on entries for $\vec u$,   $\Loss_{\biprplosssubscript}$ takes every value precisely once on $D_\probs$. 
\end{claim}

\begin{proof}
Let us fix non-trivial distributions $\vec a$ on acceptable outputs and $\vec u$ on  unacceptable outputs
with the sum of the entries of both coming to $1$. $D_\probs$ consists of the distributions $\omega \vec a$, $(1-\omega) \vec u$ for
all $0 < \omega<1$.
To prove the claim, observe that the loss is
\begin{align*}
- \frac{1}{\numallowed} \left(\sum_{i \in A} \log( \omega a_i) \right) +
\frac{1}{\outdim - \numallowed} \sum_{i \in U} \log( (1 - \omega) u_i)
  \end{align*}
Here $A$ are the indices of acceptable values and $U$ the indices of unacceptable values.
Note that this simplifies to an expression of the form
\[ 
-\log(\omega) - \frac{1}{\numallowed} \sum_{i \in A} \log(a_i) + \log(1- \omega) + \frac{1}{\outdim-\numallowed} \sum_{i \in U} u_i
\]
If we ignore terms without $\omega$, this is 
$ - \log(\omega) + \log(1-\omega)$.
  Thus we see, as in the $\prp$ case, when $\omega \to 0$ it converges to  $\infty$ and when $\omega \to 1$ it converges to $-\infty$.
  And differentiating with respect to $\omega$, we see that
 the above mapping is monotonically strictly decreasing in $\omega$, and consequently $\Loss_{\biprplosssubscript}$ takes every value of $\RR$ precisely once on the
set.
\end{proof}

Recall that we are interested in showing path connectedness
of the set $H_z$, the pre-image of singletons under $\Loss_{\biprplosssubscript}$.
By the claim above, we know that as we vary $\probs$, $\Loss_{\biprplosssubscript}$ will always hit $H_z$ exactly once on the set $D_\probs$, but the point at which it hits $H_z$ will vary with $\probs$.

Let $\pi$ be the quotient map equating two elements if they are in the same $D_\probs$.
Thus by definition  the preimage of a point under $\pi$ is precisely a set $D_\probs$ for some $\probs$.
Since we have observed above that $\Loss_{\biprplosssubscript}$ takes every value once over a fixed $D_\probs$, we conclude that $\pi$ is a bijection between $\pi(\mathcal{P})$ and $H_z$.  
We will show in the next paragraph that $\pi$ is an open map, but first introduce a useful lemma.

\begin{lemma}
If $f$ is a quotient map, then $f$ is open if and only if
$$U\subset X~\text{is open}~\Rightarrow f^{-1}(f(U))~\text{is open}$$
\end{lemma}

\begin{proof}
If $f$ is open, then $f(U)$ is open and so $f^{-1}(f(U))$.
For the converse, the fact that $f^{-1}(f(U))$ is open implies that $f(U)$ is open, because $f$ is a quotient map. Since this holds for every $U$ open, it follows that $f$ is open.
\end{proof}

We use this fact to show that $\pi$ is open. More precisely, we have that for every open $U\subset \mathcal{P}$,
$ \pi^{-1}(\pi(U)) = \bigcup_{\probs \in U} D_\probs$.
Unfortunately, the sets $D_\probs$ lines are not open subsets; so we 
cannot  deduce directly that $\bigcup_{\probs \in U} D_\probs$ is open. 
Let $\mathcal{S}$ denote the set of linear functions that send probabilities to probabilities,
 such that the ratio is preserved for the acceptable and also for the unacceptable outputs.
Since the functions in $\mathcal{S}$ are linear, they are also open maps. 
Moreover, we have $D_\probs = \cup_{S\in \mathcal{S}} S(\probs)$.
Therefore, we can write $\bigcup_{\probs\in U} D_\probs = \bigcup_{S\in\mathcal{S}} S(U)$.
Since $U$ is open, each $S(U)$ is also open, and thus so is the union over all $S$.
We conclude that $\pi^{-1}(\pi(U))$ is open.

We will now use the fact that $H_z$ is a manifold by \Cref{clm:preimagemanifoldbi}.
  It is known that for manifolds, connectedness and path-connectedness are equivalent.
  If $H_z$ were not connected, then there would be $U,V$ disjoint
non-empty open sets such that $H_z\subset U\cup V$. 
Note that the image of $\pi$ is connected: we start with a connected space, namely the whole probability space, and take quotient by a continuous function.
  Thus  $\pi(\mathcal{P})$ is connected, and $\pi{\mathcal{P}} = \pi(H_z) \subset \pi(U\cup V) = \pi(U)\cup \pi(V)$. Note that $\pi(U), \pi(V)$ are open because $\pi$ is an open map.
  We cannot have two disjoint open sets covering the connected set $\pi(\mathcal{P})$. Thus the sets $\pi(U)$ and $\pi(V)$ must overlap:
  there are $u\in U, v\in V$ points such that $\pi(u) = \pi(v)$. 
Thus there are two distinct
points in the pre-image of $\pi$ with the same value. Since
the pre-image is a $D_\probs$ line, this
contradicts \Cref{clm:everyvaluebiprp}.

  \textbf{3.} We show that $\Loss$ is constant on $H_z$ for any $z$,
  and that a $h:\RR\to\RR$ function exists such that $\Loss = h(\Loss_{\biprplosssubscript})$.

The idea will be that for any $a \neq b \in H_z$ we show
$\Loss(b)-\Loss(a)=0$.  We do this by computing $\Loss(b)-\Loss(a)$ as an integral of
a quantity, over a path in $H_z$ between $a$ and $b$, using the fact
that $H_z$ is path-connected. The quantity will involve a dot product
with the derivative of $\Loss$, and we will use part (1) to argue that this dot
product is always $0$.

We will again make use of \Cref{prop:pathint}, which states
  that  $F(b) - F(a) = \int_\gamma \langle F', d\gamma \rangle$.
 We can again apply this to $\Loss$ to get
  \begin{align*}
    \Loss(b) - \Loss(a) &= \int_\gamma \langle \Loss', d\gamma\rangle
\end{align*}

By \Cref{prop:orthogonal}, the inner product with $\Loss'_{\biprplosssubscript}$ in place of $\Loss'$ is $0$
within a constant surface $H_z$.
And again since $\Loss'$ is always a scalar multiple of $\Loss_{\biprplosssubscript}'$, we conclude
  In the last line,
we used the assumption that the gradients of $\Loss_{\biprplosssubscript}$ and $\Loss$ have a constant ratio.

The argument that $h$ is differentiable is almost identical to the argument for $\prp$.
\end{proof}

\pagebreak
\section{Label Dependent Noise Models for synthetic PLL Datasets}
\label{app:noisemodel}

In Subsection \ref{subsec:realsynexp} we described a model for adding distractors
synthetically to a real dataset. Here we provide more detail.

\cite{pllleveraging} introduces three PLL noise models for
classification with $\outdim = 10$ labels. The models are instance-
independent, i.e., the noise only depends on the true
label. \Cref{fig:cifar10_cases} presents results based on $5$ such
noise matrices. Of these the first three are taken directly from
\cite{pllleveraging} and the last two are harder variants created by
us.

The noise models are represented as $[\outdim \times \outdim]$
matrices $M$ where $M_{ij}$ represents the probability of label $j$
becoming a distractor given true label $i$.  In the following we
describe these $5$ noise matrices.

\begin{tabular}{m{0.08\textwidth} m{0.36\textwidth} m{0.48\textwidth}}
  {\bf Case} & {\bf Noise Matrix} & {\bf Description} \\
  \toprule
  {\bf  1}
  & $\left[\begin{smallmatrix}
      1 & 0.5 & 0 & 0 & 0 & 0 & 0 & 0 & 0 & 0 \\
      0 & 1 & 0.5 & 0 & 0 & 0 & 0 & 0 & 0 & 0 \\
      0 & 0 & 1 & 0.5 & 0 & 0 & 0 & 0 & 0 & 0 \\
      0 & 0 & 0 & 1 & 0.5 & 0 & 0 & 0 & 0 & 0 \\
      0 & 0 & 0 & 0 & 1 & 0.5 & 0 & 0 & 0 & 0 \\
      0 & 0 & 0 & 0 & 0 & 1 & 0.5 & 0 & 0 & 0 \\
      0 & 0 & 0 & 0 & 0 & 0 & 1 & 0.5 & 0 & 0 \\
      0 & 0 & 0 & 0 & 0 & 0 & 0 & 1 & 0.5 & 0 \\
      0 & 0 & 0 & 0 & 0 & 0 & 0 & 0 & 1 & 0.5 \\
      0.5 & 0 & 0 & 0 & 0 & 0 & 0 & 0 & 0 & 1 \\
    \end{smallmatrix}\right]$
  & There is a single potential distractor for each true label, which is present with probability $0.5$. The expected number of distractors is $0.5$.
  \\
    {\bf  2}
    & $\left[\begin{smallmatrix}
        1 & 0.3 & 0 & 0 & 0 & 0 & 0 & 0 & 0 & 0.3 \\
        0.3 & 1 & 0.3 & 0 & 0 & 0 & 0 & 0 & 0 & 0 \\
        0 & 0.3 & 1 & 0.3 & 0 & 0 & 0 & 0 & 0 & 0 \\
        0 & 0 & 0.3 & 1 & 0.3 & 0 & 0 & 0 & 0 & 0 \\
        0 & 0 & 0 & 0.3 & 1 & 0.3 & 0 & 0 & 0 & 0 \\
        0 & 0 & 0 & 0 & 0.3 & 1 & 0.3 & 0 & 0 & 0 \\
        0 & 0 & 0 & 0 & 0 & 0.3 & 1 & 0.3 & 0 & 0 \\
        0 & 0 & 0 & 0 & 0 & 0 & 0.3 & 1 & 0.3 & 0 \\
        0 & 0 & 0 & 0 & 0 & 0 & 0 & 0.3 & 1 & 0.3 \\
        0.3 & 0 & 0 & 0 & 0 & 0 & 0 & 0 & 0.3 & 1 \\
      \end{smallmatrix}\right]$
    & There are two potential distractors for each true label, each of which is present with probability $0.3$. The expected number of distractors is $0.6$.
    \\
      {\bf  3}
      & $\left[\begin{smallmatrix}
          1 & 0.5 & 0.3 & 0.1 & 0 & 0 & 0 & 0.1 & 0.3 & 0.5 \\
          0.5 & 1 & 0.5 & 0.3 & 0.1 & 0 & 0 & 0 & 0.1 & 0.3 \\
          0.3 & 0.5 & 1 & 0.5 & 0.3 & 0.1 & 0 & 0 & 0 & 0.1 \\
          0.1 & 0.3 & 0.5 & 1 & 0.5 & 0.3 & 0.1 & 0 & 0 & 0 \\
          0 & 0.1 & 0.3 & 0.5 & 1 & 0.5 & 0.3 & 0.1 & 0 & 0 \\
          0 & 0 & 0.1 & 0.3 & 0.5 & 1 & 0.5 & 0.3 & 0.1 & 0 \\
          0 & 0 & 0 & 0.1 & 0.3 & 0.5 & 1 & 0.5 & 0.3 & 0.1 \\
          0.1 & 0 & 0 & 0 & 0.1 & 0.3 & 0.5 & 1 & 0.5 & 0.3 \\
          0.3 & 0.1 & 0 & 0 & 0 & 0.1 & 0.3 & 0.5 & 1 & 0.5 \\
          0.5 & 0.3 & 0.1 & 0 & 0 & 0 & 0.1 & 0.3 & 0.5 & 1 \\
        \end{smallmatrix}\right]$
      & For each true label, there are 2 potential distractors
      with probability $0.5$, 2 with probability $0.3$ and 2 with
      probability $0.1$. The expected number of distractors is $1.8$.
      \\
        {\bf  4}
        & $\left[\begin{smallmatrix}
            1 & 0.2 & 0.8 & 0.8 & 0.8 & 0.4 & 0.4 & 0.2 & 0.2 & 0.2 \\
            0.2 & 1 & 0.2 & 0.8 & 0.8 & 0.8 & 0.4 & 0.4 & 0.2 & 0.2 \\
            0.2 & 0.2 & 1 & 0.2 & 0.8 & 0.8 & 0.8 & 0.4 & 0.4 & 0.2 \\
            0.2 & 0.2 & 0.2 & 1 & 0.2 & 0.8 & 0.8 & 0.8 & 0.4 & 0.4 \\
            0.4 & 0.2 & 0.2 & 0.2 & 1 & 0.2 & 0.8 & 0.8 & 0.8 & 0.4 \\
            0.4 & 0.4 & 0.2 & 0.2 & 0.2 & 1 & 0.2 & 0.8 & 0.8 & 0.8 \\
            0.8 & 0.4 & 0.4 & 0.2 & 0.2 & 0.2 & 1 & 0.2 & 0.8 & 0.8 \\
            0.8 & 0.8 & 0.4 & 0.4 & 0.2 & 0.2 & 0.2 & 1 & 0.2 & 0.8 \\
            0.8 & 0.8 & 0.8 & 0.4 & 0.4 & 0.2 & 0.2 & 0.2 & 1 & 0.2 \\
            0.2 & 0.8 & 0.8 & 0.8 & 0.4 & 0.4 & 0.2 & 0.2 & 0.2 & 1 \\
          \end{smallmatrix}\right]$
        & For each true label, there are 4 potential distractors
        with probability $0.2$, 3 with probability $0.8$ and 2 with
        probability $0.4$. The expected number of distractors is $4$.
        \\
          {\bf  5}
          & $\left[\begin{smallmatrix}
              1 & 0.9 & 0.8 & 0.8 & 0.8 & 0.7 & 0.7 & 0.6 & 0.9 & 0.9 \\
              0.9 & 1 & 0.9 & 0.8 & 0.8 & 0.8 & 0.7 & 0.7 & 0.6 & 0.9 \\
              0.9 & 0.9 & 1 & 0.9 & 0.8 & 0.8 & 0.8 & 0.7 & 0.7 & 0.6 \\
              0.6 & 0.9 & 0.9 & 1 & 0.9 & 0.8 & 0.8 & 0.8 & 0.7 & 0.7 \\
              0.7 & 0.6 & 0.9 & 0.9 & 1 & 0.9 & 0.8 & 0.8 & 0.8 & 0.7 \\
              0.7 & 0.7 & 0.6 & 0.9 & 0.9 & 1 & 0.9 & 0.8 & 0.8 & 0.8 \\
              0.8 & 0.7 & 0.7 & 0.6 & 0.9 & 0.9 & 1 & 0.9 & 0.8 & 0.8 \\
              0.8 & 0.8 & 0.7 & 0.7 & 0.6 & 0.9 & 0.9 & 1 & 0.9 & 0.8 \\
              0.8 & 0.8 & 0.8 & 0.7 & 0.7 & 0.6 & 0.9 & 0.9 & 1 & 0.9 \\
              0.9 & 0.8 & 0.8 & 0.8 & 0.7 & 0.7 & 0.6 & 0.9 & 0.9 & 1 \\
            \end{smallmatrix}\right]$
          & For each true label, there are 3 potential distractors
          with probability $0.9$, 3 with probability $0.8$, 2 with probability
          $0.7$ and 1 with probability $0.6$. The expected number of distractors
          is $7.1$.
          \\
\end{tabular}

\newpage
\bibliography{paper}

\end{document}